\documentclass[12pt]{article} % Anonymized submission
\usepackage{times}
\usepackage{titlesec}

\usepackage[algo2e,ruled]{algorithm2e}
\usepackage{graphicx}

\usepackage{amsmath}
\usepackage{amssymb}
\usepackage{amsthm}
\usepackage{natbib}
\usepackage{url}
\usepackage{amsfonts}
\usepackage{color}
\usepackage{bbm}
\usepackage[top=0.5in, bottom=0.5in, left=0.8in, right=0.8in, includeheadfoot]{geometry}

\usepackage{hyperref}
\usepackage{nameref}
\hypersetup{colorlinks,
            linkcolor=blue,
            citecolor=blue,
            urlcolor=magenta,
            linktocpage,
            plainpages=false}

\newtheorem{thm}{Theorem}
\newtheorem{proposition}{Proposition}
\newtheorem{lemma}{Lemma}
\newtheorem{assumption}{Assumption}

\theoremstyle{definition}
\newtheorem{definition}{Definition}

\newcommand{\ip}[2]{ #1 \cdot #2 }

\newcommand{\bydef}{:=}

\newcommand{\expect}[1]{\mathbb{E}\left[{#1}\right]}
\newcommand{\prob}[1]{\mathbb{P}\left[{#1}\right]}
\newcommand{\given}{\; \big\vert \;} 
\newcommand{\assign}{\leftarrow}
\newcommand{\appropto}{\mathrel{\vcenter{
  \offinterlineskip\halign{\hfil$##$\cr
    \propto\cr\noalign{\kern2pt}\sim\cr\noalign{\kern-2pt}}}}}
\newcommand{\kldiv}[2]{\mathbb{KL}\left( {#1} \; || \; {#2} \right)}
\newcommand*{\WITHFIGS}{}%

\title{Thompson
  Sampling for Learning Parameterized Markov Decision Processes}
\author{Aditya Gopalan\thanks{Department of Electrical Communication
    Engineering, Indian Institute of Science, Bangalore, India. Email:
    {{\tt aditya@ece.iisc.ernet.in}}} \and Shie Mannor
  \thanks{Department of Electrical Engineering, Technion - Israel
    Institute of Technology, Haifa, Israel. Email: {{\tt
        shie@ee.technion.ac.il}}}}

\begin{document}

\maketitle

\begin{abstract}
  \noindent We consider reinforcement learning in parameterized Markov
  Decision Processes (MDPs), where the parameterization may induce
  correlation across transition probabilities or
  rewards. Consequently, observing a particular state transition might
  yield useful information about other, unobserved, parts of the
  MDP. We present a version of Thompson sampling for parameterized
  reinforcement learning problems, and derive a frequentist regret
  bound for priors over general parameter spaces. The result shows
  that the number of instants where suboptimal actions are chosen
  scales logarithmically with time, with high probability. It holds
  for prior distributions that put significant probability near the
  true model, without any additional, specific closed-form structure
  such as conjugate or product-form priors. The constant factor in the
  logarithmic scaling encodes the information complexity of learning
  the MDP in terms of the Kullback-Leibler geometry of the parameter
  space.
\end{abstract}

% \begin{keywords}
% Thompson sampling, Markov Decision Process, Reinforcement learning
% \end{keywords}

\section{Introduction}
% STORY: reinf learning concerned with study of how an agent can learn
% via repeated interaction with unknown env; goal is to learn to act
% optimally to maximize some notion of performance, usually maximum
% reward or value; formally, RL can be expressed through the MDP
% formalism - state, action, reward and prob transitions; learner
% plays actions using past info;

Reinforcement Learning (RL) is concerned with studying how an agent
learns by repeated interaction with its environment. The goal of the
agent is to act optimally to maximize some notion of performance,
typically its net reward, in an environment modeled by a Markov
Decision Process (MDP) comprising states, actions and state
transition probabilities.

%%  The reinforcement learning problem is formally
%% expressed using the Markov Decision Process (MDP) framework. An MDP
%% comprises a set of states for the environment and a set of actions
%% that the agent can take. It also specifies the reward that the agent
%% receives for taking an action in a state, and the probability of
%% transitioning into the next state given the action played in the
%% current state. The learner, starting with no knowledge of the
%% underlying MDP, plays an action at each time possibly depending on the
%% current state and past history, i.e., the observed sequence of
%% previous states, actions and rewards.

% when mdp known accurately, optimal behavior becomes in essence a
% dyn programming task; main difficulty of RL problem stems from
% uncertainty in the model; the need to simult. explore environment to
% discover its structure -- rewards, state transitions, and on the
% other hand exploit earned knowledge to act optimally; moreover,
% current actions can influence information which learner will acquire
% in future; so balance between exploration and exploitation is
% necessary for efficient learning

The difficulty of reinforcement learning stems primarily from the
learner's uncertainty in knowing the environment. When the environment
is perfectly known, finding optimal behavior essentially becomes a
dynamic programming or planning task. Without this knowledge, the
learner faces a conflict between the need to {\em explore} the
environment to discover its structure (e.g., reward/state transition
behavior), and the need to {\em exploit} accumulated information. The
trade-off is compounded by the fact that the agent's current action
influences future information. Thus, one has to strike the right
balance between exploration and exploitation in order to learn
efficiently.

% quick intro to state-of-art- ucrl,regal, rmax ..., how they operate

Several modern reinforcement learning algorithms, such as UCRL2
\citep{JakschOA10}, REGAL \citep{Bartlett2009} and R-max
\citep{Brafman2003}, learn MDPs using the well-known ``optimism under
uncertainty'' principle. The underlying strategy is to maintain
high-probability confidence intervals for each state-action transition
probability distribution and reward, shrinking the confidence interval
corresponding to the current state transition/reward at each
instant. Thus, observing a particular state transition/reward is
assumed to provide information for {\em only} that state and action.

% but, more structure in parameterized mdps- e.g. queueing system,
% coupled info, ...

However, one often encounters learning problems in complex
environments, often with some form of lower-dimensional
structure. {\em Parameterized} MDPs, in which the entire structure of
the MDP is determined by a parameter with only a few degrees of
freedom, are a typical example. With such MDPs, observing a state
transition at an instant can be informative about other, unobserved
transitions. As a motivating example, consider the problem of learning
to control a queue, where the state represents the occupancy of the
queue at each instant (\#packets), and the action is either FAST or
SLOW denoting the (known) rate of service that can be provided. The
state transitions are governed by (a) the type of service (FAST/SLOW)
chosen by the agent, together with (b) the arrival rate of packets to
the queue, and the cost at each step is a sum of a (known) cost for
the type of service and a holding cost per queued packet. Suppose that
packets arrive to the system with a fixed, {\em unknown} rate
$\lambda$ that alone parameterizes the underlying MDP. Then, {every}
state transition is informative about $\lambda$, and only a few
transitions are necessary to pinpoint $\lambda$ accurately and learn
the MDP fully. A more general example is a system with several queues
having potentially state-dependent arrival rates of a parametric form,
e.g., $\lambda(s) = f(\theta, s)$ for $\theta, s \in \mathbb{R}^d$.

% cannot easily adapt confidence-interval based frequentist techniques
% to exploit structure in parameterized learning problems; certainty
% equivalence: build most likely model given data you have, use policy
% optimal for model; but may not do enough exploration; 

% With regard to parameterized MDPs, most RL algorithms
% \citep{JakschOA10,Bartlett2009,Brafman2003} cannot be expected to
% learn at the best possible rate as they ``ignore'' the structure of
% the underlying problem. It is also a priori unclear if the core
% ``optimistic confidence region'' technique -- building a
% high-probability confidence interval for the unknown parameter -- can
% be adapted in a principled and tractable manner to exploit structural
% interdependencies for learning in general, parameterized MDPs.

% TS overcomes this issue but what if you put on a Bayesian hat?
% natural approach that results is posterior-based sampling also known
% by its original name of thompson sampling due to creator in 1933; in
% this approach maintain a fictitious prior over unknown
% parameters/models, sample from prior, optimize given sample, act,
% update prior and repeat; idea behind this is simple - sampling from
% prior helps explore; and that with more and more info prior will
% concentrate around true so you will do the right thing often; note
% that true env is assumed to be some fixed but unknown model, so
% there is nothing bayesian about the setup per se; rather
% interestingly there is also evidence to suggest that aspects of
% human decision making are driven by similar ideas as posterior
% sampling \citep{Vul et al}

A conceptually simple approach to learn MDPs with complex, parametric
structure is posterior or Thompson sampling \citep{Thompson}, in which
the learner starts by imposing a fictitious ``prior'' probability
distribution over the uncertain parameters (thus, over all possible
MDPs). A parameter is then sampled from this prior, the optimal
behavior for that particular parameter is computed and the action
prescribed by the behavior for the current state is taken. After the
resulting reward/state transition is observed, the prior is updated
using Bayes' rule, and the process repeats.

\subsection{Contributions}
\label{sec:contributions}
The main contribution of this work is to present and analyze {\em
  Thompson Sampling for MDPs (TSMDP)} -- an algorithm for
undiscounted, online, non-episodic reinforcement learning in general,
parameterized MDPs. The algorithm operates in {\em cycles} demarcated
by visits to a reference state, samples from the posterior once every
cycle and applies the optimal policy for the sample throughout the
cycle. Our primary result is a structural, problem-dependent
regret\footnote{more precisely, {\em pseudo-regret}
  \citep{AudBub10:minimax}} bound for TSMDP that holds for
sufficiently general parameter spaces and initial priors. The result
shows that for priors that put sufficiently large probability mass in
neighborhoods of the underlying parameter, with high probability the
TSMDP algorithm follows the optimal policy for all but a logarithmic
(in the time horizon) number of time instants. To our knowledge, these
are the first logarithmic gap-dependent bounds for Thompson sampling
in the MDP setting, without using any specific/closed form prior
structure. Furthermore, using a novel sample-path based concentration
analysis, we provide an explicit bound for the constant factor in this
logarithmic scaling which admits interpretation as a measure of the
``information complexity'' of the RL problem. The constant factor
arises as the solution to an optimization problem involving the
Kullback-Leibler geometry of the parameter space\footnote{more
  precisely, involving {\em marginal KL divergences} -- weighted
  KL-divergences that measure disparity between the true underlying
  MDP and other candidate MDPs. We discuss this in detail in Sections
  \ref{sec:overview}, \ref{sec:formal}.}, and encodes in a natural
fashion the interdependencies among elements of the MDP induced by the
parametric structure\footnote{In fact, the constant factor is similar
  in spirit to the notion of {\em eluder dimension} coined by Russo
  and Van Roy \citep{RusVan13:eluder} in their fully Bayesian analysis
  of Thompson sampling for the bandit setting.}. This results in
significantly improved regret scaling in settings when the
state/policy space is potentially large but where the space of
uncertain parameters is relatively much smaller (Section
\ref{sec:scalingconst}), and represents an advantage over decoupled
algorithms like UCRL2 which ignore the possibility of generalization
across states, and explore each state transition in isolation.

We also implement and evaluate the numerical performance of the TSMDP
algorithm for a queue MDP with unknown, state-dependent, parameterized
arrival rates, which appears to be significantly better than the
generic UCRL2 strategy.

% generally speaking, analysis of ts schemes presents difficulties of
% a different flavour than those encountered with quantifying
% performance in ucb-like strategies; in ucb-type schemes, the focus
% is on building tight confidence sets whcih the alg uses, and in
% analyzing how these sets shrink; in contrast, the alg design for ts
% is basically indpt of point or conf set estimates and its analysis
% revolves around understanding how the posterior distribution changes
% with time; 

The analysis of a distribution-based algorithm like Thompson sampling
poses difficulties of a flavor unlike than those encountered in the
analysis of algorithms using point estimates and confidence regions
\citep{JakschOA10,Bartlett2009}. In the latter class of algorithms,
the focus is on (a) theoretically constructing tight confidence sets
within which the algorithm uses the most optimistic parameter, and (b)
tracking how the size of these confidence sets diminishes with
time. In contrast, Thompson sampling, by design, is completely
divorced from analytically tailored confidence intervals or point
estimates. Understanding its performance is often complicated by the
exercise of tracking the (posterior) distribution, driven by
heterogeneous and history-dependent observations, concentrates with
time.

% moreover, in general, parameterized spaces, tracking evolution of
% complex posteriors without closed form presents challenges; most
% existing analyses of ts-style bayesian algs in bandits rely heavily
% on specific properties of the environment like independent and
% decoupled actions; and structural properties of the prior like being
% in a specific family of conjugate priors \citep{}; there are also
% complications moving beyond bandit case- accounting for markovian
% state evolution ...

The problem of quantifying how the prior in Thompson sampling evolves
in a general parameter space, with potentially complex structure or
coupling between elements, where the posterior may not even be
expressible in a convenient closed-form manner, poses unique
challenges that we address here. Almost all existing analyses of
Thompson sampling for the multi-armed bandit (a degenerate special
case of MDPs), rely heavily on specific properties of the problem,
especially independence across actions' rewards, and/or specific
structure of the prior such as belonging to a closed-form conjugate
prior family
\citep{AgrawalG,KauKorMun12:thompson,KorKauMun13:tsexpfam,AgrGoy13:contextual},
or finitely supported priors \citep{GopManMan14:thompson}.

Additional technical complications arise when generalizing from the
bandit case -- where the environment is stateless and
IID\footnote{Independent and Identically Distributed} -- to
state-based reinforcement learning in MDPs, in which state evolution
is coupled across time and evolves as a function of decisions
made. This makes tracking the evolution of the posterior and the
algorithm's decisions especially challenging. 

There is relatively little work on the rigorous performance analysis
of Thompson sampling schemes for reinforcement learning. To the best
of our knowledge, the only known regret analyses of Thompson sampling
for reinforcement learning are those of \citet{OsbRusVan13:more} and
\citet{OsbVanRoy14:eluderRL} which study the {\em (purely) Bayesian}
setting, in which nature draws the true MDP episodically from a prior
which is also completely known to the algorithm. The former work
establishes Bayesian regret bounds for Thompson sampling in the
canonical parameterization setup (i.e., each state-action pair having
independent transition/reward parameters) whereas the latter considers
the same for parameterized MDPs as we do here. Our interest, however,
is in the continuous (non-episodic) learning setting, and more
importantly in the {\em frequentist} of regret performance, where the
``prior'' plays the role of merely a parameter used by the algorithm
operating in an unknown, fixed environment. We are also interested in
problem (or ``gap'') dependent $O\left(\log T\right)$ regret bounds
depending on the explicit structure of the MDP parameterization. 

% Need to explain the difficulty with having a continuous
% parameter space; Need to explain when will such approach work;

% how we do it- our approach to derive learning bounds for tsmdp
% involves ...

In this work, we overcome these hurdles to derive the first
regret-type bounds for TSMDP at the level of a general parameter space
and prior. First, we directly consider the posterior density in its
general form of a normalized, exponentiated, empirical
Kullback-Leibler divergence. This is reminiscent of approaches towards
posterior consistency in the statistics literature
\citep{SheWas01:rates,ghosal2000}, but we go beyond it in the sense of
accounting for partial information from adaptively gathered
samples. We then develop self-normalized, maximal concentration
inequalities \citep{delapena2007} for sums of sub-exponential random
variables to Markov chain cycles, which may be of independent interest
in the analysis of MDP-based algorithms. These permit us to show
sample-path based bounds on the concentration of the posterior
distribution, and help bound the number of cycles in which suboptimal
policies are played -- a measure of regret.

%% highlight:
%% \begin{itemize}
%% \item parameterized, parameterized, parameterized!
%% \item TS algorithm independent of analytical bounds, i.e., alg does
%%   not use theoretical bounds; also horizon-free
%% \item Natural way to incorporate dependencies and structure via prior
%% \item unique analysis of dynamics in terms of general posterior
%%   concentration; allows us to track complicated, coupled posteriors,
%%   capture correlation among MDP params; 
%% \item corollary for 'optimal' $\sqrt{T}$ regret?
%% \end{itemize}

\section{Preliminaries}
\label{sec:setup}
Let $\Theta$ be a space of parameters, where each $\theta \in \Theta$
parameterizes an MDP
$m_\theta \bydef (\mathcal{S}, \mathcal{A}, r, p_\theta)$. Here,
$\mathcal{S}$ and $\mathcal{A}$ represent finite state and action
spaces, $r: \mathcal{S} \times \mathcal{A} \to \mathbb{R}$ is the
reward function and
$p_\theta: \mathcal{S} \times \mathcal{A} \times \mathcal{S} \to
[0,1]$
is the probability transition kernel of the MDP (i.e.,
$p_\theta(s_1,a,s_2)$ is the probability of the next state being $s_2$
when the current state is $s_1$ and action $a$ is played). We assume
that the learner is presented with an MDP $m_{\theta^\star}$ where
$\theta^\star \in \Theta$ is initially unknown. In the {\em canonical}
parameterization, the parameter $\theta$ factors into separate
components for each state and action \citep{dearden1999model}.

We restrict ourselves to the case where the reward function $r$ is
completely known, with the only uncertainty being in the transition
kernel of the unknown MDP. The extension to problems with unknown
rewards is well-known from here
\citep{Bartlett2009,tewari08optimistic}.

A (stationary) {\em policy} or {\em control} $c$ is a prescription to
(deterministically) play an action at every state of the MDP, i.e.,
$c: \mathcal{S} \to \mathcal{A}$. Let $\mathcal{C}$ denote the set of
all stationary policies\footnote{Note that $\mathcal{C}$ is finite
  since $\mathcal{S}, \mathcal{A}$ are finite. In general,
  $\mathcal{C}$ can be a {\em subset} of the set of all stationary
  policies, containing optimal policies for every $\theta \in \Theta$.
  This serves to model policies with specific kinds of structure,
  e.g., threshold rules.}  over $(\mathcal{S},\mathcal{A})$, which are
the ``reference policies'' to compete with. Each policy
$c \in \mathcal{C}$, together with an MDP $m_\theta$, induces the
discrete-time stochastic process
$\left(S^{\theta,c}_t,A^{\theta,c}_t,R^{\theta,c}_t\right)_{t=0}^\infty
\equiv (S_t,A_t,R_t)_{t=0}^\infty$,
with $S^{\theta,c}_t$, $A^{\theta,c}_t$ and $R^{\theta,c}_t$ denoting
the state, action taken and reward obtained respectively at time
$t$. In particular, the sequence of visited states
$\left(S^{\theta,c}_t\right)_{t=0}^\infty$ becomes a discrete time
Markov chain.

% \RestyleAlgo{boxruled}
\begin{algorithm2e}[htbp]
  \caption{Thompson Sampling for Markov Decision Processes (TSMDP)}
  \label{alg:tsmdp}
  {\bf Input:} Model space $\Theta$, action space $\mathcal{A}$,
  reward function $r: \mathcal{S} \times \mathcal{A} \to \mathbb{R}$,
  transition kernels $\{p_\theta: \theta \in \Theta\}$, start state
  $s_0 \in \mathcal{S}$.

  {\bf Output:} Action $A_t \in \mathcal{A}$ at each time $t \in \mathbb{Z}^+$.

  {\bf Parameters:} Probability distribution $\pi$ over
  $\Theta$, Sequence of stopping times $t_0 \bydef 0 < t_1 < t_2 < \ldots$ 
  
  {\bf Initialize:} $\pi_0 \assign \pi$, $t \assign 0$, $S_0 = s_0$, $R_0 = 0$. 

  {\bf for} $k = 1, 2, 3, \ldots$ 
  \begin{enumerate}
  \item ({\em Start of epoch $k$}) Sample $\theta_k \in \Theta$
    according to the probability distribution $\pi_{t_k}$.
  \item Set $C_k \assign c^{\mathsf{OPT}}(\theta_k) \equiv \arg \max_{c \in \mathcal{C}}
    \lim_{u \to \infty} \frac{H_{u,\theta_k,c}}{u}$.
  \item {\bf repeat}
    \begin{enumerate}
    \item Play action $A_{t+1} \assign C_k(S_t)$.
    \item Observe $S_{t+1}$, $R_{t+1} \equiv r(S_t,A_{t+1})$.
    \item Update (Bayes Rule): Set the probability distribution
      $\pi_{t+1}$ over $\Theta$ to satisfy
      \begin{equation}
        \label{eqn:bayesrule}
        \forall \theta \quad \pi_{t+1}(d \theta) \propto  p_{\theta}(S_t,A_{t+1},S_{t+1}) \; \pi_t(d\theta).  
      \end{equation}
    \item $t \assign t + 1$.
    \end{enumerate}
    {\bf until} $t = t_k$ ({\em End of epoch $k$}).
  \end{enumerate}
{\bf end for}
\end{algorithm2e}

For each policy $c$, MDP $m_\theta$ and time horizon
$t \in \{0,1,2,\ldots\}$, we define the {\em $t$-step value function}
$H_{t,\theta,c}: \mathcal{S} \to \mathcal{R}$ over initial states to
be
$H_{t,\theta,c}(s) \bydef \mathbb{E}_{\theta,c}\left[\sum_{i=0}^t
  R^{\theta,c}_i \given S_0 = s\right]$,
with the subscripts\footnote{We will often drop subscripts when
  convenient for the sake of clarity in notation.}  $\theta,c$
indicating the stochasticity induced by $c$ in the MDP
$m_\theta$. Denote by
$c^{\mathsf{OPT}}(\theta) \bydef \arg \max_{c \in \mathcal{C}} \lim_{t
  \to \infty} \frac{H_{t,\theta,c}}{t}$
the policy with the best long-term average reward\footnote{We assume
  that the limiting average reward is well-defined. If not, one can
  restrict to the limit inferior.} in $\mathcal{C}$ (ties are assumed
to be broken in a fixed fashion). Correspondingly, let
$\mu^{\mathsf{OPT}}(\theta) \bydef \max_{c \in \mathcal{C}} \lim_{t
  \to \infty} \frac{H_{t,\theta,c}}{t}$
be the best attainable long-term average reward for $\theta$. We will
overload notation and use
$c^\star \equiv c^{\mathsf{OPT}}(\theta^\star)$ and
$\mu^\star \equiv \mu^{\mathsf{OPT}}(\theta^\star)$.

In general, $a(i)$ denotes the $i$th coordinate of the vector $a$, and
$a \cdot b$ is taken to mean the standard inner product
$\sum_i a(i)b(i)$ of vectors $a$ and $b$. Here, $\kldiv{\mu}{\nu}$
denotes the standard Kullback-Leibler divergence
$\sum_{y \in \mathcal{Y}} \mu(y) \log\frac{\mu(y)}{\nu(y)}$ between
probability distributions $\mu$ and $\nu$ on a common finite alphabet
$\mathcal{Y}$. The notation $\mathbbm{1}\{A\}$ is employed to
denote the indicator random variable corresponding to event $A$.

% Define: $V_c(t) \bydef \sum_{i=0}^{t-1} \mathbbm{1}\left\{ C_{e(i)} = c
% \right\}$ is the total number of time instants up to $t$ for which the
% epoch policy $c$ was used.

% ?? define (pseudo) regret formally and try to prove a theorem
% about the regret

%% Our objective is to design a learning algorithm that can obtain the
%% best long-term reward with respect to the reference class of policies
%% $\mathcal{C}$. A learning algorithm $\si$ is a (possibly randomized)
%% rule that chooses an action $A_t$, at each time $t$, depending on the
%% past history $(S_i,A_i,R_i)_{i=0}^{t-1}$.

%\section{The TSMDP algorithm}
%% We analyze in this paper the performance of the following Thompson
%% Sampling for MDPs (TSMDP) algorithm (Algorithm \ref{alg:tsmdp}).
 
{\bf The TSMDP Algorithm.} TSMDP (Algorithm \ref{alg:tsmdp}) operates
in contiguous intervals of time called {\em epochs}, induced in turn
by an increasing sequence of stopping times $t_0, t_1, \ldots$ We will
analyze the version that uses the {\bf return times to the start state
  $s_0$ as epoch markers}, i.e.,
$t_{k} \bydef \min\{t > t_{k-1}: S_t = s_0\}$, $k \geq 1$. The
algorithm maintains a ``prior'' probability distribution (denoted by
$\pi_t$ at time $t$) over the parameter space $\Theta$, from which it
samples\footnote{If the prior is analytically tractable, accurate
  sampling may be feasible. If not, a variety of schemes for sampling
  approximately from a posterior distribution, e.g.,
  Gibbs/Metropolis-Hastings samplers, can be used.} a parameterized
MDP at the beginning of each epoch. It then uses an average-reward
optimal policy w.r.t. $\mathcal{C}$ for the sampled MDP throughout the
epoch
%% \footnote{The average-reward optimal policy may be computed using
%%   relative value iteration, linear programming,
%%   etc. \citep{Put1994:MDP}}
, and updates the prior to a ``posterior'' distribution via Bayes'
rule (\ref{eqn:bayesrule}), effectively at the end of each epoch.

% {\em Remark 1.} It is possible to define Thompson sampling for MDPs with
% sampling from $\Theta$ at an {\em arbitrary} sequence of times (e.g.,
% each time step, once every $\Delta$ time steps, etc.) instead of at
% the recurrence times to $s_0$ as described above. We restrict
% ourselves to the recurrence version for being able to show theoretical
% results, by treating recurrence cycles as being ``independent'' in a
% certain technical sense (see the Appendix for details). Extending the
% results to Thompson sampling with arbitrary sampling times remains a
% challenging task devoted to future work.

% {\em Remark 2.} Thompson sampling is inherently a {\em randomized}
% algorithm since it samples from the posterior distribution, in
% contrast to deterministic algorithms like UCRL2, REGAL, Rmax,
% etc. that are based on the optimistic upper confidence bound
% philosophy.

\section{Assumptions Required for the Main Result}
\label{sec:formal}
We describe in this section our main result for the TSMDP algorithm
(Algorithm \ref{alg:tsmdp}), driven by the intuition presented in
Section \ref{sec:overview}. We begin by stating and explaining the
assumptions needed for our results to hold.

% \begin{assumption}[Discrete, ``Grain of truth'' prior]
%   \label{ass:finite}
%   The support of the ``prior'' probability distribution $\pi$ is
%   finite: $|\Theta| < \infty$. Moreover, $\pi(\theta^\star) > 0$.
% \end{assumption}

% The finite-$\Theta$ assumption has been used in the past to study
% reinforcement learning problems
% \citep{LatHutSun13:sample,AgrTenAna89:markov}. The
% finiteness assumption helps capture the essence of the posterior
% concentration around the true parameter $\theta^\star$ without undue
% technical complications. As such, the TS algorithm does not require a
% finite parameter space to be specified and executed. The only
% requirement is to be able to sample from a posterior.

% We believe that even for continuous parameter spaces (e.g.,
% $\mathbb{R}^d$) and suitable priors with enough prior mass in
% neighborhoods of $\theta^\star$, the result will continue to hold in
% spirit. However, the analysis is likely to be much more involved as
% the normalization factor of the posterior becomes a continuous
% integral, requiring technical tools like the Laplace approximation
% \citep{TieKad86:accurate} to be applied. We thus defer this challenging
% technical task to future work.

\begin{assumption}[Recurrence]
  \label{ass:ergodic}
  The start state $s_0$ is {\em recurrent}\footnote{Recall that a
    state $s$ is said to be recurrent in a discrete time Markov chain
    $X_1, X_2, X_3, \ldots$ if
    $\prob{\min\{t \geq 1: X_t = s \} < \infty \given X_0 = s} = 1$
    \citep{LevinPeresWilmer2006}.}  for the true MDP
  $m_{\theta^\star}$ under each policy
  $c^{\mathsf{OPT}}(\theta) \in \mathcal{C}$ for $\theta$ in the
  support of $\pi$.
\end{assumption}

Assumption \ref{ass:ergodic} is satisfied, for instance, if
$m_{\theta^\star}$ is an ergodic\footnote{A Markov chain is ergodic if
  it is irreducible, i.e., it is possible to go from every state to
  every state (not necessarily in one move)} Markov chain under every
stationary policy -- a condition commonly used in prior work on MDP
learning \citep{tewari08optimistic,BurKat97:optimal}. Define
$\bar{\tau}_c$ to be the expected recurrence time to state $s_0$,
starting from $s_0$, when policy $c$ is used in the true MDP
$m_{\theta^\star}$.

% It is required to split
% the sample path of states under TSMDP into ``IID-like'' cycles with a
% connection to the analysis of multi-armed bandits, where concentration
% results can then be applied.

\begin{assumption}[Bounded Log-likelihood ratios]
  \label{ass:abscon}
  Log-likelihood ratios are upper-bounded by a constant
  $\Gamma < \infty$:
  $\forall \theta \in \Theta \; \forall (s_1,s_2,a) \in \mathcal{S}
  \times \mathcal{S} \times \mathcal{A}: \pi(\theta) > 0 \Rightarrow
  \left| \log
    \frac{p_{\theta^\star}(s_1,a,s_2)}{p_{\theta}(s_1,a,s_2)} \right|
  \leq \Gamma$.

%%   \begin{align*} 
%%     \Gamma \bydef &\sup \; \Big\{ \Big| \log
%%       \frac{p_{\theta^\star}(s_1,a,s_2)}{p_{\theta}(s_1,a,s_2)}
%%     \Big|: \theta \in \Theta, \\
%%  &(s_1,s_2,a) \in \mathcal{S} \times
%%     \mathcal{S} \times \mathcal{A}, p_{\theta^\star}(s_1,a,s_2) > 0
%%   \Big\} < \infty. 
%% \end{align*}
\end{assumption}
Assumption \ref{ass:abscon} is primarily technical, and helps control
the convergence of sample KL divergences in $\Theta$ to (expected)
true KL divergences, and is commonly employed in the
statistics literature, e.g., \citep{SheWas01:rates}. 
% It is possible to
% relax this with a constraint on bracketing entropy, but we prefer
% working in a simpler setup to capture the essential dynamics.

\begin{assumption}[Unique average-reward-optimal policy]
\label{ass:unique}
For the true MDP $m_{\theta^\star}$,
$c^\star \equiv c^{\mathsf{OPT}}(\theta^\star)$ is the {\em unique}
average-reward optimal policy:
$c \neq c^\star \Rightarrow \lim_{t \to \infty}
\frac{H_{t,\theta^\star,c}}{t} < \lim_{t \to \infty}
\frac{H_{t,\theta^\star,c^\star}}{t}$.
\end{assumption}
The uniqueness assumption is made merely for ease of exposition;
our results continue to hold with suitable redefinition otherwise.

The remaining assumptions (\ref{ass:largeprob} and
\ref{ass:largedenominator}) concern the behavior of the prior and the
posterior distribution under ``near-ideal'' trajectories of the
MDP. In order to introduce them, we will need to make a few
definitions. Let $\pi^{(c)}_{s_1}$ (resp.  $\pi^{(c)}_{s_1,s_2}$) be
the stationary probability of state $s_1$ (resp. joint probability of
$s_1$ immediately followed by $s_2$) when the policy $c$ is applied to
the true MDP $m_{\theta^\star}$; correspondingly, let
$\bar{\tau}_c \bydef 1/\pi^{(c)}_{s_1}$ be the expected first return
time to state $s_0$.%
%
%   Let us now introduce a key notion of dissimilarity over the
%   parameter space $\Theta$, which we call {\em marginal divergence}.
%
% \begin{definition}
%   \label{def:mardiv}
%   The marginal divergence of parameter $\theta$, with respect to the
%   true parameter $\theta^\star$ under policy $c$, is defined to be
%   \begin{align*}
%     D_c(\theta^\star || \theta) &\bydef \sum_{s_1 \in \mathcal{S}}
%     \pi_{s_1}(\theta^{\star},c) \sum_{s_2 \in \mathcal{S}}
%     p_{\theta^\star}(s_1,c(s_1),s_2) \log
%     \frac{p_{\theta^\star}(s_1,c(s_1),s_2)}{p_{\theta}(s_1,c(s_1),s_2)},
%     \quad D(\theta^\star || \theta) \bydef (D_c(\theta^\star || \theta))_{c \in \mathcal{C}}.
%   \end{align*}
% \end{definition}
%
% In words, $D_c(\theta^\star || \theta)$ measures the Kullback-Leibler divergence
% between corresponding rows in the state transition probability
% matrices for the MDPs $m_{\theta^\star}$ and $m_{\theta}$ under the
% stationary rule $c$, weighted by stationary probabilities in
% $m_{\theta^\star}$. If $D_c(\theta^\star || \theta)$ is positive, then the MDPs
% $m_{\theta}$ and $m_{\theta^\star}$ can be ``resolved apart'' using
% the policy $c$. Note that $D(\theta^\star) = 0$ by definition. \\
%
We denote by $D_c(\theta^\star || \theta)$ the important {\em marginal
  Kullback-Leibler divergence}\footnote{The marginal KL divergence
  appears as a fundamental quantity in the {\em lower} bound for
  regret in parameterized MDPs established by
  \citep{AgrTenAna89:markov}.}  for $\theta$ under $c$:
\begin{align*} 
  D_c(\theta^\star || \theta) &\bydef \sum_{s_1 \in \mathcal{S}}
  \pi^{(c)}_{s_1} \sum_{s_2 \in \mathcal{S}}
  p_{\theta^\star}(s_1,c(s_1),s_2) \log
  \frac{p_{\theta^\star}(s_1,c(s_1),s_2)}{p_{\theta}(s_1,c(s_1),s_2)} \\
  &= \sum_{s_1 \in \mathcal{S}} \pi^{(c)}_{s_1}
  \; \kldiv{p_{\theta^\star}(s_1,c(s_1),\cdot)}{p_{\theta}(s_1,c(s_1),\cdot)}.
\end{align*}

The marginal KL divergence $D_c(\theta^\star || \theta)$ is a convex
combination of the KL divergences between the transition probability
kernels of $m_{\theta^\star}$ and $m_\theta$, with the weights of the
convex combination being the appropriate invariant probabilities
induced by policy $c$ under $m_{\theta^\star}$. If
$D_c(\theta^\star || \theta)$ is positive, then the MDPs $m_{\theta}$
and $m_{\theta^\star}$ can be ``resolved apart'' using samples from
the policy $c$. % Note that (a) $D_c(\theta^\star || \theta) \geq 0$
%and (b) $ D_c(\theta^\star || \theta^\star) = 0$, $\forall c$. 
Denote
$D(\theta^\star || \theta) \bydef (D_c(\theta^\star || \theta))_{c\in
  \mathcal{C}}$,
i.e., the vector of $D_c(\theta^\star || \theta)$ values across all
policies, with the convention that the final coordinate is associated with the optimal policy $c^\star$.

For each policy $c$, define
$S_c \bydef \{\theta \in \Theta: c^{\mathsf{OPT}}(\theta) = c\}$ to be
the {\em decision region} corresponding to $c$, i.e., the set of
parameters/MDPs for which the average-reward optimal policy is
$c$. Fixing $\epsilon' \geq 0$, let
$S_c' \equiv S_c'(\epsilon') \bydef \{\theta \in S_c:
D_{c^\star}(\theta^\star || \theta) \leq \epsilon' \}$.
In other words, $S_c'$ comprises all the parameters (resp. MDPs) with
average reward-optimal policy $c$ that ``appear similar'' to
$\theta^\star$ (resp. $m_{\theta^\star}$) under the true optimal
policy $c^\star$. Correspondingly, put
$S_c'' \equiv S_c''(\epsilon') \bydef S_c \setminus S_c'$ as the
remaining set of parameters (resp. MDPs) in the decision region $S_c$
that are separated by at least $\epsilon'$ w.r.t. $D_{c^\star}$.

Let us use $e(t)$ to denote the epoch to which time instant $t$
belongs, i.e., $e(t) \bydef k$ if
$t \in \{t_{k-1} + 1, t_{k-1} + 2, \ldots, t_k \}$.  Let
$N_c(k) \bydef \sum_{l=1}^k \mathbbm{1}\{\theta_l \in S_c\}$ be the
number of epochs, up to and including epoch $k$, in which the policy
applied by the algorithm was $c$.  Let $J_{(s_1,s_2)}(k,c)$ denote the
total number of time instants that the state transition $s_1 \to s_2$
occurred in the first $k$ epochs when policy $c$ was used, i.e.,
$J_{(s_1, s_2)}(k,c) \bydef \sum_{t=1}^\infty \mathbbm{1}\{C_{e(t)} =
c, (S_t,S_{t+1}) = (s_1,s_2), N_c(e(t)) \leq k\}$.

The next assumption controls the posterior probability of playing the
true optimal policy $c^\star$ during any epoch, preventing it from
falling arbitrarily close to $0$. Note that at the beginning of epoch
$k$ (time instant $t_k$), the posterior measure
$\pi_{t_k}(\mathcal{M})$ of any legal subset
$\mathcal{M} \subseteq \Theta$ can be expressed solely as a function
of the sample state pair counts $J_{(\cdot,\cdot)}(\cdot,\cdot)$ as
\begin{align*}
  \pi_{t_k}(\mathcal{M}) &= \frac{\int_{\mathcal{M}} W_{t_k}(\theta) \pi(d\theta)}{\int_\Theta W_t(\theta) \pi(d\theta)}, \quad  W_{t_k}(\theta) \bydef \exp \sum_{c,s_1,s_2} J_{(s_1,s_2)}\left(N_c(k), c\right) \log
    \frac{p_\theta(s_1,c(s_1),s_2)}{p_{\theta^\star}(s_1,c(s_1),s_2)},
\end{align*}
where $W_{t_k}(\theta)$ represents the posterior density or {\em
  weight} at time $t_k$. The assumption requires that the posterior
probability of the decision region of $c^\star$ is uniformly bounded
away from $0$ whenever the empirical state pair frequencies
$\frac{J_{(s_1,s_2)}\left(N_c(k), c\right)}{N_c(k)}$ are ``near''
their corresponding expected\footnote{Expectation w.r.t. the state
  transitions of $m_{\theta^\star}$} values
$\bar{\tau}_c \; \pi^{(c)}_{(s_1,s_2)}(\theta^\star, c)$.

\begin{assumption}[Posterior probability of the optimal policy under
  ``near-ideal'' trajectories]
  \label{ass:largeprob}
  For any scalars $e_1, e_2 \geq 0$, there exists
  $p^\star \equiv p^\star(e_1,e_2) > 0$ such that
  \[ \pi_{t_k}(S_{c^\star}) \geq p^\star \quad \mbox{whenever
    ``near-ideal" state pair frequencies have been observed:} \]
  \[ \left| \frac{J_{(s_1,s_2)}(k_c,c)}{k_c} - \bar{\tau}_c \;
    \pi^{(c)}_{(s_1,s_2)} \right| \leq \sqrt{\frac{e_1 \log\left({e_2 \log
          k_c} \right)}{k_c}} \quad \forall s_1, s_2 \in \mathcal{S},
  k_c \geq 1, c \in \mathcal{C}, k = \sum_{c \in \mathcal{C}}
  k_c. \]
\end{assumption}

% \begin{assumption}[Large ``idealized posterior'' mass on the true
%   optimal policy]
%   \label{ass:largeprob}
%   There exists $p^\star > 0$ such that 
%   \[ \frac{\int_{S_{c^\star}} \exp \left[ -\sum_{c \in \mathcal{C}}
%       k_c \bar{\tau}_c D_c(\theta^\star || \theta) \right] \pi(d\theta)}{\int_{\Theta}
%     \exp \left[ -\sum_{c \in \mathcal{C}} k_c \bar{\tau}_c D_c(\theta^\star || \theta)
%     \right] \pi(d\theta)} \geq p^\star \] for all choices of
%   nonnegative integers $k_c \geq 0$.
% \end{assumption}

% \begin{assumption}[``Slow enough'' decay of the posterior normalizing
%   constant]
%   \label{ass:largedenominator}
%   \hfill \\ 
%   {\bf (A)} There exist $a_1 > 0, a_2 \geq 0$ such that $\int_{\Theta}
%   W_{t_k}(\theta') \pi(d\theta') \geq a_1 k^{-a_2}$ whenever
%   \[ \left| J_{(s_1,s_2)}(k_c,c) - V_c(t_k)\; \pi_{(s_1,s_2)} \right|
%   \leq \rho(k_c)\sqrt{k_c} \quad \forall s_1, s_2 \in \mathcal{S}, c
%   \in \mathcal{C}, k_c \geq 0, \sum_{c \in \mathcal{C}} k_c = k, k
%   \geq 1. \] Moreover, \\ {\bf (B)} There exist $a_3 > 0, a_4 > 0$
%   such that if the condition above holds whenever $k_{c^\star} \geq k
%   - \log^2(k)$, then $\int_{\Theta} W_{t_k}(\theta') \pi(d\theta')
%   \geq a_3 k^{-a_4}$.
% \end{assumption}

The final assumption we make is a ``grain of truth'' condition on the
prior, requiring it to put sufficient probability on/around the true
parameter $\theta^\star \in \Theta$. Specifically, we require that
prior probability mass in {\em weighted} marginal KL-neighborhoods of
$\theta^\star$ to not decay too fast as a function of the total
weighting. This form of local prior property is analogous to the {\em
  Kullback-Leibler condition}
\citep{Bar98:informationtheoretic,ChoRam08:consistency,GhoGhoRam99:dirichlet}
used to establish consistency of Bayesian procedures, and in fact can
be thought of as an extension of the standard condition to the partial
observations setting of this paper.
\begin{assumption}[Prior mass on KL-neighborhoods of $\theta^\star$]
  \label{ass:largedenominator}
  \hfill \\
  {\bf (A)} There exist $a_1 > 0, a_2 \geq 0$ such that
  $ \pi\left(\left\{ \theta \in \Theta: \sum_{c \in \mathcal{C}} k_c \bar{\tau}_c
      D_c(\theta^\star || \theta) \leq 1 \right\} \right) \geq a_1 k^{-a_2}$,  for
  all choices of nonnegative integers $k_c$, and $k = \sum_{c \in
    \mathcal{C}} k_c$. \\
   
 \noindent {\bf (B)} There exist $a_3 > 0, a_4 > 0$
  such that 
  $\pi\left(\left\{ \theta \in \Theta: \sum_{c \in \mathcal{C}} k_c \bar{\tau}_c
      D_c(\theta^\star || \theta) \leq 1 \right\} \right) \geq a_3 k^{-a_4}$, for
  all choices of nonnegative integers $k_c$, $k = \sum_{c \in
    \mathcal{C}} k_c$, that satisfy $k_{c^\star} \geq k - 3\log^2(k)$.
\end{assumption}

The key factor that will be shown to influence the regret scaling with
time is the quantity $a_4$ above, which bounds the (polynomial) decay
rate of the prior mass around essentially the marginal KL neighborhood
of $\theta^\star$ corresponding to always playing the policy
$c^\star$. 

We show later how these assumptions are satisfied in finite parameter
spaces (Section \ref{sec:appln1}) , and in continuous parameter spaces
(Section \ref{sec:continuous}). In particular, in finite parameter
spaces, the assumptions can be shown to be satisfied with
$a_2 = a_4 = 0$ while for smooth (continuous) priors, the typical
square-root rate of $1/2$ per independent parameter dimension holds,
i.e., $a_4 \leq \frac{1}{2} \# \mbox{(indpt. parameter dimensions)}$
holds.

%   \begin{assumption}
%     \label{ass:llkl}
%     The probability transition kernels $p_\theta$, $\theta \in \Theta$
%     are such that there exists $g < \infty$ with
%     \[\sup_{\substack{\theta \in \Theta, c \in \mathcal{C}\\D_c(\theta^\star || \theta) > 0}} \frac{\sum_{(s_1,s_2) \in
%         \mathcal{S}^2} \left| \log
%         \frac{p_{\theta^\star}(s_1,c(s_1),s_2)}{p_{\theta}(s_1,c(s_1),s_2)}
%       \right|}{\sqrt{D_c(\theta^\star || \theta)}} \leq g.\]
%   \end{assumption}

% \begin{assumption} 
%   \label{ass:selfdistance}
%   There exists $\epsilon_{\min} > 0$ such that
%   \[\forall c \neq c^\star \; \inf_{\theta \in S_c'} D_c(\theta^\star || \theta) \geq \epsilon_{\min}. \]
% \end{assumption}

\section{Main Result}
\label{sec:mainresult}
We are now in a position to state\footnote{Due to space constraints,
  the proofs of all results are deferred to the appendix.} the main,
top-level result of this paper.
\begin{thm}[Regret-type bound for TSMDP]
  \label{thm:main}
  Suppose Assumptions \ref{ass:ergodic} through
  \ref{ass:largedenominator} hold. Let $\epsilon, \delta \in (0,1)$,
  and let $c^\star$ be the unique optimal stationary policy for the
  true MDP $m_{\theta^\star}$. For the TSMDP algorithm, there exists
  $T_0 \equiv T_0(\epsilon) > 0$ such that with probability at least
  $1-\delta$, it holds for all $T \geq T_0$ that
  \begin{equation}
    \label{eqn:main2}
    \sum_{t=1}^T \mathbbm{1}\{A_t \neq c^\star(S_t)\} \leq \mathsf{B}
    + \mathsf{C} \log T,
  \end{equation}

  where $\mathsf{B} = \mathsf{B}(\delta,m_{\theta^\star},\pi)$ is a
  problem- and prior-dependent quantity independent of $T$, and
  $\mathsf{C}$ is the value of the optimization problem\footnote{Note
    that $a_4$ in (\ref{eqn:deftau}) is the constant from Assumption
    \ref{ass:largedenominator}(B).}
  \begin{equation}
    \label{eqn:main}
    \begin{aligned}
      &\max && \left|\left| x_{|\mathcal{C}|-1}\right|\right|_1 \\
      &\text{ s.t.}
      && x_l \in \mathbb{R}_+^{|\mathcal{C}|}, \quad \forall l = 1, 2, \ldots, |\mathcal{C}|-1, \\
      &&& x_l(|\mathcal{C}|) = 0, \quad \forall l = 1, 2, \ldots, |\mathcal{C}|-1, \\
      &&& x_i \geq x_j, \quad \forall 1 \leq j \leq i \leq |\mathcal{C}| - 1, \\
      &&& x_i(l) = x_l(l), \quad \forall i \geq l,  l = 1, 2, \ldots, |\mathcal{C}|-1, \\
      &&& \sigma: \{1, 2, \ldots, |\mathcal{C}|-1 \} \to \mathcal{C} \setminus \{c^\star\} \; \mbox{\emph{injective}}, \\
      &&& \min_{\theta \in S_{\sigma(l)}'} \quad x_l \cdot D(\theta^\star || \theta) =
      (1+a_4)\left(\frac{1+\epsilon}{1-\epsilon}\right), \quad \forall 1 \leq l \leq
      |\mathcal{C}|-1.
    \end{aligned}
  \end{equation}
\end{thm}

%\subsection{Discussion and Interpretation of the Result}

% explain finer details of theorem, connect back to intuition
% explain technical methods used in proof, redirect to appendix

\noindent {\bf Discussion.} Theorem \ref{thm:main} gives a high-probability,
logarithmic-in-$T$ bound on the quantity\\ $\sum_{t=1}^T
\mathbbm{1}\{A_t \neq c^\star(S_t)\}$, the number of time
instants in $1, 2, \ldots, T$ when a suboptimal choice of action
(w.r.t. $c^\star$) is made.
%% In other words, it says that actions are
%% chosen according to the optimal average-reward policy $c^\star$ all but
%% a logarithmic-in-$T$ number of times. 
This can be interpreted as a natural regret-minimization property of
the algorithm\footnote{In the case of a stochastic multi-armed bandit
($|\mathcal{S}| = 1$ and $r: \mathcal{A} \to \mathbb{R}$ IID across
time) with rewards bounded in $[0,1]$, for instance, this quantity
serves as an upper bound to the standard {\em pseudo
  regret}\footnote{A bound on a suitably defined version of pseudo
  regret - see e.g., \citet{JakschOA10} - can easily be obtained from
  our main result (Theorem \ref{thm:main}) by appropriate weighting;
  we leave the details to the reader.} \citep{AudBub10:minimax},
defined as
$\sum_{t=1}^T \left(\expect{r(a^\star) - r(A_t)} \right)
\mathbbm{1}\{A_t \neq a^\star\}$,
with $a^\star \bydef \arg\max_{a \in \mathcal{A}} \expect{r(a)}$}. %
%% The optimization problem (\ref{eqn:main}) provides a refined bound for
%% the logarithmic scaling of the number of suboptimal plays, in terms of
%% the marginal KL divergences $\{D_c(\theta^\star || \theta): c \in \mathcal{C}, \theta
%% \in \Theta\}$ for the parameterized MDP setup in question. With regard
%% to the intuition presented in Section \ref{sec:overview}, the crux of
%% (\ref{eqn:main}) is in calculating how suboptimal models in the
%% regions $S_c'$ can ``persist'' the longest in the sense of maintaining
%% a posterior probability greater than $O\left(1/T\right)$ (equivalent
%% to the negative exponent of the posterior density being at most $\log
%% T$). 
The optimization problem (\ref{eqn:main}) and the bound
(\ref{eqn:main2}) can be interpreted as a multi-dimensional ``game''
in the space of (epoch) play counts of policies $c \in \mathcal{C}$,
with the following ``rules'': %
% \begin{enumerate}
% \item Start growing the non-negative $|\mathcal{C}|$-dimensional
%   vector $z$ of epoch play counts of all policies, with initial value
%   $(0, 0, \ldots, 0)$ (the $|\mathcal{C}|$-th coordinate of $z$
%   represents the number of plays of the optimal policy $c^\star$,
%   which is irrelevant as far as regret is concerned, and is thus
%   pegged to $0$ throughout),
% \item Wait until the first time that some suboptimal policy
%   $c \neq c^\star$ is ``eliminated'', in the sense
%   $z \cdot D(\theta^\star || \theta) \approx \log T$
%   $\forall \theta \in S_{c}'$
% \item Record $\sigma(1) = c$, $z_1 = z$,
% \item Impose the constraint that no further growth is allowed to occur
%   in $z$ along dimension $c$ in the future,
% \item Repeat growing the play count vector $z$ until the time all
%   suboptimal policies $c \neq c^\star$ are eliminated, and aim to
%   maximize the final $||z||_1$ when this occurs.
% \end{enumerate}
{\bf (1)} Start growing the non-negative $|\mathcal{C}|$-dimensional
vector $z$ of epoch play counts of all policies, with initial value
$(0, 0, \ldots, 0)$ (the $|\mathcal{C}|$-th coordinate of $z$
represents the number of plays of the optimal policy $c^\star$, which
is irrelevant as far as regret is concerned, and is thus pegged to $0$
throughout), {\bf (2)} Wait until the first time that some suboptimal
policy $c \neq c^\star$ is ``eliminated'', in the sense
$z \cdot D(\theta^\star || \theta) \approx \log T$
$\forall \theta \in S_{c}'$, {\bf (3)} Record $\sigma(1) = c$,
$z_1 = z$, {\bf (4)} Impose the constraint that no further growth is
allowed to occur in $z$ along dimension $c$ in the future, and {\bf
  (5)} Repeat growing the play count vector $z$ until the time all
suboptimal policies $c \neq c^\star$ are eliminated, and aim to
maximize the final $||z||_1$ when this occurs. An overview of how this
optimization naturally arises as a regret bound for Thompson sampling
is provided in Section \ref{sec:overview}.

% (a) start growing the non-negative
% $\mathcal{C}$-dimensional vector $z$ of play counts of all policies
% from the all-zeros point, (b) wait until the first time that some
% suboptimal policy $c \neq c^\star$ is ``eliminated'', in the sense
% $z \cdot D(\theta^\star || \theta) \approx \log T$
% $\forall \theta \in S_{c}'$, (c) set $\sigma(1) = c$, $z_1 = z$, (d) 
% no further growth is allowed to occur in $z$ along dimension $c$ in
% future, (d) .

%% The ingredients in proving the general theorem above include the use
%% of a maximal, self-normalizing inequality to control the tails of sums
%% of sub-exponential random variables (i.e., cycle times in an ergodic
%% MDP and related random variables). This, in turn, helps to give a
%% uniform, sample-path concentration result for empirical
%% KL-divergences, which can be used to track the shrinkage of the
%% posterior in terms of the play counts of different policies in
%% $\mathcal{C}$.
% \section{Overview of Analytical Techniques}

% compare to UCRL constants and show improvement over flat bandit? show
% coupled nature? even though \#variables is $|\mathcal{C}|$, constant
% could be much smaller ...

We also have the following square-root scaling for the usual notion of
regret for MDPs \citep{JakschOA10}:
\begin{thm}[Regret bound for TSMDP]
\label{thm:regret}
Under the hypotheses of Theorem \ref{thm:main}, with
$0 < \delta \leq 1$, for the TSMDP algorithm, there exists $T_1 > 0$
such that with probability at least $1-2\delta$, for all $T \geq T_1$,
$T\mu^\star - \sum_{t=1}^T r(S_t,A_t) =
O\left(\sqrt{\frac{T}{\bar{\tau}_{c^\star}} \log \left(\frac{\log
        T}{\delta} \right)}\right)$.
\end{thm}
This can be compared with the probability-at-least
$(1-\delta)$ regret bound of 
$O\left( \mathcal{D} |\mathcal{S}|
  \sqrt{|\mathcal{A}|T\log\left(\frac{T}{\delta} \right)} \right)$
for UCRL2 \citep[Theorem 4]{JakschOA10}, with $\mathcal{D}$ being the
{\em diameter}\footnote{The diameter D is the time it takes to move
  from any state $s$ to any other state $s'$, using an appropriate
  policy for each pair of states $s,s'$.} of the true MDP.

% Due to space constraints, the proofs of Theorems \ref{thm:main} and
% \ref{thm:regret} are provided in the Appendix.

The following sections show how the conclusions of Theorem
\ref{thm:main} are applicable to various MDPs and illustrate the
behavior of the scaling constant $\mathsf{C}$, showing that
significant gains are obtained in the presence of correlated
parameters.

\subsection{Application: Discrete Parameter Spaces}
\label{sec:appln1}
We show here how the conclusion of Theorem \ref{thm:main} holds in a
setting where there the true MDP is known to be one among finitely
many candidate models (MDPs).

\begin{assumption}[Finitely many parameters, ``Grain of truth''
  prior]
  \label{ass:finite}
  The prior probability distribution $\pi$ is supported on finitely
  many parameters: $|\Theta| < \infty$. Moreover,
  $\pi(\{\theta^\star\}) > 0$.
\end{assumption}

\begin{thm}[Regret-type bound for TSMDP, Finite parameter setting]
  \label{thm:mainfinite}
  Suppose Assumptions \ref{ass:ergodic}, \ref{ass:abscon},
  \ref{ass:unique} and \ref{ass:finite} hold. Then, with
  $\epsilon' = 0$, (a) Assumption \ref{ass:largeprob} holds, and (b)
  Assumption \ref{ass:largedenominator} holds with $a_2 = 0$ and
  $a_4 = 0$.  Consequently, the conclusion of Theorem \ref{thm:main}
  holds, namely: Let $\epsilon, \delta \in (0,1)$, and let $c^\star$
  be the unique optimal stationary policy for the true MDP
  $m_{\theta^\star}$. For the TSMDP algorithm, there exists
  $T_0 \equiv T_0(\epsilon) > 0$ such that with probability at least
  $1-\delta$, it holds for all $T \geq T_0$ that
  $\sum_{t=1}^T \mathbbm{1}\{A_t \neq c^\star(S_t)\} \leq \mathsf{B} +
  \mathsf{C} \log T$,
  where $\mathsf{B} = \mathsf{B}(\delta,m_{\theta^\star},\pi)$ is a
  problem- and prior-dependent quantity independent of $T$, and
  $\mathsf{C}$ is the value of the optimization problem
  (\ref{eqn:main}) with $a_4 = 0$.

% \footnote{Note
%     that $a_4$ in (\ref{eqn:deftau}) is the constant from Assumption
%     \ref{ass:largedenominator}(B).}
%   \begin{equation}
%     \label{eqn:mainfinite}
%     \begin{aligned}
%       &\max && \sum_{l=1}^{|\mathcal{C}|-1} x_l(l)\\
%       &\text{ s.t.}
%       && x_l \in \mathbb{R}_+^{|\mathcal{C}|}, \quad \forall l = 1, 2, \ldots, |\mathcal{C}|-1, \\
% %      &&& x_l(|\mathcal{C}|) = 0, \quad \forall l = 1, 2, \ldots, |\mathcal{C}|-1, \\
%       &&& x_i(l) = x_l(l), \quad \forall i \geq l,  l = 1, 2, \ldots, |\mathcal{C}|-1, \\
%       &&& x_i \geq x_j, \quad \forall 1 \leq j \leq i \leq |\mathcal{C}| - 1, \\
%       &&& \sigma: \{1, 2, \ldots, |\mathcal{C}|-1 \} \to \mathcal{C} \setminus \{c^*\} \; \mbox{\emph{injective}}, \\
%       &&& \min_{\theta \in S_{\sigma(l)}'} \quad x_l \cdot D(\theta^\star || \theta) =
%       (1+a_4)\left(\frac{1+\epsilon}{1-\epsilon}\right), \quad \forall 1 \leq l \leq
%       |\mathcal{C}|-1.
%     \end{aligned}
%   \end{equation}
\end{thm}

% The proof of this result appears in the Appendix. 

\subsection{Application: Continuous Parameter Spaces}
\label{sec:continuous}

To illustrate the generality of our result, we apply our main result
(Theorem \ref{thm:main}) to obtain a regret bound for Thompson
Sampling with a {\em continuous} prior, i.e.,
$\Theta \in \mathbb{R}^p$, and $\pi$ a probability density\footnote{By
  a probability density on $\mathbb{R}^p$, we mean a probability
  measure absolutely continuous w.r.t. Lebesgue measure on
  $\mathbb{R}^p$.} on $\mathbb{R}^p$. For ease of exposition, let us
consider a $2$-state, $2$-action MDP: $\mathcal{S} = \{1, 2\}$,
$\mathcal{A} = \{1, 2\}$ (the theory can be applied in general to
finite-state, finite-action MDPs). The (known) reward in state $s_i$
is $r_i$, $i \in \{1,2\}$, irrespective of the action played, i.e.,
$r(i, a) = r_i$, $\forall i \in \{1,2\}, a \in \mathcal{A}$, with
$r_1 < r_2$. All the uncertainty is in the transition kernel of the
MDP, parameterized by the canonical parameters
$\left(p(1, a, 2), p(2, a, 1)\right)_{a = 1, 2}$. Hence, we take the
parameter space to be $\Theta = [0,1]^4$, with the
identification\footnote{Note that we retain only $4$ {\em independent}
  parameters of the MDP model.}
$\theta = \left(\theta^{(1)}_{12}, \theta^{(1)}_{21},
  \theta^{(2)}_{12}, \theta^{(2)}_{21}\right) \in \Theta$
and $\theta^{(i)}_{jl} = p_{\theta}(j, i, l)$ $\forall i,j,l$. It
follows that the optimal policy for a parameter $\theta$ is one that
maximizes the probability of staying at state $2$:
\[ c^{\mathsf{OPT}}(\theta) \equiv \left(c(1), c(2)\right) = ({j_1},
{j_2}), \quad j_1 = \arg\max_{i} \theta^{(i)}_{12}, \quad j_2 =
\arg\min_{i} \theta^{(i)}_{21}. \]
Imagine that the TSMDP algorithm is run with initial/recurrence state
$1$ and prior $\pi$ as the uniform density on the sub-cube
$[\upsilon, 1-\upsilon]^4$, $0 < \upsilon < 1/2$ on the MDP
$m_{\theta^\star}$, $\theta^\star \in \Theta$. Also, without loss of
generality, let
$\upsilon < \theta^{\star(2)}_{12} < \theta^{\star(1)}_{12} <
1-\upsilon, \quad \upsilon < \theta^{\star(1)}_{21} <
\theta^{\star(2)}_{21} < 1-\upsilon$,
implying that
$c^\star \equiv c^{\mathsf{OPT}}(\theta^\star) = (1, 1)$, i.e., the
optimal policy is to always play action $1$. It can be checked that
under this setup, Assumptions \ref{ass:ergodic}, \ref{ass:abscon} and
\ref{ass:unique} hold. The following result establishes the validity
of Assumptions \ref{ass:largeprob} and \ref{ass:largedenominator} in
this continuous prior setting.

\begin{thm}[Regret-type bound for TSMDP, Continuous parameter/prior
  setting]
  \label{thm:maincontinuous}
  In the above MDP, with $\epsilon' > 0$ small enough, (a) Assumption
  \ref{ass:largeprob} holds, and (b) Assumption
  \ref{ass:largedenominator} holds with $a_2 = 2$ and $a_4 = 1$.
  Consequently, the conclusion of Theorem \ref{thm:main} holds.
\end{thm}

% \subsection{Application: MDPs under the ``Tabula Rasa'' or Canonical
%   Parameterization}
% \label{sec:appln2}

% \subsection{Application: Stochastic Multi-Armed Bandits }
% \label{sec:appln3}

\subsection{Dependence of the Regret Scaling on MDP and Parameter
  Structure}
\label{sec:scalingconst}

We derive the following consequence of Theorem \ref{thm:main}, useful
in its own right, that explicitly guarantees an improvement in regret
directly based on the Kullback-Leibler resolvability of parameters in
the parameter space -- a measure of the coupling across policies in the
MDP.

\begin{thm}[Explicit Regret Improvement due to shared Marginal
  KL-Divergences]
  \label{thm:atleastLmain}
  % Let $T$ be large enough such that $\max_{\theta \in \Theta, a \in
  % \mathcal{A}} D(\theta^\star_a || \theta_a) \leq
  % \frac{1+\epsilon}{1-\epsilon}\log T$. 
  Suppose that $\Delta > 0$ and the integer $L \in \mathbb{Z}^+$ are such that
\[\forall c \neq c^\star, \theta \in S_c' \;  |\{\hat{c} \in \mathcal{C}: \hat{c} \neq c^\star,
D_{\hat{c}}(\theta^\star || \theta) \geq \Delta \}| \geq L, \]
i.e., at least $L$ coordinates\footnote{Note that the coordinate
  corresponding to the optimal policy $c^\star$ is excluded from the
  condition.} of $D(\theta^\star || \theta)$ are at least
$\Delta$. Then, the multiplicative scaling factor $\mathsf{C}$ in
(\ref{eqn:main2}) satisfies %
$ \mathsf{C}\leq \left(\frac{|\mathcal{C}| - L}{\tilde{\Delta}}\right)
\frac{2(1+a_4)(1+\epsilon)}{1-\epsilon}$,%
where
$\tilde{\Delta} \bydef \min\left\{\Delta, \min_{c \neq c^\star, \theta
    \in S_c'} D_c(\theta^\star || \theta) \right\}$.
\end{thm}
The result assures a non-trivial additive reduction of
$\Omega\left(\frac{L}{\Delta} \log T\right)$ from the naive decoupled
regret, whenever any suboptimal model in $\Theta$ can be resolved
apart from $\theta^\star$ by at least $L$ actions in the sense of
marginal KL-divergences of their observations. 

Although the net number of decision vectors $x_l$ in (\ref{eqn:main})
is nearly $|\mathcal{C}| = O(|\mathcal{A}|^{\mathcal{S}})$, the scale
of $\mathsf{C}$ can be {\em significantly less} than the number of
policies $|\mathcal{C}|$ owing to the fact that the posterior
probability of several parameters is driven down simultaneously via
the marginal K-L divergence terms $D(\theta^\star || \theta)$. Put
differently, using a standard bandit algorithm (e.g., UCB) naively
with each arm being a stationary policy will perform much worse with a
scaling like $|\mathcal{C}|\log(T)$. We show {\bf (Appendix
  \ref{app:queueing})} an example of an MDP in which the number of
states can be arbitrarily large but which has only one uncertain
scalar parameter, for which Thompson sampling achieves a much better
regret scaling than its frequentist counterparts like UCRL2
\citep{JakschOA10} which are forced to explore all possible state
transitions in isolation.

% On the improvement relative to a flat bandit algorithm over all
% possible optimal policies: We thank the reviewers for pointing out
% that a clearer illustration of the gain over a ``flat'' bandit
% algorithm over all possible MDP policies would be useful. We felt
% constrained by the available space in the paper, but we present here a
% brief sketch to show that the regret bound can be MUCH smaller than
% linear in the \# of possible policies (in fact, the bound in this
% example does not even scale with the \# of policies!). We will add this
% example to the final version.

\section{Sketch of Proof and Techniques used to show Theorem \ref{thm:main}}
\label{sec:overview}
At the outset, TSMDP is a randomized algorithm, whose decision is
based on a random sample from the parameter space $\Theta$. The
essence of Thompson sampling performance lies in understanding how the
posterior distribution evolves as time progresses.

Let us assume, for ease of exposition, that we have finitely many
parameters, $|\Theta| < \infty$. Writing out the expression for the
posterior density at time $t$ using Bayes' rule, we have,
$\forall \theta \in \Theta$,
\begin{align*}
  &\pi_{t+1}(d \theta) \propto
  p_{\theta}(S_t,A_{t+1},S_{t+1}) \pi_t(d\theta) 
%  &\propto \frac{p_{\theta}(S_t,A_{t+1},S_{t+1})}{p_{\theta^\star}(S_t,A_{t+1},S_{t+1})} \pi_t(d\theta) \\
  = \exp\left(-\sum_{i=0}^{t-1} \log
   \frac{p_{\theta^\star}(S_t,A_{t+1},S_{t+1})}{p_{\theta}(S_t,A_{t+1},S_{t+1})} \right) \pi_0(d\theta).
\end{align*}
The sum in the exponent above can be rearranged into
\begin{align*}
  \sum_{c \in \mathcal{C}} V_c(t) \sum_{s_1 \in \mathcal{S}} \frac{V_{s_1,c}(t)}{V_c(t)} \sum_{s_2 \in \mathcal{S}} \frac{1}{V_{s_1,c}(t)} \sum_{i=0}^{t-1} \mathbbm{1}\{\left(S_{i+1}, S_i\right) = \left(s_2, s_1\right), C_{e(i)} = c\} \log
  \frac{p_{\theta^\star}(s_1,c(s_1),s_2)}{p_{\theta}(s_1,c(s_1),s_2)},
\end{align*}
in which,
$V_c(t) \bydef \sum_{i=0}^{t-1} \mathbbm{1}\{C_{e(i)} = c\}$, and
$V_{s_1,c}(t) \bydef \sum_{i=0}^{t-1} \mathbbm{1}\{C_{e(i)} = c, S_i =
s_1\}$.%
% and $l_{s_1,c,s_2}^{(\theta^\star||\theta)}(i)$ is the log likelihood
% ratio
% $\mathbbm{1}\{S_{i+1} = s_2, S_i = s_1, C_{e(i)} = c\} \times \log
% \frac{p_{\theta^\star}(s_1,c(s_1),s_2)}{p_{\theta}(s_1,c(s_1),s_2)}$.
The above sum is an empirical quantity depending on the (random)
sample path $S_0, A_1, S_1, A_2, \ldots$ To gain a clear understanding
of the posterior evolution, let us replace the empirical terms in the
above sum by their ``ergodic averages'' (i.e., expected value under
the respective invariant distribution) under the respective
policies. In other words, for each $c \in \mathcal{C}$ and
$s_1 \in \mathcal{S}$, let us approximate
$\frac{V_{s_1,c}(t)}{V_c(t)} \approx \pi^{(c)}_{s_1}$, the stationary
probability of state $s_1$ when the policy $c$ is applied to the true
MDP $m_{\theta^\star}$. In the same way, we approximate
$\frac{\sum_{i=0}^{t-1} \mathbbm{1}\{S_{i+1} = s_2, S_i = s_1,
  \hat{C}_i = c\}}{V_{s_1,c}(t)} \approx
p_{\theta^\star}(s_1,c(s_1),s_2)$.%
% [Note: This is an approximation since (a) we replace a random quantity
% by a constant, and (b) we imagine, in a sense, that whenever policy
% $c$ is applied, we encounter a ``typical'' frequency of state transitions
% exactly according to its stationary distribution, instead of the
% actual, tightly-coupled Markovian state evolution.]
With these ``typical'' estimates, our approximation to the posterior
density simply becomes
\begin{equation}
\label{eqn:postapprox}
 \pi_{t+1}(d \theta) \appropto e^{-\sum_{c \in \mathcal{C}} V_c(t)
  D_c(\theta^\star || \theta)} \; \pi_0(d\theta), 
\end{equation}
% where $\sum_{c \in \mathcal{C}} V_c(t) = t$.

Expression (\ref{eqn:postapprox}) is the result of effectively
eliminating one of the two sources of randomness in the dynamics of
the TSMDP algorithm -- the variability of the environment, i.e., state
transitions. The other source of randomness arises due to the
algorithm's sampling behavior from the posterior distribution. We use
approximation (\ref{eqn:postapprox}) to extract {\em two basic
  insights} that determine the posterior shrinkage and regret
performance of TSMDP even for general parameter spaces: For a total
time horizon of $T$ steps, we claim {\bf Property 1. The true model
  always has ``high'' posterior mass.}  Assuming
$\pi_0(\{\theta^\star\}) > 0$ (the discrete ``grain of truth''
property), observe that (\ref{eqn:postapprox}) implies
$\pi_t(\{\theta^\star\}) \geq
\frac{\int_{\theta^\star} e^{-\sum_{c \in \mathcal{C}} V_c(t)
    D_c(\theta^\star || \theta)} \pi_0(d\theta)}{\int_{\Theta} e^0
  \pi(d\theta)} = \pi_0(\{\theta^\star\}) > 0$
at all times $t$. Thus, roughly, the true parameter $\theta^\star$ is
sampled by TSMDP with a frequency at least $\pi_0(\theta^\star) > 0$
during the entire horizon, i.e.,
$V_{c^\star}(t) \geq  t \pi_0(\theta^\star)$
$\forall t$.

We also have {\bf Property 2. Suboptimal models are sampled only as
  long as their posterior probability is above $\frac{1}{T}$}. The
total number of times a parameter with posterior mass less than
$\frac{1}{T}$ can be picked in Thompson sampling is at most
$\frac{1}{T} \times T = O(1)$, which is irrelevant as far as the
scaling of the regret with $T$ is concerned. 

With these two insights, we can now estimate the net number of times
bad parameters may be chosen. To this end, partition the parameter
space $\Theta$ into the optimal decision regions $\{S_c\}_{c \in
  \mathcal{C}}$, setting $S_c' \bydef \{\theta \in S_c:
D_{c^\star}(\theta^\star || \theta) = 0 \}$ and $S_c'' \bydef S_c
\setminus S_c'$.  Now, for each $c \neq c^\star$ and $\theta \in
S_c''$, $D_{c^\star}(\theta)$ is positive; thus, since $\Theta$ is
finite, $\exists \xi > 0$ such that $D_{c^\star}(\theta) > \xi$
uniformly across all such $\theta$. But this in turn implies, using
Property 1 and (\ref{eqn:postapprox}), that the posterior probability
of $\theta$ decays exponentially with time $t$: $\pi_{t+1}(d\theta)
\leq \frac{\pi_0(\theta)}{\pi_0(\theta^\star)} e^{-t
  \pi_0(\theta^\star) \xi}$.  Hence, such parameters $\theta \in
S_c''$, $c \neq c^\star$ are sampled at most a {\em constant} number
of times in any time horizon with high probability and do not
contribute to the overall regret scaling.

The interesting and non-trivial contribution to the regret comes from
the amount that parameters from $S_c'$, $c \neq c^\star$ are
sampled. To see this, let us follow the vector of play counts of
policies, i.e., $\left(V_c(t)\right)_{c \neq c^\star}$ as it starts
growing from the all-zeros vector at $t = 0$, increasing by $1$ in
some coordinate at each time step $t$. By Property 2 above, once
$\sum_{c \in \mathcal{C}} V_c(t) D_c(\theta^\star || \theta) \approx
\log T$
is reached, sampling from $S_c'$ effectively ceases. Thus, considering
the ``worst-case'' path that $\left(V_c(t)\right)_c$ can follow to
delay this condition for the longest time across all $c \neq c^\star$,
we arrive (approximately) at the optimization problem (\ref{eqn:main})
stated in Theorem \ref{thm:main}.

% the following optimization problem over the space of ``play counts''
% of policies $c \in \mathcal{C}$ can be framed.
%  \begin{equation*}
%     \begin{aligned}
%       &\max && \sum_{l=1}^{|\mathcal{C}|} z_l(l)\\
%       &\text{ s.t.}
%       && \mbox{$z_1, z_2, \ldots, z_{|\mathcal{C}|}$: non-negative vectors $\in \mathbb{Z}^{|\mathcal{C}|}$} \\
% %      &&& \mbox{(each co-ordinate of $\mathbb{R}^{|\mathcal{C}|}$ is for a policy $c \in \mathcal{C}$),} \\
%       &&& \mbox{$\sigma$: an ordering of suboptimal policies $c$}, \\
%       &&& \min_{\theta \in S_{\sigma(l)}'} \quad  z_l \cdot  D(\theta^\star || \theta) = \log T \quad \forall l, \\
%       &&& \mbox{No growth in co-ordinate $\sigma(l)$ for $\{z_l,
%         z_{l+1}, \ldots\}$}.
%     \end{aligned}
%   \end{equation*}

Though the argument above was based on rather coarse approximations to
empirical, path-based quantities, the underlying intuition holds true
and is made rigorous (Appendix \ref{app:main}) to show that this is
indeed the right scaling of the regret. This involves several
technical tools tailored for the analysis of Thompson sampling in
MDPs, including (a) the development of self-normalized concentration
inequalities for sub-exponential IID random variables (epoch-related
quantities), and (b) control of the posterior probability using
properties of the prior in Kullback-Leibler neighborhoods of the true
parameter, using techniques analogous to those used to establish
frequentist consistency of Bayesian procedures
\citep{ghosal2000,ChoRam08:consistency}.

%% the upshot of
%%   the calculation is that, for large horizons $T$, it accurately
%%   captures the regret scaling of Thompson sampling. Our intuition up
%%   until now can, in fact, be made rigorous using high-probability
%%   estimates and maximal, self-normalizing concentration inequalities
%%   to justify the approximation of empirical objects to their ergodic
%%   means. Additionally, the act of replacing an average of state
%%   transitions by a ``typically'' distributed state transition is
%%   facilitated precisely due to the recurrence-cycle based structure of
%%   TSMDP (Algorithm \ref{alg:tsmdp}). 

\section{Numerical Evaluation}
\label{sec:numerical}
%% In this section, we present empirical results from implementing the
%% TSMDP learning algorithm on a prototype parameterized MDP.

{\bf MDP and Parameter Structure:} Along the lines of the motivating
example in the Introduction, we model a single-buffer, discrete time
queueing system with a maximum occupancy of $50$
packets/customers. The state of the MDP is simply the number of
packets in the queue at any given time, i.e., $\mathcal{S} = \{0,1,2,
\ldots 50\}$. At any given time, one of $2$ actions -- Action $1$
(SLOW service) and Action $2$ (FAST service) may be chosen, i.e.,
$\mathcal{A} = \{1,2\}$. Applying SLOW (resp. FAST) service results in
serving one packet from the queue with probability $0.3$ (resp. $0.8$)
if it is not empty, i.e., the service model is Bernoulli($\mu_i$)
where $\mu_i$ is the packet processing probability under service type
$i = 1,2$. Actions $1$ and $2$ incur a per-instant cost of $0$ and
$25$ units respectively. In addition to this cost, there is a holding
cost of $1$ per packet in the queue at all times. The system gains a
reward of $200$ units whenever a packet is served from the
queue\footnote{A candidate physical interpretation of such a queueing
  system is in the form of a restaurant with $|\mathcal{S}|$ tables,
  with the possibility to add more ``chefs'' or staff into service
  when desired (service rate control). However, adding staff costs the
  restaurant, as does customers waiting long until their orders
  materialize (holding cost).}.

The arrival rate to the queueing system -- the probability with which
a new packet enters the buffer -- is modeled as being
state-dependent. Most importantly, the function $\lambda:\mathcal{S}
\to [0,1]$ mapping a state to its corresponding packet arrival rate is
parameterized using a standard Normal distribution ( make this
clearer, avoid confusion with Normal probability distn.) as follows:
$\lambda(s) = \kappa e^{-\frac{(s - \bar{\mu})^2}{2
    \bar{\sigma}^2}}$. Here, $\bar{\mu}$ and $\bar{\sigma}$ represent
the $2$-dimensional (mean,standard deviation) parameter for the
arrival rate curve, and $\kappa$ is chosen to be a constant that makes
$\max_{s \in \mathcal{S}} \lambda(s) = 0.95$ (to ensure valid
Bernoulli packet arrival distributions). For the true, unknown MDP, we
set $\theta^\star \equiv (\bar{\mu}, \bar{\sigma}) = (0.6,0.3) \times
|\mathcal{S}|$ ( clarify Cartesian prod). Figure
\ref{fig:truemdp} depicts (a) the optimal policy $c^\star$ over
$\mathcal{S}$, (b) the stationary distribution under the optimal
policy and (c) the (parameterized) mean arrival rate curve over
$\mathcal{S}$.

{\bf Simulation Results:} We simulate both TSMDP and the UCRL2
algorithm (Jaksch et al \citep{JakschOA10}) for the parameterized
queueing MDP above. For UCRL2, we run the algorithm both with (a)
fixed confidence intervals $\delta \in \{0.01, 0.1, 0.5\}$ and (b)
$\delta = 1/T$ (horizon-dependent confidence intervals\footnote{This
  choice of $\delta$ is used by Jaksch et al to show a logarithmic
  expected regret bound for UCRL2 \citep{JakschOA10}.}). We initialize
TSMDP with a uniform prior for the normalized parameter
$\frac{1}{|\mathcal{S}|}(\bar{\mu}, \bar{\sigma})$ on the discretized
space $\{0.05, 0.10, 0.15, \ldots, 0.95\} \times \{0.2, 0.4, 0.6,
\ldots, 2.0\}$.

Figure \ref{fig:ts-ucrl2} shows the results of running the TSMDP and
UCRL2 algorithms for various time horizons $T = 10$ up to $T =
1,000,000$ time steps, and across $1,000$ sample runs. We report both
the average regret (w.r.t. a best per-step average reward of
$96.4088$) and the $20\%-80\%$ percentile of the regret across the
runs. Thompson sampling is seen to significantly outperform UCRL2 as
the horizon length increases. This advantage is presumably due to the
fact that TSMDP is capable of exploiting the parameterized structure
of $\lambda$ better than UCRL2, which updates each confidence interval
only when the associated state is visited.

\ifdefined\WITHFIGS
\begin{figure}[htbp]
  \centering \includegraphics[scale=0.4]{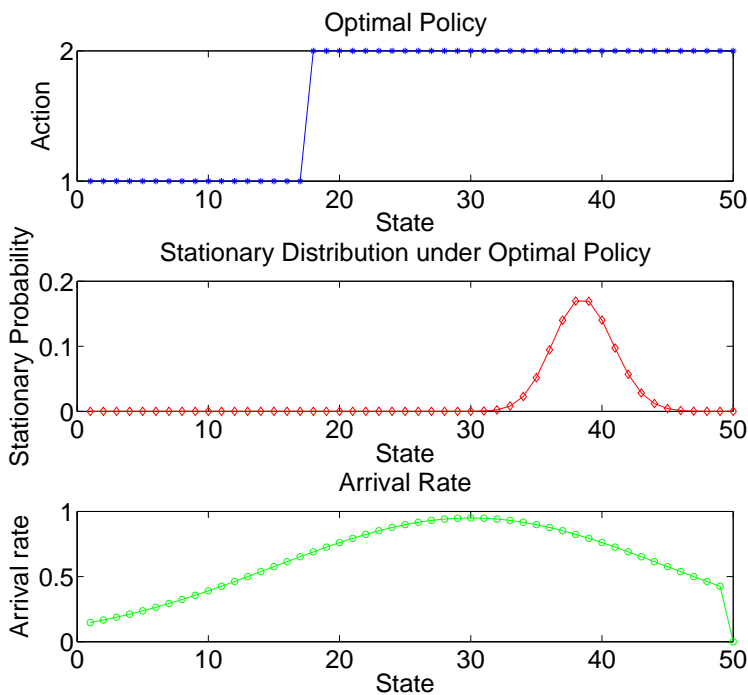}
  \caption{The true MDP for the queueing system simulation. (a) The
    optimal policy $c^\star$ selecting an action ($1,2$) for each
    state in $0,1,2,\ldots,50$. The optimal average reward is
    $96.4088$ units. (b) Stationary probability distribution for the
    optimal policy. (c) Parameterized (by a Gaussian curve),
    state-dependent, average arrival rate profile.}
  \label{fig:truemdp}
\end{figure}
\fi

\ifdefined\WITHFIGS
\begin{figure}[htbp]
  \centering \includegraphics[scale=0.4]{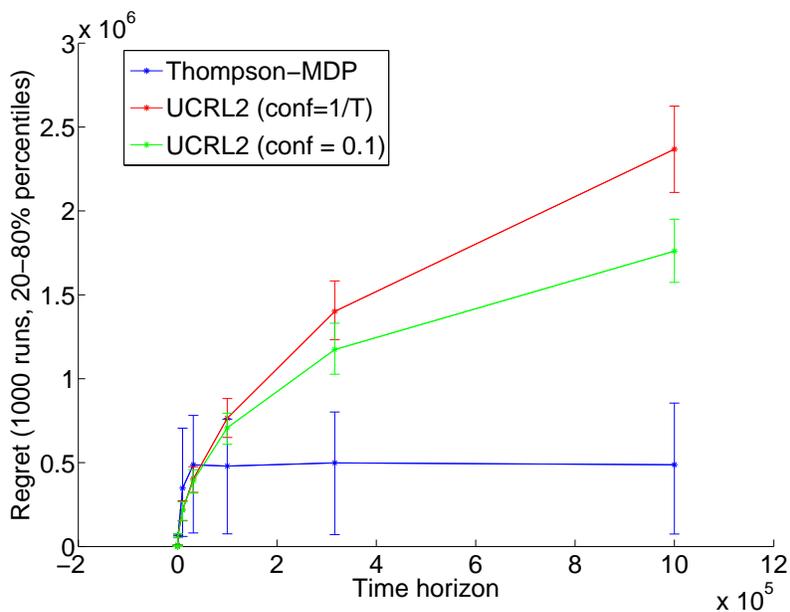}
  \caption{Regret for TSMDP and UCRL2 \citep{JakschOA10}, across 1000
    runs, for the single-queue parameterized MDP. Both mean regret and
    $[20,80]\%$-percentiles are shown.}
  \label{fig:ts-ucrl2}
\end{figure}
\fi

% \ifdefined\WITHFIGS
% \begin{figure}
% \centering
% \begin{minipage}{.4\textwidth}
%   \centering
%   \includegraphics[scale=0.15]{truemdp}
%   \captionof{figure}{{\footnotesize The true MDP for the queueing system simulation. (a) The
%      optimal policy $c^\star$ selecting an action ($1,2$) for each
%      state in $0,1,2,\ldots,50$. The optimal average reward is
%      $96.4088$ units. (b) Stationary probability distribution for the
%      optimal policy. (c) Parameterized (by a Gaussian curve),
%      state-dependent, average arrival rate profile.}}
%   \label{fig:truemdp}
% \end{minipage}%
% \hspace{0.7cm}
% \begin{minipage}{.4\textwidth}
%   \centering
%   \includegraphics[scale=0.15]{ts-ucrl2-taqeria}
%   \captionof{figure}{{\footnotesize Regret for TSMDP and UCRL2 \citep{JakschOA10}, across 1000
%     runs, for the single-queue parameterized MDP. Both mean regret and
%      $[20,80]\%$-percentiles are shown.}}
%   \label{fig:ts-ucrl2}
% \end{minipage}
% \end{figure}
% \fi

\section{Related Work}
\label{sec:relatedwork}

% (situate the paper well w.r.t. papers by Ben Van-roy -- we deal
% with frequentist notion of regret as opposed to BVR et al's Bayesian
% regret; Bayesian regret is freq regret averaged wrt prior; while this
% is useful and elegant it is weaker than standard regret in the sense
% that it is the standard freq regret averaged across models wrt the
% prior; further, it is not indicative of how structure of a specific
% model affects regret performance - we provide structural, ``gap
% dependent'' regret bounds - check this wrt ``hardness'' explanation in
% BvR papers; we provide the first known frequentist regret analysis of
% thompson sampling for MDPs ...)

A line of recent work
\citep{AgrawalG,KauKorMun12:thompson,KorKauMun13:tsexpfam,AgrGoy13:contextual,GopManMan14:thompson}
has demonstrated that the Thompson sampling enjoys near-optimal regret
guarantees for {\em multi-armed bandits} -- a widely studied subclass
of reinforcement learning problems.

% w.r.t. Osband-VanRoy-Russo: non-episodic continuous online learning, so
% no dependence on episode length!; problem-dependent, frequentist
% logarithmic learning bounds; in fact first logarithmic (in time horiz)
% learning bounds for ts in mdps; we show that essentially TS follows
% opt policy $T - O(\log T)$ times whp; capture correlations in
% parameterized MDPs in our perf bound

The work of Osband et al \citep{OsbRusVan13:more}, perhaps the most
relevant to us, studies the {\em Bayesian regret} of Thompson
sampling for MDPs. In this setting, the true MDP is assumed to have
been drawn from the same prior used by the algorithm; consequently,
the Bayesian regret becomes the standard frequentist regret {\em
  averaged} across the entire parameter space w.r.t. the prior. While
this is useful, it is arguably weaker than the standard frequentist
notion of regret in that it is an averaged notion of standard regret
(w.r.t. the specific prior), and moreover is not indicative of how the
structure of the MDP exactly influences regret performance. Moreover,
the learning model considered in their work is {\em episodic} with
fixed-length {\em episodes} and resets, as opposed to the non-episodic
learning setting treated in this work, where we are able to show the
first known structural (``gap-dependent'') regret bounds for Thompson
sampling in fixed but unknown parameterized MDPs.

Prior to this, \citet{OrtegaBraun10} investigate the consistency
performance of posterior-sampling based control rules, again in the
fully Bayesian setting where nature's prior is known.

% w.r.t. REGAL, UCRL2, RMax: parameterized MDPs; alg not dependent on
% analysis, so possibility to tighten; alg does not use theory-dependent
% confidence bounds; in fact, no clean confidence sets for general
% parameterized problems; only for ``decoupled'' params; regret bounds
% depend on diameter, but not problem dependent structural bounds

Several deterministic algorithms relying on the ``optimism under
uncertainty'' philosophy have been proposed for RL in the frequentist
setup considered here
\citep{Brafman2003,JakschOA10,Bartlett2009}. These algorithms work by
maintaining confidence intervals for each transition probability and
reward, computing the most optimistic MDP satisfying all confidence
intervals and adaptively shrinking the confidence intervals each time
the relevant state transition occurs. This strategy is potentially
inefficient in parameterized MDPs where, potentially, observing a
particular state transition can give information about other parts of
the MDP as well.
% Moreover, known regret bounds for these algorithms
% are ``worst-case, gap-independent'' and in terms of the diameter of
% the MDP, instead of being structural and problem-dependent as we show
% here.

%% w.r.t. Dyagilev-Mannor-Shimkin: they consider learning parameterized MDPs
%% as we do; but for discounted case only; PAC model, so does not really
%% capture notion of regret

The parameterized MDP setting we consider in this work has been
previously studied by other authors. Dyagilev et al
\citep{DyaManShi08:param} investigate learning parameterized MDPs for
{\em finite} parameter spaces in the discounted setting (we consider
the average-reward setting), and demonstrate sample-complexity results
under the Probably-Approximately-Correct (PAC) learning model, which
is different from the notion of regret.

%% Agrawal-teneketzis-anantharam: give fundamental lower bounds on regret
%% scaling for general parameterized mdp problems; tight in the sense
%% that for finite param spaces, an alg is proposed which achieves it;
%% although our analysis is also for finite params, the ATA strategy
%% altogether breaks down for infinite param spaces; tsmdp is
%% well-defined for any type of param spaces though, and has guarantees
%% for continuous obs spaces e.g. linear contextual bandits

The {\em certainty equivalence} approach to learning MDPs
\citep{kumar1986stochastic} -- building the most plausible model given
available data and using the optimal policy for it -- is perhaps
natural, but it suffers from a serious lack of adequate exploration
necessary to achieve low regret \citep{KaeLitMoo96:survey}.

A noteworthy related work is the seminal paper of Agrawal et al
\citep{AgrTenAna89:markov} that gives fundamental lower bounds on the
asymptotic regret scaling for general, parameterized reinforcement
learning problems. The bound is also tight, in the sense that for
finite parameter spaces, the authors show a learning algorithm that
achieves the bound. Even though our analytical results also hold for
the setting of a finite parameter space, the strategy in
\citep{AgrTenAna89:markov} relies {\em crucially} on the finiteness
assumption. This is in sharp contrast to Thompson sampling which can
be defined for any kind of parameter space. In fact, Thompson sampling
has previously been shown to enjoy favorable regret guarantees with
continuous priors in linear bandit problems
\citep{AgrGoy13:contextual}.

\section{Conclusion and Future Work}
% what we did
% what we can do later

% We have proposed the TSMDP algorithm in this paper for solving
% parameterized reinforcement learning problems, in which the parameter
% controls the structural elements of the underlying MDP. Thompson
% sampling is a natural Bayesian-inspired algorithm that factors in all
% available information by maintaining a ``prior'' distribution over
% candidate parameters, and gradually shrinks the prior as more
% information is obtained via state transitions in the environment. The
% flexible structure of the algorithm allows it to be applied in a wide
% variety of complex reinforcement learning problems. We have derived a
% general-purpose regret bound for TSMDP that quantifies the logarithmic
% growth in the number of suboptimal actions chosen by the
% algorithm. This supports the increasing evidence for the success of
% Thompson sampling and pseudo-Bayesian methods for reinforcement
% learning/bandit problems.

We have proposed the TSMDP algorithm in this paper for solving
parameterized RL problems, and have derived regret-style bounds for
the algorithm under significantly general initial priors. This
supports the increasing evidence for the success of Thompson sampling
and pseudo-Bayesian methods for reinforcement learning/bandit
problems.

Moving forward, it would be useful to extend the performance results
for Thompson sampling to continuous parameter spaces, as well as
understand what happens when feedback can be delayed. Specific
applications to reinforcement learning problems with additional
structure would also prove insightful. In particular, studying the
regret of Thompson Sampling for MDPs with linear function
approximation \citep{MelMeyRib08:funcapprox} would be of interest --
in this setting, the parameterization of the MDP is in terms of linear
weights corresponding to a known basis of state-action value
functions, and one could develop a variant of Thompson sampling which
uses information from sample paths to update its posterior over the
space of weights.

\newpage 
\appendix

\noindent {\bf \Large Supplementary Material (Appendices and
  References) \\\\ {\Large \it Thompson Sampling for Learning Parameterized
    Markov Decision Processes}} 

\section{Proof of Theorem \ref{thm:main}}
\label{app:main}

% With regard to the TSMDP algorithm (Algorithm \ref{alg:tsmdp}), let
% us use the notation $e(t)$ to denote the epoch index $k$
% corresponding to a time instant $t$. Also, without loss of
% generality\footnote{This is because the TSMDP algorithm always
% samples from
% $\{c^{\mathsf{OPT}}(\theta): \theta \in \Theta\} \subset
% \mathcal{C}$},
% we take $\mathcal{C} \equiv \{c^{\mathsf{OPT}}(\theta): \theta \in \Theta\}$.

\subsection{Expressing the ``posterior'' distribution}
At time $t$, the ``posterior distribution'' $\pi_t$ that TSMDP uses
can be expressed by iterating Bayes' rule (\ref{eqn:bayesrule}):
\begin{align*}
  \forall \mathcal{M} \subseteq \Theta \quad \pi_t(\mathcal{M}) &=
  \frac{W_t(\mathcal{M})}{W_t(\Theta)} = \frac{\int_{\mathcal{M}}
    W_t(\theta) \pi(d\theta)}{\int_\Theta W_t(\theta) \pi(d\theta)},
\end{align*}
with the posterior density or {\em weight} $W_t(\theta)$ simply being
the likelihood ratio of the entire observed history up to $t$ under
the MDPs $m_\theta$ and $m_{\theta^\star}$, i.e.,
\begin{align}
  &W_t(\theta) \bydef \prod_{i=0}^{t-1}
  \frac{p_\theta(S_i,A_{i+1},S_{i+1})}{p_{\theta^\star}(S_i,A_{i+1},S_{i+1})} \nonumber \\
  &= \exp \left( \sum_{c \in \mathcal{C}} \sum_{i=0}^{t-1}
    \mathbbm{1}\{ C_{e(i)} = c \} \log
    \frac{p_\theta(S_i,A_{i+1},S_{i+1})}{p_{\theta^\star}(S_i,A_{i+1},S_{i+1})}
  \right) \nonumber \\
  &= \exp \left( \sum_{c \in \mathcal{C}} \sum_{(s_1,s_2) \in
      \mathcal{S}^2} \sum_{i=0}^{t-1} \mathbbm{1}\left\{ C_{e(i)} = c,
      (S_i,S_{i+1}) = (s_1,s_2) \right\} \log
    \frac{p_\theta(s_1,c(s_1),s_2)}{p_{\theta^\star}(s_1,c(s_1),s_2)}
  \right) \nonumber \\
  &= \exp \left( -\sum_{c \in \mathcal{C}} V_c(t) \sum_{(s_1,s_2) \in
      \mathcal{S}^2} \sum_{i=0}^{t-1} \frac{\mathbbm{1}\left\{
        C_{e(i)} = c, (S_i,S_{i+1}) = (s_1,s_2) \right\}}{V_c(t)} \log
    \frac{p_{\theta^\star}(s_1,c(s_1),s_2)}{p_{\theta}(s_1,c(s_1),s_2)}
  \right), \label{eqn:wt1}
\end{align}
where $V_c(t) \bydef \sum_{i=0}^{t-1} \mathbbm{1}\left\{ C_{e(i)} = c
\right\}$ is the total number of time instants up to $t$ for which the
epoch policy $c$ was used.

We will find it convenient in the sequel to introduce the following
decomposition of the number of epochs up to epoch $k$ for which $c$
was chosen to be the epoch policy:
  \begin{equation} 
    \label{eqn:pcdecomp}
    N_c(k) \bydef \sum_{l=1}^k \mathbbm{1}\{\theta_l \in S_c\} = N_c'(k) +
    N_c''(k),
  \end{equation}
 \[N_c'(k) \bydef \sum_{l=1}^k \mathbbm{1}\{\theta_l
  \in S_c'\}, N_c''(k) \bydef \sum_{l=1}^k \mathbbm{1}\{\theta_l \in
  S_c''\}.\]

\subsection{An alternative probability space} 
\label{sec:alternative}
In order to analyze the dynamics of the TSMDP algorithm, it is useful
to work in an equivalent probability space defined as follows. Define
a $\infty \times |\mathcal{C}|$ random matrix $Q$ with elements in
$\mathcal{S} \times \mathcal{A} \times \mathbb{R}$. The rows of $Q$
are indexed by {\em sampling indices} $l = 1, 2, \ldots$, and the
columns by policies in $\mathcal{C}$. For each $c \in \mathcal{C}$,
{\em independently} generate the $c$-th column of $Q$ by applying the
stationary policy $c$ to the MDP $m_{\theta^\star}$, starting from
initial state $s_0$, and noting down the resulting
$(state,action,reward)$ sequence, i.e.,
$Q(l,c) \equiv (Q_1(l,c),Q_2(l,c),Q_3(l,c)) \bydef
({S}^{\theta^\star,c}_l,{A}^{\theta^\star,c}_l,{R}^{\theta^\star,c}_l)$.
For the $c$-th column of $Q$, we let $\tilde{\tau}_{0,c} \bydef 0$,
and
$\tilde{\tau}_{k,c} \bydef \min\{l \geq \tilde{\tau}_{k-1,c}: Q_1(l,c)
= s_0 \}$
$\forall k \geq 1$. In words, $\tilde{\tau}_{k,c}$ is the $k$-th
successive ``virtual time'' at which the MDP $m_{\theta^\star}$ under
policy $c$ returns to the start state $s_0$. We thus have that the
expected first return time to $s_0$, defined earlier in Section
\ref{sec:formal}, satisfies
$\bar{\tau}_{c} = \tilde{\mathbb{E}}[\tilde{\tau}_{1,c}]$. \\

Given the matrix $Q$, we can alternatively simulate the TSMDP
algorithm operating in the MDP $m_{\theta^\star}$ as follows. At each
round $t \geq 1$ with the epoch index $e(t) = k$, if the epoch policy
in effect is $C_k = c$, then the action $A_t =
Q_2(\tilde{\tau}_{N_c(k),c} + t - t_k, c)$ is played, with the next
state (resp. reward) being $S_t = Q_1(\tilde{\tau}_{N_c(k),c} + t -
t_k, c)$ (resp. $R_t = Q_3(\tilde{\tau}_{N_c(k),c} + t - t_k, c)$). \\

Let $\tilde{\mathbb{P}}$ denote the probability measure for the
alternative probability space described above. The following
equivalence lemma records the fact that the distributions of the
$(state, action, reward)$ sample path seen by the TSMDP algorithm
under the original probability measure $\mathbb{P}$ and under in the
alternative measure $\tilde{\mathbb{P}}$ are both identical.

\begin{lemma}[Equivalence of probability spaces]
  For each $(state,action,reward)$ sequence \\
  $\{(s_t,a_t,r_t)\}_{t=1}^T$, we have, under the TSMDP algorithm,
  \[ \tilde{\mathbb{P}}\left[ \forall 1 \leq t \leq T \; (S_t,A_t,R_t)
    = (s_t,a_t,r_t) \right] = \prob{\forall 1 \leq t \leq T \;
    (S_t,A_t,R_t) = (s_t,a_t,r_t)}. \]
\end{lemma}
Henceforth, we will work in the alternative space with measure
$\tilde{\mathbb{P}}$ but will dispense with the tilde for ease of
notation.

We now develop some useful concentration estimates for the random
sample path matrix $Q$. Define the following empirical estimates:
\begin{itemize}
\item $U_{(s_1,s_2)}(j,c) \bydef \frac{1}{j} \sum_{l=1}^j
  \mathbbm{1}\left\{Q_1(l-1,c) = s_1, Q_1(l,c) = s_2 \right\}$, $s_1,
  s_2 \in \mathcal{S}$, $j \geq 1$, denote the empirical mean number
  of state transitions $s_1 \rightarrow s_2$ down column $c$ of $Q$
  (or the {\em pairwise} empirical frequency),
\item $U(j,c) \bydef \left(U_{(s_1,s_2)}(j,c)\right)_{s_1,s_2 \in
    \mathcal{S}}$ denote the empirical state transition vector for
  policy $c$,
\item $U_{s_1}(j,c) \bydef \sum_{s_2 \in \mathcal{S}}
  U_{(s_1,s_2)}(j,c)$, $s_1 \in \mathcal{S}$, $j \geq 1$, be the {\em
    marginal} empirical frequency, and
\item $U_{s_2|s_1}(j,c) \bydef
  \frac{U_{(s_1,s_2)}(j,c)}{U_{s_1}(j,c)}$, $s_1 \in \mathcal{S}$,
  $s_2 \in \mathcal{S}$, $j \geq 1$, be the {\em conditional}
  empirical frequency (whenever $U_{s_1}(j,c) > 0$; defined to be
  $0$ otherwise)
\end{itemize}
in $j$ virtual time steps. With this alternative view of the TSMDP
execution, equation (\ref{eqn:wt1}) for the posterior probability
density $W_t$ at time $t$ becomes

\begin{equation}
  \label{eqn:wt2}
  -\log W_t(\theta) = \sum_{c \in \mathcal{C}} V_c(t) \sum_{(s_1,s_2) \in
    \mathcal{S}^2} U_{(s_1,s_2)}\left(V_c(t), c\right)  \log
  \frac{p_{\theta^\star}(s_1,c(s_1),s_2)}{p_{\theta}(s_1,c(s_1),s_2)}.
\end{equation}

The following key self-normalized uniform bound controls the large
deviation behavior of the empirical means $U_{(s_1,s_2)}(j,c)$ and the
return times $\tilde{\tau}_{k,c}$. It may be interpreted as a
finite-sample version of the Law of the Iterated Logarithm (LIL).

\begin{proposition}[Uniform concentration for empirical means]
  \label{prop:conc}
  Fix $\delta \in [0,1]$. Then, there exist constants $d_1 \geq 0$,
  $d_2 \geq 0$ such that the following estimates hold with probability at
  least $1-\delta$ for all $k \geq 1$, $c \in \mathcal{C}$, $s_1, s_2
  \in \mathcal{S}$:
  \begin{align} & \left| \tilde{\tau}_{k,c} - k \bar{\tau}_{c}
    \right| \leq \sqrt{ k d_1 \log\left(\frac{|\mathcal{C}|
          |\mathcal{S}|^2 d_2 \log
          k}{\delta} \right)}, \label{eqn:conc1}\\
    &\left|\tilde{\tau}_{k,c} \cdot U_{(s_1,s_2)}(\tilde{\tau}_{k,c}, c) -
      k \bar{\tau}_{c} \cdot {\pi^{(c)}_{(s_1,s_2)}}
    \right| \leq \sqrt{ k d_1 \log\left(\frac{|\mathcal{C}|
          |\mathcal{S}|^2 d_2 \log k}{\delta} \right)}, \label{eqn:conc2} \\
    &\left|\tilde{\tau}_{k,c} \cdot U_{s_1}(\tilde{\tau}_{k,c}, c) -
      k \bar{\tau}_{c} \cdot {\pi^{(c)}_{s_1}}
    \right| \leq \sqrt{ k d_1 \log\left(\frac{|\mathcal{C}|
          |\mathcal{S}|^2 d_2 \log k}{\delta} \right)}. \label{eqn:conc2b}
  \end{align}
\end{proposition}

\begin{proof}
  By the Markov property, it follows that the (non-negative) random
  variables $\tilde{\tau}_{1,c}$, $(\tilde{\tau}_{2,c} -
  \tilde{\tau}_{1,c})$, $(\tilde{\tau}_{3,c} - \tilde{\tau}_{2,c})$,
  $\ldots$, $(\tilde{\tau}_{k,c} - \tilde{\tau}_{k-1,c})$ are
  IID. From standard arguments for finite-state, irreducible Markov
  chains \citet[Lemma 7]{LeeOzdSha13:stationary}, we have that the
  recurrence times to $s_0$ have exponential tails:
  \begin{equation}
    \label{eqn:tautail}
    \forall v \geq 0 \quad \prob{\tilde{\tau}_{1,c} > v} \leq 2 \cdot 2^{-\left(\frac{v}{2 \bar{\tau}_{\max}}\right)},
  \end{equation}
  where $\bar{\tau}_{\max}$ is the maximum expected hitting time, over
  states in the same communicating class as $s_0$, to $s_0$. We also
  have $\expect{\tilde{\tau}_{1,c}} = \bar{\tau}_{c}$. \\
  
  On the other hand, using the definition of
  $U_{(s_1,s_2)}(\tilde{\tau}_{k,c}, c)$, we can write 
  \[ \tilde{\tau}_{k,c} \cdot U_{(s_1,s_2)}(\tilde{\tau}_{k,c}, c) =
  \sum_{j=1}^k B_{c,s_1,s_2}(\tilde{\tau}_{j-1,c}+1,
  \tilde{\tau}_{j,c}), \] where the partial sums
  \[ B_{c,s_1,s_2}(\tilde{\tau}_{j-1,c}+1, \tilde{\tau}_{j,c}) \bydef
  \sum_{l=\tilde{\tau}_{j-1,c}+1}^{\tilde{\tau}_{j,c}}
  \mathbbm{1}\left\{Q_1(l-1) = s_1, Q_1(l) = s_2 \right\}, \quad j =
  1, 2, \ldots, k \] are again non-negative IID random variables due
  to the Markov property, and are bounded by the corresponding cycle
  lengths $(\tilde{\tau}_{j,c} - \tilde{\tau}_{j-1,c})$. Thus,
  $B_{c,s_1,s_2}(1, \tilde{\tau}_{1,c})$ also satisfies the
  exponential tail inequality (\ref{eqn:tautail}) satisfied by
  $\tilde{\tau}_{1,c}$, with mean\footnote{The expectation can be
    computed via the renewal-reward theorem \citep{grimmett92} and
    Markov chain ergodicity.} $\expect{B_{c,s_1,s_2}(1,
    \tilde{\tau}_{1,c})} =
  \frac{\pi^{(c)}_{(s_1,s_2)}}{\pi_{s_0}(\theta^\star,c)} =
  \bar{\tau}_{c} \cdot {\pi^{(c)}_{(s_1,s_2)}}$. \\

  The conclusions of the proposition now follow by (a) appealing to
  the maximal concentration inequality of Lemma \ref{lem:maximal}, and
  (b) taking a union bound over all $c \in \mathcal{C}$, $s_1, s_2 \in
  \mathcal{S}$ with the least possible uniform upper bounds on the
  constants $\eta_1$ and $\eta_2$ guaranteed by Lemma
  \ref{lem:maximal}.
\end{proof}

Lemma \ref{lem:maximal} below gives a concentration bound for the
entire sample path of the empirical mean of an IID process, and may be
viewed as a finite-sample analog of the asymptotic Law of the Iterated
Logarithm (LIL).

  \begin{lemma}[A maximal concentration inequality for random walks with
    sub-exponential increments]
    \label{lem:maximal}
    Let $X_1, X_2, \ldots$ be a sequence of IID random variables such
    that $\prob{|X_1| > v} \leq \alpha_1 e^{-\alpha_2 v}$ for some
    $\alpha_1, \alpha_2 > 0$, and fix $\delta \in [0,1]$. Then, there
    exist constants $\eta_1 \geq 0$, $\eta_2 \geq 0$ such that the
    following event occurs with probability at least $1-\delta$:
    \[ \forall {k \geq 1} \; \left| \sum_{i=1}^k X_i - k \expect{X_1}
    \right| \leq \sqrt{\eta_1 k \log\left( \frac{\eta_2 \log
          k}{\delta} \right)} .\]
  \end{lemma}

  \begin{proof}
    We begin by noticing that the exponential tail property implies
    finiteness of the moment generating function in a neighborhood of
    zero: for any $\lambda \in (0,\alpha_2)$, 
    \begin{align*}
      e^{\Lambda_{X_1}(\lambda)} \bydef \expect{e^{\lambda X_1}} &=
      \int_0^\infty \prob{e^{\lambda X_1} > y } dy  \\
      &\leq 1 +  \int_1^\infty \prob{e^{\lambda X_1} > y } dy \\
      &\leq 1 + \int_1^\infty \alpha_1 y^{-\alpha_2/\lambda} dy <
      \infty.
    \end{align*}
    This allows us to take a second-order Taylor series expansion of
    $\Lambda_{X_1}(\lambda)$ around $\lambda = 0$, to get that
    $\exists \beta \in \mathbb{R}$ such that $\Lambda_{X_1}(\lambda)
    \leq \lambda \expect{X_1} + \frac{\beta^2 \lambda^2}{2}$ $\forall
    \lambda \in \left[-\frac{\alpha_2}{2}, \frac{\alpha_2}{2}
    \right]$. As a consequence,
    \[ M_t \bydef \exp \left(\lambda \sum_{i=1}^t X_i - \lambda t
      \expect{X_1} - \frac{t \beta^2 \lambda^2}{2} \right), \quad t =
    0, 1, 2, \ldots \]
    is a non-negative supermartingale for each
    $\lambda \in \left[-\frac{\alpha_2}{2}, \frac{\alpha_2}{2}
    \right]$.
    Applying the {\em method of mixtures} technique for martingale
    suprema \citep[Example 2.5]{delapena2007} (due, in turn, to the
    pioneering work of \citet[Example 4]{robbins1970}), we obtain the
    bound
    \[ \prob{\sum_{i=1}^k X_i - k\expect{X_1} \geq g_k \quad \mbox{for
        some } k \geq 1} \leq \delta, \]
    with
    $g_k \bydef \sqrt{\gamma_2 \beta^2 k \log\left(\frac{\gamma_1
          \log(\beta^2 k)}{\delta}\right)}$
    for some constants $\gamma_1 \geq 0$, $\gamma_2 \geq 0$. This
    finishes one half of the proof for the ``positive tail''
    $\sum_{i=1}^k X_i$ . The other half follows in an analogous
    fashion by considering the negated random variables $\{-X_i\}_i$.
  \end{proof}
  
  We henceforth consider as fixed the confidence parameter $\delta \in
  [0,1]$, and denote $\rho(x) \equiv \rho_\delta(x) \bydef \sqrt{d_1
    \log\left(\frac{|\mathcal{C}| |\mathcal{S}|^2 d_2 \log x}{\delta}
    \right)}$, $x \geq 1$. Note that $\rho(x) =
  O\left(\sqrt{\log\log(x)}\right)$ as a function of $x$.  

  \begin{definition}[``Typical'' trajectories]
    \label{def:G}
    Let
    \[ G \bydef \left\{
      \begin{array}{ll}
        & \left| \tilde{\tau}_{k,c} - k \bar{\tau}_{c}
        \right| \leq \rho(k)\sqrt{k}, \\
        \forall c \in \mathcal{C} \; \forall s_1, s_2 \in \mathcal{S} \; \forall k \geq 1: & \left|\tilde{\tau}_{k,c} U_{(s_1,s_2)}(\tilde{\tau}_{k,c}, c) -
          k \bar{\tau}_{c}  {\pi^{(c)}_{(s_1,s_2)}}
        \right| \leq \rho(k) \sqrt{k}, \\
        & \left|\tilde{\tau}_{k,c} U_{s_1}(\tilde{\tau}_{k,c}, c) -
          k \bar{\tau}_{c}  {\pi^{(c)}_{s_1}}
        \right| \leq \rho(k) \sqrt{k} 
      \end{array}
    \right\}
\]
be the event that the random matrix $Q$ from Section
\ref{sec:alternative} satisfies (\ref{eqn:conc1}) and
(\ref{eqn:conc2}) (``near-ideal'' sample paths).

  \end{definition}
 We thus have, by our previous estimates, that
  \begin{equation}
    \label{eqn:probG}
    \prob{G} \geq 1-\delta.
  \end{equation}
  
% For each stationary policy $c \in \mathcal{C}$, define $S_c \bydef
% \{\theta \in \Theta: c^{\mathsf{OPT}}(\theta) = c\}$, i.e., the {\em decision
%   region} of $c$, or the set of parameters/MDPs for which the
% average-reward optimal policy is $c$. Let $\epsilon' > 0$, and within
% $S_c$, let $S_c' \equiv S_c'(\epsilon') \bydef \{\theta \in S_c:
% D_{c^\star}(\theta^\star || \theta) \leq \epsilon' \}$. In words, $S_c'$ comprises all
% the parameters (i.e., MDPs) with average reward-optimal policy $c$ that
% are ``similar'' to $\theta^\star$ (i.e., $m_{\theta^\star}$) under the
% true optimal policy $c^{\mathsf{OPT}}(\theta^\star)$. Correspondingly, put
% $S_c'' \equiv S_c''(\epsilon') \bydef S_c \setminus S_c'$, i.e., the
% remaining set of parameters in the decision region $S_c$ that are
% separated by at least $\epsilon'$ under the true policy $c^\star$. \\

%   \begin{lemma}
%     \label{lem:sc''D}
%     If $c \neq c^\star$, then $\min \{D_{c^\star}(\theta^\star || \theta): \theta \in S_c''\}
%     > 0$.
%   \end{lemma}

  The crux of the proof of Theorem \ref{thm:main} is in controlling
  regret of two kinds.

  \begin{enumerate}
  \item {\em Regret due to sampling parameters from $S_c''$, $c \neq
    c^\star$:} We will show that the true parameter $\theta^\star$ is
    sampled at least a constant fraction (bounded away from $0$) of
    times in $0, 1, \ldots, T$. This implies that parameters in
    $S_c''$ are sampled at most a {\em constant} number of times.

  \item {\em Regret due to sampling parameters from $S_c', c \neq
    c^\star$:} We will establish that the number of times that
    parameters from $S_c'$ are sampled is the claimed logarithmic
    bound in Theorem \ref{thm:main}.
  \end{enumerate}

  \subsection{Regret due to sampling from $S_c''$}
  In this section, our goal is to show
  \begin{proposition}[$O(1)$ samples from $S_c''$ whp.]
    \label{prop:Sc''}
    There exists $\alpha < \infty$ such that
    \[ \prob{\exists c \neq c^\star \; \sum_{k=1}^\infty
      \mathbbm{1}\{\theta_k \in S_c''\} > \frac{\alpha |\mathcal{C}|}{\delta} \given G} \leq \delta. \]
  \end{proposition}

%% Let the random variables $J_{(s_1, s_2)}(t,c) \bydef \sum_{i=1}^t
%% \mathbbm{1}\{C_{e(i) = c, (S_i,S_{i+1} = (s_1,s_2))}\}$ denote the
%% number of instants at which the algorithm used policy $c$ and saw
%% state transitions $s_1 \to s_2$ up to time $t$.

  Let $J_{(s_1,s_2)}(k_c,c)$ denote the number of instants that the
  state transition $s_1 \to s_2$ occurs in $k_c$ successive epoch uses
  of policy $c$.

  \begin{lemma}
    \label{lem:wtbd1}
    Under the event $G$, for each $\theta \in \Theta$ satisfying
    $\pi(\theta) > 0$, each $c \in \mathcal{C}$ and $k \geq 1$,
    \begin{enumerate}
    \item The following lower bound holds on the negative
      log-density. \[ -\log W_{t_k}(\theta)
      \geq N_c(k)\bar{\tau}_{c} \cdot D_c(\theta^\star || \theta) - \Gamma
      |\mathcal{S}|^2 \rho(N_c(k)) \sqrt{N_c(k)}.\]
    \item The following upper bound holds on the negative
      log-density. \[ -\log W_{t_k}(\theta) \leq
      N_c(k)\bar{\tau}_{c} \cdot D_c(\theta^\star || \theta) + \Gamma
      |\mathcal{S}|^2 \rho(N_c(k)) \sqrt{N_c(k)}.\]
    \end{enumerate}
    
  \end{lemma}
  
  \begin{proof}
    Since $t_k$ is an epoch boundary, $V_c(t_k) =
    \tilde{\tau}_{k'_c,c}$ for $k'_c := N_c(k)$. Using
    (\ref{eqn:wt2}), we can write
    \begin{align}
      -\log W_{t_k}(\theta) &= V_c(t_k) \sum_{(s_1,s_2) \in \mathcal{S}^2}
      U_{(s_1,s_2)}\left(V_c(t_k), c\right) \log
      \frac{p_{\theta^\star}(s_1,c(s_1),s_2)}{p_{\theta}(s_1,c(s_1),s_2)}  \nonumber \\
      &= \sum_{(s_1,s_2) \in \mathcal{S}^2} \tilde{\tau}_{k'_c,c} \cdot
      U_{(s_1,s_2)}\left(\tilde{\tau}_{k'_c,c}, c\right) \log
      \frac{p_{\theta^\star}(s_1,c(s_1),s_2)}{p_{\theta}(s_1,c(s_1),s_2)}  \nonumber \\
      &= \sum_{(s_1,s_2) \in \mathcal{S}^2} \left[\tilde{\tau}_{k'_c,c}
        \cdot U_{(s_1,s_2)}\left(\tilde{\tau}_{k'_c,c}, c\right) -
        k'_c\bar{\tau}_{c} \cdot {\pi^{(c)}_{(s_1,s_2)}}
      \right] \log
      \frac{p_{\theta^\star}(s_1,c(s_1),s_2)}{p_{\theta}(s_1,c(s_1),s_2)} \nonumber \\
      &\quad \quad \quad \quad + \sum_{(s_1,s_2) \in \mathcal{S}^2}
      k'_c\bar{\tau}_{c} \cdot {\pi^{(c)}_{(s_1,s_2)}} \log
      \frac{p_{\theta^\star}(s_1,c(s_1),s_2)}{p_{\theta}(s_1,c(s_1),s_2)} \nonumber \\
      &\geq - \sum_{(s_1,s_2) \in \mathcal{S}^2} \rho(k'_c) \sqrt{k'_c} \cdot
      \left| \log
        \frac{p_{\theta^\star}(s_1,c(s_1),s_2)}{p_{\theta}(s_1,c(s_1),s_2)}
      \right| \nonumber \\
      &\quad \quad \quad \quad + k'_c\bar{\tau}_{c} \sum_{s_1 \in \mathcal{S}}
      \pi^{(c)}_{s_1} \sum_{s_2 \in \mathcal{S}}
      \frac{\pi^{(c)}_{(s_1,s_2)}}{\pi^{(c)}_{s_1}} \log
      \frac{p_{\theta^\star}(s_1,c(s_1),s_2)}{p_{\theta}(s_1,c(s_1),s_2)} \nonumber \\
      &\geq k'_c\bar{\tau}_{c} \cdot D_c(\theta^\star || \theta) - \Gamma |\mathcal{S}|^2 \rho(k'_c) \sqrt{k'_c}, \label{eqn:wtbd1}
    \end{align}
%     The first inequality above is obtained thanks to (\ref{eqn:conc2})
%     of Proposition \ref{prop:conc}. For a fixed $\theta \neq
%     \theta^{\star}$ and $c$, the expression in (\ref{eqn:wtbd1}) tends
%     to $\infty$ as $k'_c \to \infty$. Denote the infimum of the
%     expression over all $k'_c \geq 1$ by $-\lambda_{\theta,c}$. The
%     lemma now follows by setting $\lambda$ to be the largest
%     $\lambda_{\theta,c}$ across the finitely many $\theta$ and $c$.

    where the final line is by the definition of event $G$ and by
    using Assumption \ref{ass:abscon}. This proves the first assertion
    of the lemma. The second assertion follows in a similar fashion.
  \end{proof}  

%%   Using the bound of Lemma \ref{eqn:wtbd1} in the expression for the
%%   posterior density (\ref{eqn:wt2}), we can bound (\ref{eqn:pitktrue})
%%   from below as
%%   \[\forall k \geq 1 \; \pi_{t_k}(\theta^\star) \geq \frac{\pi(\theta^\star)}{\int_{\Theta} \exp(\lambda |\mathcal{C}|) 
%%     \pi(d\theta)} = \pi(\theta^\star) e^{-\lambda |\mathcal{C}|}
%%   \equiv p^\star > 0, \quad \mbox{say.}\] 

\begin{lemma}[Bounded ratio of Log-likelihood and KL-divergence]
  \label{lem:llkl}
  Denote, for policy $c \in \mathcal{C}$ and parameter $\theta \in
  \Theta$, \[ L_c(\theta) \bydef \sum_{(s_1,s_2) \in \mathcal{S}^2}
  \left| \log
    \frac{p_{\theta^\star}(s_1,c(s_1),s_2)}{p_{\theta}(s_1,c(s_1),s_2)}
  \right|. \] There exists a universal constant $g$ such that
  \[\sup_{\substack{\theta \in \Theta, c \in \mathcal{C}\\D_c(\theta^\star || \theta)
      > 0}} \frac{L_c(\theta)}{\sqrt{D_c(\theta^\star || \theta)}} \leq g < \infty. \]
\end{lemma}

\begin{proof}
  By Assumption \ref{ass:abscon},
  $L_c(\theta) \leq \Gamma |\mathcal{S}|^2$, so it only suffices to
  bound from above the ratio $L_c(\theta)/D_c(\theta^\star || \theta)$
  for $\theta \to \theta^\star$. In this case, it is not hard to see
  that for $\theta = \theta^\star + \delta'$ for $|\delta'|$ small
  enough, $L_c(\theta) = O(\delta')$ while\footnote{This is the
    standard phenomenon of the ``local'' $||\cdot||^2_2$-like
    behaviour of the KL-divergence.}
  $D_c(\theta^\star || \theta) = O({\delta'}^2)$. Hence, the ratio
  $L_c(\theta)/D_c(\theta^\star || \theta)$ is bounded above by a
  universal constant, which completes the proof of the lemma.
\end{proof}

%   \begin{assumption}
%     \label{ass:llkl}
%     The probability transition kernels $p_\theta$, $\theta \in \Theta$
%     are such that there exists $g < \infty$ with
%     \[\sup_{\substack{\theta \in \Theta, c \in \mathcal{C}\\D_c(\theta^\star || \theta) > 0}} \frac{\sum_{(s_1,s_2) \in
%         \mathcal{S}^2} \left| \log
%         \frac{p_{\theta^\star}(s_1,c(s_1),s_2)}{p_{\theta}(s_1,c(s_1),s_2)}
%       \right|}{\sqrt{D_c(\theta^\star || \theta)}} \leq g.\]
%   \end{assumption}

Let $\Theta_1 \bydef \{\theta \in \Theta: \sum_{c \in \mathcal{C}}
N_c(k) \bar{\tau}_c D_c(\theta^\star || \theta) \leq 1\}$. By Assumption
\ref{ass:largedenominator}A, $\pi(\Theta_1) \geq a_1
k^{-a_2}$. \\

By the penultimate inequality in the derivation of Lemma
\ref{lem:wtbd1}, we have that under the event $G$, for any $\theta \in
\Theta_1$,
\begin{align*}
  -\log W_{t_k}(\theta) &\leq \sum_{c \in \mathcal{C}}
  N_c(k)\bar{\tau}_{c} D_c(\theta^\star || \theta) +
  \sum_{c \in \mathcal{C}} \rho(N_c(k)) \sqrt{N_c(k)} L_c(\theta) \\
  &\leq 1 + \sum_{c \in \mathcal{C}} \rho(N_c(k)) \sqrt{N_c(k)} L_c(\theta) \quad \mbox{(since $\theta \in \Theta_1$)}\\
  &\leq 1 + \sqrt{\sum_{c \in \mathcal{C}} N_c(k) \bar{\tau}_c
    D_c(\theta^\star || \theta)} \sqrt{\sum_{c \in \mathcal{C}}
    \frac{\rho^2(N_c(k))}{\bar{\tau}_c} \cdot \frac{L_c^2(\theta)}{D_c(\theta^\star || \theta)}  } \quad \mbox{(Cauchy-Schwarz inequality)}\\
  &\leq 1 + \rho(k) \sqrt{\sum_{c \in \mathcal{C}}
    \frac{L_c^2(\theta)}{D_c(\theta^\star || \theta)} } \quad \mbox{(since $\bar{\tau}_c \geq 1$ $\forall c \in \mathcal{C}$)}\\
  &\leq 1 + \rho(k) g \sqrt{|\mathcal{C}|},
\end{align*}
where $g$ is the constant guaranteed by Lemma \ref{lem:llkl}.  Thus,
under $G$,
\begin{align}
  \int_{\Theta} W_{t_k}(\theta') \pi(d\theta') &\geq \int_{\Theta_1}
  W_{t_k}(\theta') \pi(d\theta') \nonumber \\
  &\geq \int_{\Theta_1} e^{ -1 - \rho(k) g \sqrt{|\mathcal{C}|} }
  \; \pi(d\theta') \nonumber \\
  &= e^{ -1 - \rho(k) g \sqrt{|\mathcal{C}|} } \; \pi(\Theta_1) \nonumber \\
  &\geq e^{ -1 - \rho(k) g \sqrt{|\mathcal{C}|} } a_1 k^{-a_2} \geq
  a_1' k^{-a_2'} \label{eqn:noisylargedenominator}
\end{align}
for some suitable constants $a_1', a_2'$. \\

  We proceed to bound from above the posterior probability of $S_c''$,
  $c \neq c^\star$ under the event $G$. To this end, write
  \begin{align*}
    &\frac{W_{t_k}(\theta)}{\int_{\Theta} W_{t_k}(\theta')
      \pi(d\theta')} \leq
    \frac{W_{t_k}(\theta)}{a_1' k^{-a_2'}} \\
    &= \frac{1}{a_1' k^{-a_2'}} \exp \left[- \sum_{c \in
        \mathcal{C}} V_c(t_k) \sum_{s_1,s_2}
      U_{(s_1,s_2)}\left(V_c(t_k), c\right) \log
      \frac{p_{\theta^\star}(s_1,c(s_1),s_2)}{p_{\theta}(s_1,c(s_1),s_2)}
    \right] \\
%     &\leq \frac{\pi(\theta) e^{\lambda
%         |\mathcal{C}|}}{\pi(\theta^\star)} \exp \left[-
%       V_{c^\star}(t_k) \sum_{s_1,s_2}
%       U_{(s_1,s_2)}\left(V_{c^\star}(t_k), {c^\star} \right) \log
%       \frac{p_{\theta^\star}(s_1,{c^\star}(s_1),s_2)}{p_{\theta}(s_1,{c^\star}(s_1),s_2)}
%     \right] \\
    &\leq \frac{1}{a_1' k^{-a_2'}} \exp \left[- N_{c^\star}(k)
      \bar{\tau}_{{c^\star}} \cdot D_{c^\star}(\theta^\star || \theta) + \Gamma
      |\mathcal{S}|^2 \rho(N_{c^\star}(k)) \sqrt{N_{c^\star}(k)}
    \right],
  \end{align*}
  where, from Section \ref{sec:formal}, $N_{c^\star}(k)$ is the number
  of epochs up until epoch $k$ (i.e., until time instant $t_k$) in
  which the optimal policy $c^\star$ is chosen. The first inequality
  is by (\ref{eqn:noisylargedenominator}). The second inequality
  results by applying the conclusion of Lemma \ref{lem:wtbd1} to all
  policies $c \neq c^\star$. Using the uniform lower bound
  $D_{c^\star}(\theta^\star || \theta) \geq \epsilon'$ $\forall \theta \in S_c''$ and
  integrating the above inequality over $\theta \in S_c''$ gives the
  bound
  \[ \pi_{t_k}(S_c'') \leq \nu_k \exp \left[ - \epsilon' N_{c^\star}(k)
    \bar{\tau}_{{c^\star}} + \Gamma |\mathcal{S}|^2
    \rho(N_{c^\star}(k)) \sqrt{N_{c^\star}(k)} \right], \] with $\nu_k
  \bydef \frac{1}{a_1' k^{-a_2'}}$. The key property of the above estimate
  is that it decays exponentially with $N_{c^\star}(k)$. (Intuitively,
  since $\theta^\star$ is sampled with frequency at least $p^\star$,
  we expect that $N_{c^\star}(k) \approx k p^\star$, and thus the
  estimate is also exponential in $k$.) \\
%   This should mean that the total number of times $S_c''$ is chosen
%   should be no more than a constant (i.e., independent of the time
%   horizon $T$).

  \begin{proof}[Proof of Proposition \ref{prop:Sc''}]
    We begin by estimating the moment generating function of
    $N_{c^\star}(k)$. Let $\mathcal{F}_t$ denote the $\sigma$-algebra
    generated by the history of the algorithm up to time $t$ and state
    $S_t$, i.e., the $\sigma$-algebra generated by the random
    variables \[\left\{(S_0,A_0,R_0), \ldots, (S_{t-1},A_{t-1},R_{t-1}), S_t\right\}. \] We have
    \begin{align*}
      \expect{e^{-\epsilon' N_{c^\star}(k)} \given G} &= \expect{\expect{e^{-\epsilon' N_{c^\star}(k)} \given \mathcal{F}_{t_{k-1}}, G} \given G} \\
      &= \expect{ e^{-\epsilon' N_{c^\star}(k-1)}
        \expect{e^{-\epsilon' \mathbbm{1}\{C_k = c^\star \}} \given
          \mathcal{F}_{t_{k}},
          G} \given G} \\
      &\leq \expect{ e^{-\epsilon' N_{c^\star}(k-1)}
        \expect{e^{-\epsilon' \mathbbm{1}\{\theta_{t_k} \in S_{c^\star}
            \}} \given \mathcal{F}_{t_{k}},
          G} \given G} \\
      &\leq \expect{ e^{-\epsilon' N_{c^\star}(k-1)} \left( p^\star
          e^{-\epsilon'} + 1
          - p^\star \right) \given G} \\
      &= \left( p^\star e^{-\epsilon'} + 1 - p^\star \right) \expect{
        e^{-\epsilon' N_{c^\star}(k-1)} \given G},
    \end{align*}
    where, in the penultimate step, we have used the fact that the
    probability of sampling $\theta^\star$ under $G$ is at least
    $p^\star$ at all epoch boundaries (Assumption
    \ref{ass:largeprob}). Iterating the estimate further gives
    \[ \expect{e^{-\epsilon' N_{c^\star}(k)} \given G} \leq \left( p^\star
      e^{-\epsilon'} + 1 - p^\star \right)^k. \] Using this with the
    conditional version of Markov's inequality, we have, for $c \neq
    c^\star$ and $\chi > 0$,
    \begin{align*}
      &\prob{\sum_{k=1}^\infty \mathbbm{1}\{\theta_k \in S_c''\} > \chi
        \given G} \leq \chi^{-1} \expect{\sum_{k=1}^\infty
        \mathbbm{1}\{\theta_k \in S_c''\} > \chi \given G} \\
      &= \chi^{-1} \sum_{k=1}^\infty \expect{ \mathbbm{1}\{\theta_k
        \in S_c''\} >
        \chi \given G} \\
      &\leq \chi^{-1} \sum_{k=1}^\infty \left(1 \wedge \expect{\nu_k e^{- \epsilon'
            N_{c^\star}(k) \bar{\tau}_{{c^\star}} + \Gamma |\mathcal{S}|^2
            \rho(N_{c^\star}(k)) \sqrt{N_{c^\star}(k)} } \given G} \right) \\
      &\leq \chi^{-1} \sum_{k=1}^\infty \left(1 \wedge \expect{\nu_k e^{- \epsilon'
            N_{c^\star}(k) \bar{\tau}_{{c^\star}} + \Gamma |\mathcal{S}|^2
            \rho(k) \sqrt{k} } \given G} \right) \\
      &\leq \chi^{-1} \sum_{k=1}^\infty \left(1 \wedge \nu_k \left( p^\star
      e^{-\epsilon'} + 1 - p^\star \right)^k e^{\Gamma |\mathcal{S}|^2
            \rho(k) \sqrt{k}} \right).
    \end{align*}
    Note that since $p^\star$ and $\epsilon'$ are positive, $p^\star
    e^{-\epsilon'} + 1 - p^\star < 1$. Moreover, since both
    $\rho(k)\sqrt{k} = o(k)$ and $\log \nu_k = o(k)$, the sum above is
    dominated by a convergent geometric series after finitely many
    $k$, and is thus a finite quantity $\alpha < \infty$. Taking a
    union bound over all $c \neq c^\star$ completes the proof of
    Proposition \ref{prop:Sc''}.
  \end{proof}
 
  \subsection{Regret due to sampling from $S_c'$}
  We now turn to bounding the number of times that parameters from
  $S_c'$ with $c \neq c^\star$ are sampled by the TSMDP algorithm.

%   \begin{assumption}[``Slower'' ... ]
%     \label{ass:largerdenominator}
%       There exist $a_3 > 0, a_4 \geq 0$ such that 
%       \[ \int_{\Theta} W_{t_k}(\theta') \pi(d\theta') \geq a_3
%       k^{-a_4}\] whenever
%   \[ \left| J_{(s_1,s_2)}(k_c,c) - V_c(t_k)\; \pi_{(s_1,s_2)} \right|
%   \leq \rho(k_c)\sqrt{k_c}, \; \forall s_1, s_2 \in \mathcal{S}, c
%   \in \mathcal{C}, k_c \geq 0, \sum_{c \in \mathcal{C}} k_c = k, k_{c^\star} \geq k - \log^2(k). \]
%   \end{assumption}

  We begin with the following key lemma, which helps to give a more
  refined estimate of the posterior weight exponent compared to Lemma
  \ref{lem:wtbd1}.

  \begin{lemma}
    \label{lem:bd1}
    Fix $\epsilon \in (0,1)$. By Assumption \ref{ass:ergodic} and
    Lemma \ref{lem:llkl}, it holds under the event $G$ that for each $\theta
    \in \Theta$, $c \in \mathcal{C}$ and $T \geq n \geq 1$,
    \begin{align*}
      V_c(t_k) \sum_{(s_1,s_2) \in \mathcal{S}^2} & U_{(s_1,s_2)}
      \left(V_c(t_k), c\right) \log
      \frac{p_{\theta^\star}(s_1,c(s_1),s_2)}{p_{\theta}(s_1,c(s_1),s_2)}
      \\
      & \geq (1-\epsilon) N_c(k)\bar{\tau}_{c} D_c(\theta^\star || \theta) -
      \frac{g^2 d_1}{4 \epsilon }
      \log\left(\frac{|\mathcal{C}| |\mathcal{S}|^2 d_2 \log
          T}{\delta} \right).
    \end{align*}

  \end{lemma}
  The usefulness of the result stems from the fact that the left-hand
  term (which in fact helps to form the posterior log-density of
  $\theta$) can be approximated by a constant fraction of the marginal
  KL divergence $D_c(\theta^\star || \theta)$, with the approximation error being only
  $O\left( \frac{\log \log T}{\epsilon} \right)$. 

  \begin{proof}
    Denote $L_c(\theta) \bydef \sum_{(s_1,s_2) \in \mathcal{S}^2}
    \left| \log
      \frac{p_{\theta^\star}(s_1,c(s_1),s_2)}{p_{\theta}(s_1,c(s_1),s_2)}
    \right|$. By the penultimate inequality in the derivation of Lemma
    \ref{lem:wtbd1}, we have that under the event $G$,

    \begin{align*}
      &V_c(t_k) \sum_{(s_1,s_2) \in \mathcal{S}^2}
      U_{(s_1,s_2)}\left(V_c(t_k), c\right) \log
      \frac{p_{\theta^\star}(s_1,c(s_1),s_2)}{p_{\theta}(s_1,c(s_1),s_2)} \\
      &\geq k'_c\bar{\tau}_{c} \cdot D_c(\theta^\star || \theta) - \rho(k'_c)
      \sqrt{k'_c} \sum_{(s_1,s_2) \in \mathcal{S}^2} \left| \log
        \frac{p_{\theta^\star}(s_1,c(s_1),s_2)}{p_{\theta}(s_1,c(s_1),s_2)}
      \right|,
    \end{align*}
    where $k'_c := N_c(k)$. The right hand side of the inequality
    above is of the form $a x - b \rho(x)\sqrt{x}$ if we identify $x
    \equiv k'_c \in [0,T]$, $a \equiv \bar{\tau}_{c} \cdot
    D_c(\theta^\star || \theta)$ and $b \equiv L_c(\theta)$. To prove the lemma, it is
    enough to find $\gamma$ such that $ax - b \rho(x)\sqrt{x} \geq
    (1-\epsilon) ax - \gamma$ for every choice of $\theta \in \Theta$
    and $c \in \mathcal{C}$. This is equivalent to requiring that
    $\gamma \geq -\epsilon a x + b \rho(x)\sqrt{x}$. Consider now
    \begin{align*}
      &\sup_{T \geq x \geq 0} \left[-\epsilon a x + b
        \rho(x)\sqrt{x}\right] \leq \sup_{T \geq x \geq 0}
      \left[-\epsilon a x + b \sqrt{d_1 x
          \log\left(\frac{|\mathcal{C}| |\mathcal{S}|^2 d_2 \log
              x}{\delta}
          \right)}\right] \\
      &\leq \sup_{T \geq x \geq 0} \left[-\epsilon a x + b \sqrt{d_1 x
          \log\left(\frac{|\mathcal{C}| |\mathcal{S}|^2 d_2 \log
              T}{\delta} \right)}\right] \\
      &\leq \sup_{x \in \mathbb{R}} \left[-\epsilon a x + b \sqrt{d_1
          x \log\left(\frac{|\mathcal{C}| |\mathcal{S}|^2 d_2 \log
              T}{\delta} \right)}\right] \\
      &= \frac{b^2 {d_1  \log\left(\frac{|\mathcal{C}|
              |\mathcal{S}|^2 d_2 \log T}{\delta} \right)}}{4 \epsilon a},
    \end{align*} 
    where the final step simply finds the maximum of the quadratic
    function over $x$. The only quantities depending on $\theta$ in the
    right hand side above are $a$ and $b$, so maximizing over $\theta
    \in \Theta$ for which $a \equiv \bar{\tau}_{c} \cdot D_c(\theta^\star || \theta)
    > 0$, we further obtain
    \begin{align*}
      &\sup_{\substack{\theta \in \Theta, c \in
          \mathcal{C}\\D_c(\theta^\star || \theta) > 0}} \sup_{T \geq x \geq 0}
      \left[-\epsilon a x + b \rho(x)\sqrt{x}\right] \leq \frac{d_1
      }{4 \epsilon} \log\left(\frac{|\mathcal{C}| |\mathcal{S}|^2 d_2
          \log T}{\delta} \right) \sup_{\substack{\theta \in \Theta, c
          \in \mathcal{C}\\D_c(\theta^\star || \theta) > 0}}
      \left(\frac{L_c^2(\theta)}{\bar{\tau}_{c} D_c(\theta^\star || \theta)}
      \right) \\
      &\leq \frac{g^2 d_1 }{4 \epsilon} \log\left(\frac{|\mathcal{C}|
          |\mathcal{S}|^2 d_2 \log T}{\delta} \right),
    \end{align*}
    where we have used Assumption \ref{ass:ergodic} and
    Lemma \ref{lem:llkl} in the final step. This proves the statement of the
    lemma.
  \end{proof}

%   \begin{lemma}[Approximating the pair-empirical KL divergence its
%     expectation -- the actual KL divergence]
%     \label{lem:wtbd2}
%     Let $\epsilon \in (0,1)$ be fixed. Then, there exists an integer
%     $k_0$ such that, under the event $G$, for each $\theta \in
%     \Theta$, $c \in \mathcal{C}$ and $k \geq k_0$,
%     \[ V_c(t_k) \sum_{(s_1,s_2) \in \mathcal{S}^2}
%     U_{(s_1,s_2)}\left(V_c(t_k), c\right) \log
%     \frac{p_{\theta^\star}(s_1,c(s_1),s_2)}{p_{\theta}(s_1,c(s_1),s_2)}
%     \geq (1-\epsilon) \cdot N_c(k) \cdot \bar{\tau}_{c} \cdot D_c(\theta^\star || \theta).\]
%   \end{lemma}
%   \begin{proof}
%     Notice that by (\ref{eqn:wtbd1}),
%     \[ V_c(t_k) \sum_{s_1,s_2} U_{(s_1,s_2)}\left(V_c(t_k), c\right)
%     \log
%     \frac{p_{\theta^\star}(s_1,c(s_1),s_2)}{p_{\theta}(s_1,c(s_1),s_2)}
%     \geq N_c(k) \cdot \bar{\tau}_{c} \cdot D_c(\theta^\star || \theta) - \Gamma
%     |\mathcal{S}|^2 \rho(N_c(k)) \sqrt{N_c(k)}. \] Now, since
%     $\rho(N_c(k)) \sqrt{N_c(k)} = o(k)$, there exists an integer, say
%     $k_{0,\theta,c}$, such that \[\forall x \geq k_{0,\theta,c} \quad
%     \Gamma |\mathcal{S}|^2 \rho(x) \sqrt{x} \leq \epsilon \cdot x
%     \cdot \bar{\tau}_{c} \cdot D_c(\theta^\star || \theta)\] By choosing $k_0 =
%     \max_{\theta,c} k_{0,\theta,c}$, the proof is complete.
%   \end{proof}

  We will henceforth fix $\epsilon \in (0,1)$ as per Lemma
  \ref{lem:bd1}. A consequence of Lemma \ref{lem:bd1} is the following
  bound, under the event $G$, on the posterior density for any
  parameter $\theta \in \Theta$ at the epoch boundary times $\{t_k\}$:
  \begin{equation}
    \label{eqn:wtbd3}
    W_{t_k}(\theta) \mathbbm{1}_G \leq e^{-\sum_{c \in \mathcal{C}} \phi_{\theta,c}(N_c(k))} \leq e^{-\sum_{c \in \mathcal{C}} \phi_{\theta,c}(N_c'(k))},
  \end{equation}
  where for each $\theta$ and $c$, $\phi_{\theta,c}(x) \bydef
  (1-\epsilon)x \bar{\tau}_{c} D_c(\theta^\star || \theta) - \frac{g^2 d_1}{4
    \epsilon } \log\left(\frac{|\mathcal{C}| |\mathcal{S}|^2 d_2 \log
      T}{\delta} \right) := (1-\epsilon)x \bar{\tau}_{c} D_c(\theta^\star || \theta)
  - \psi_{\epsilon,T}$, and with the $O\left( \frac{\log \log T}{\epsilon} \right)$
  correction term $\psi_{\epsilon,T}$ thanks to
  Lemma \ref{lem:bd1}. \\

  We proceed to define the following sequence of non-decreasing
  stopping times (more precisely, stopping {\em epochs}), which we
  term ``elimination times'', and their associated policies in
  $\mathcal{S}$.

Let ${\hat{\tau}}_0 \bydef 0$, $M'_0 \bydef (0, 0, \ldots, 0) \in
\mathbb{R}^{|\mathcal{C}|}$, and $\mathcal{C}_0 \bydef \emptyset$. For
each $l = 1, \ldots, |\mathcal{C}|-1$, set
\begin{equation}
\label{eqn:deftau}
\begin{aligned}
  {\hat{\tau}}_l \bydef \; &\min && k \geq {\hat{\tau}}_{l-1} \\
  &\text{ s.t.}
  && \exists \mathbbm{c}_l \in \mathcal{C} \setminus \left(\mathcal{C}_{l-1} \cup \{c^\star\}\right) \quad \forall {\theta \in S_{\mathbbm{c}_l}'}: \\
  &&& \sum_{m=1}^{l-1}
  M_{\mathbbm{c}_m}'({\hat{\tau}}_m) \bar{\tau}_{\mathbbm{c}_m}  D_{\mathbbm{c}_m}(\theta) + \sum_{\substack{c \notin
      \mathcal{C}_{l-1} }} N_c'(k) \bar{\tau}_{c} D_c(\theta^\star || \theta) \geq
  (1+a_4)\left(\frac{1+\epsilon}{1-\epsilon} \right) \log T,
\end{aligned}
\end{equation}
\begin{align}
  \label{eqn:defCl}
  \mathcal{C}_{l} &\bydef \mathcal{C}_{l-1}
  \cup \{\mathbbm{c}_l\},
\end{align} [Note that $a_4$ in (\ref{eqn:deftau}) is the constant
from Assumption \ref{ass:largedenominator}(B).] and where the
$|\mathcal{C}|$-dimensional non-negative vector $M'(\hat{\tau}_l)
\equiv \left(M'_c(\hat{\tau}_l)\right)_{c \in \mathcal{C}}$ is defined
as follows. For each $\mathbbm{c}_m$ such that $m \leq l-1$, define
$M'_{\mathbbm{c}_m}(\hat{\tau}_l) \bydef
M'_{\mathbbm{c}_m}(\hat{\tau}_m)$. Recall that $C_{\hat{\tau}_l}$
denotes the policy which was played at epoch $\hat{\tau}_l$, and which
led to the stopping time $\hat{\tau}_l$ being reached by satisfying
inequality (\ref{eqn:deftau}). For each $c \neq \mathcal{C}_{l-1}$ and
$c \neq C_{\hat{\tau}_l}$, let $M'_c(\hat{\tau}_l) \bydef
N'_c(\hat{\tau}_l)$. Finally, for $c = C_{\hat{\tau}_l}$, put
$M'_c(\hat{\tau}_l) \bydef x$, where $x$ is the unique real number in
the interval $\left[N'_c(\hat{\tau}_l)-1, N'_c(\hat{\tau}_l)\right]$
that satisfies\footnote{In case of non-uniqueness, i.e., if more than
  one $\mathbbm{c}_l \in \mathcal{C} \setminus \left(\mathcal{C}_{l-1}
    \cup \{c^\star\}\right)$ exists that satisfies (\ref{eqn:deftau})
  at epoch $\hat{\tau}_l$, then we proceed by choosing $\mathbbm{c}_l$
  for which the value of $x$ in (\ref{eqn:Mc'propty}) is the {\em
    least}.}
\begin{equation} 
\label{eqn:Mc'propty}
\sum_{c \neq C_{\hat{\tau}_l}}
  M_{c}'({\hat{\tau}}_l) \bar{\tau}_{c} D_{c}(\theta) + x \cdot \bar{\tau}_{C_{\hat{\tau}_l}} D_{C_{\hat{\tau}_l}}(\theta ) =
  (1+a_4)\left(\frac{1+\epsilon}{1-\epsilon} \right) \log T. 
\end{equation}
{\em Remark:} The purpose of defining the vectors $M'(\hat{\tau}_l)$, $l
= 1, 2, \ldots, |\mathcal{C}|-1$ is to essentially convert the
inequality in (\ref{eqn:deftau}) to the equality (\ref{eqn:Mc'propty})
by relaxing from {\em integers} $N'$ to {\em reals} $M'$. At the same
time, we maintain the point-wise dominance $M'(\hat{\tau}_l) \leq
N'(\hat{\tau}_l)$. We will require precisely these properties in the
proof of Proposition \ref{prop:optim}. \\

In other words, for each $l$, $\mathcal{C}_l$ represents the set of
the first $l$ ``eliminated'' suboptimal policies. ${\hat{\tau}}_l$ is
the first time\footnote{All the $\hat{\tau}_l$, $l \geq 0$ index {\em
    epochs} w.r.t. the TSMDP algorithm, but we will refer to them as
  ``times''. This distinction should be clear throughout.} after
${\hat{\tau}}_{l-1}$, when some suboptimal policy (which is not
already eliminated) gets eliminated\footnote{In case more than one
  suboptimal policy is eliminated at some ${\hat{\tau}}_l$, we use a
  predetermined tie-breaking rule among $\mathcal{C}$ to resolve the
  tie.  } by satisfying the inequality in
(\ref{eqn:deftau}). Essentially, the inequality checks whether the
condition
\[ \sum_{c} N_c'(k) \bar{\tau}_{c} D_c(\theta^\star || \theta) \approx \log T \] is
satisfied for all particles $\theta \in S_{\mathbbm{c}_l}'$ at epoch
$k$, with two slight modifications -- (a) the play count $N_c'(k)$ is
``frozen'' to $N_{c}'({\hat{\tau}}_m)$ if action $c$ has been
eliminated at an earlier time ${\hat{\tau}}_m \leq k$, and (b) paying
a multiplicative penalty factor of
$(1+a_4)\left(\frac{1+\epsilon}{1-\epsilon} \right)$ on
the right hand side. \\

Thus, ${\hat{\tau}}_0 \leq {\hat{\tau}}_1 \leq \ldots \leq
{\hat{\tau}}_{|\mathcal{C}|-1}$, and $\mathcal{C}_0 \subseteq
\mathcal{C}_1 \subseteq \ldots \subseteq \mathcal{C}_{|\mathcal{C}|-1}
= \mathcal{C} \setminus \{c^\star\}$. For each policy $c \neq c^\star$, by our
definitions above, there exists a unique ${\hat{\tau}}_l$ at which $c$
is eliminated at ${\hat{\tau}}_l$, i.e., $\mathbbm{c}_l = c$. Let the
notation
${\hat{\tau}}(c) \bydef {\hat{\tau}}_l$ denote the elimination time for policy $c$.

\begin{definition}[Minimum ``resolvability'' of suboptimal actions]
  \label{def:emin}
  We define \[\epsilon_{\min} \bydef \min_{c\in \mathcal{C}, c \neq c^\star}
  \min_{\theta \in S_c'} D_c(\theta^\star || \theta). \]
\end{definition}

% \begin{assumption}
%   \label{ass:selfdistance}
%   There exists $\epsilon_{\min} > 0$ such that
%   \[\forall c \neq c^\star \; \inf_{\theta \in S_c'} D_c(\theta^\star || \theta)
%   \geq \epsilon_{\min} \geq 0. \]
% \end{assumption}

Observe that if $\epsilon_{\min} = 0$, then the optimization problem
(\ref{eqn:main}) in the regret bound of Theorem \ref{thm:main} has
value $\infty$. This is because if $D_c(\theta^\star || \theta) = 0$ for some $\theta
\in S_c'$ with $c \neq c^\star$, then one can obtain arbitrarily large
solutions to (\ref{eqn:main}) simply by considering all vectors $x_l
\in \mathbb{R}^{|\mathcal{C}|}_+$, $l = 1, 2, \ldots,
|\mathcal{C}|-1$, to be of the form $(x, 0, \ldots, 0)$.  \\

Thus, we proceed by assuming that the regions $S_c'$ and $S_c''$, $c
\in \mathcal{C}$ (induced by the parameter $\epsilon'$) are such that
the minimum resolvability parameter $\epsilon_{\min}$ is a positive
quantity.

\begin{lemma}
  \label{lem:atmostlog}
  We have that
  \[ N_{\mathbbm{c}_l}'(\hat{\tau}_l) \leq \left\lceil
    \frac{(1+a_4)(1+\epsilon)}{\epsilon_{\min} (1 - \epsilon)} \log T
  \right\rceil + 1 \] for each $l = 1, 2, \ldots, |\mathcal{C}|-1$.
\end{lemma}
\begin{proof}
  Assuming the contrary leads to equation (\ref{eqn:deftau}) being contradicted.
\end{proof}

The following important lemma states that after a policy $c$ is
eliminated, the TSMDP algorithm does not sample parameters from the
region $S_c'$ for too many epochs, with high probability.

\begin{lemma}[At most $O(1)$ samples from $S_c'$ after policy $c$ is
  eliminated]
  \label{lem:notmuchSa'}
  For $\epsilon \in (0,1)$ and $T$ large enough so that
  $\mathcal{C}\left( 1 + \left\lceil
      \frac{(1+a_4)(1+\epsilon)}{\epsilon_{\min} (1 - \epsilon)} \log
      T \right\rceil \right) \leq \log^2 (T)$, it holds that
  \[ \prob{\exists l \in \{1, 2, \ldots, |\mathcal{C}|-1\} \; \sum_{k \geq {\hat{\tau}}_l+1} \mathbf{1}\{\theta_k
    \in S_{\mathbbm{c}_l}'\} > \frac{|\mathcal{C}|}{\delta a_3} + o(1) \given G} \leq \delta.\]
\end{lemma}
% \noindent [Remark: Since $\psi_{\epsilon,T} = O(\log \log T)$, $e^{\psi_{\epsilon,T} |\mathcal{C}|} =
% O(\mbox{poly}(\log \log T)) = o(\log T)$.]

\begin{proof}
  Whenever $k > {\hat{\tau}}_l$, we have that every $\theta \in
  S_{\mathbbm{c}_l}'$ satisfies
  \begin{align}
    W_{t_k}(\theta) \mathbbm{1}_{G} &\leq \exp\left(-\sum_{c \in
        \mathcal{C}}
      \phi_{\theta,c}(N_c'(k))\right) \nonumber \\
    &= \exp\left(-\sum_{c \in \mathcal{C}} \left((1-\epsilon)
        N_c'(k) \bar{\tau}_{c} D_c(\theta^\star || \theta) - \psi_{\epsilon,T} \right) \right) \nonumber \\
    &= \exp\left(-(1-\epsilon)\sum_{c \in \mathcal{C}}
      N_c'(k) \bar{\tau}_{c} D_c(\theta^\star || \theta) + \psi_{\epsilon,T} |\mathcal{C}|  \right) \nonumber \\
    &\leq \exp\left(-(1-\epsilon)\sum_{c \in \mathcal{C}_{l-1}}
      N_{c}'\bar{\tau}_{c} ({\hat{\tau}}(c))D_{c}(\theta)
      -(1-\epsilon) \sum_{\substack{c \notin \mathcal{C}_{l-1}}}
      N_c'(k) \bar{\tau}_{c} D_c(\theta^\star || \theta) + \psi_{\epsilon,T} |\mathcal{C}| \right) \nonumber \\
    &\leq \exp\left( -(1-\epsilon)
      (1+a_4)\left(\frac{1+\epsilon}{1-\epsilon} \right) \log T +
      \psi_{\epsilon,T} |\mathcal{C}| \right) =
    \frac{e^{\psi_{\epsilon,T}
        |\mathcal{C}|}}{T^{1+a_4}} \; e^{-\epsilon (1+a_4) \log T} \nonumber \\
    &\leq T^{-(1+a_4)}.   \label{eqn:wtbd4}
  \end{align}

  The first inequality in the display above follows from
  (\ref{eqn:wtbd3}). The second inequality is due to the fact that for
  any $m \leq l$, we have ${\hat{\tau}}_m \leq {\hat{\tau}}_l \leq k$,
  implying that $\forall c \in \mathcal{C}_{l-1}$, $N_c'(k) \geq
  N_c'({\hat{\tau}(c)})$. The third inequality follows from
  (\ref{eqn:deftau}). The final inequality above holds for $T$ large
  enough such that \[ \epsilon(1+a_4)\log T \geq \psi_{\epsilon,T} =
  \frac{g^2 d_1}{4 \epsilon } \log\left(\frac{|\mathcal{C}|
      |\mathcal{S}|^2 d_2 \log T}{\delta} \right). \]

  Now, define the nonnegative integer-valued random variable
  \[ K_B = \min\left\{ k \geq 0: \sum_{c \neq c^\star} N_c(k) >
    3\log^2(T) \right\}, \]
  i.e., $K_B$ is the first epoch at which suboptimal policies have
  been chosen in at least $2\log^2(T)$ previous epochs. Let us
  estimate
  \begin{align}
    &\expect{\mathbbm{1}\{k > {\hat{\tau}}_l\}\mathbbm{1}\{\theta_k
      \in S_{\mathbbm{c}_l}' \} \mathbbm{1}\{ k < K_B\} \given G} \nonumber \\
    &= \expect{ \expect{ \mathbbm{1}\{k >
        {\hat{\tau}}_l\}\mathbbm{1}\{\theta_k \in S_{\mathbbm{c}_l}'\} \mathbbm{1}\{ k < K_B\} \given G, \mathcal{F}_{t_k} } \given G} \nonumber \\
    &= \expect{\mathbbm{1}\{k > {\hat{\tau}}_l\} \mathbbm{1}\{ k <
      K_B\} \pi_{t_k}(S_{\mathbbm{c}_l}') \given G} =
    \expect{\mathbbm{1}\{k > {\hat{\tau}}_l\} \mathbbm{1}\{ k < K_B\}
      \frac{\int_{S_{\mathbbm{c}_l}'} W_{t_k}(\theta)\pi(d\theta)
      }{\int_\Theta W_{t_k}(\theta)\pi(d\theta)} \given G} \nonumber \\
    &\leq \expect{\mathbbm{1}\{k > {\hat{\tau}}_l\}
      \frac{\int_{S_{\mathbbm{c}_l}'} W_{t_k}(\theta)\pi(d\theta)}{a_3
        T^{-a_4}} \given G} \quad \quad \mbox{(by Assumption \ref{ass:largedenominator}(B))} \nonumber \\
    &\leq \frac{1}{a_3 T^{1+a_4 - a_4}} =
    \frac{1}{a_3 T} \quad \quad \mbox{(by
      (\ref{eqn:wtbd4})).} \label{eqn:1byTprob}
  \end{align}
  Together with the fact that the epoch index is at most $T$ for a
  time horizon of $T$ time steps, this implies that
  \begin{align}
    &\expect{ \sum_{T \geq k\geq{\hat{\tau}}_l+1} \mathbbm{1}\{\theta_k \in
      S_{\mathbbm{c}_l}' \} \mathbbm{1}\{ k < K_B\} \given G} \nonumber \\
    &\quad \quad \quad \quad \quad = \sum_{k=1}^T
    \expect{\mathbbm{1}\{k > {\hat{\tau}}_l\}\mathbbm{1}\{\theta_k \in
      S_{\mathbbm{c}_l}' \} \mathbbm{1}\{ k < K_B\} \given G} \leq T
    \cdot \frac{1}{a_3 T} = \frac{1}{a_3}. \label{eqn:bda}
\end{align}
In a similar fashion, considering plays of all suboptimal policies
$\mathcal{C}\setminus \{c^\star\}$ post their respective elimination
times, we can write
\begin{align}
  &\expect{ \sum_{l=1}^{|\mathcal{C}|-1} \sum_{T \geq
      k\geq{\hat{\tau}}_l+1} \mathbbm{1}\{\theta_k \in
    S_{\mathbbm{c}_l}' \} \mathbbm{1}\{ k \geq K_B\} \given G} \nonumber \\
  &= \expect{ \sum_{l=1}^{|\mathcal{C}|-1} \sum_{k=1}^T \mathbbm{1}\{k
    > {\hat{\tau}}_l\} \mathbbm{1}\{\theta_k \in S_{\mathbbm{c}_l}' \}
    \mathbbm{1}\{ k \geq K_B\} \given G} \leq \expect{ \sum_{k=1}^T
    \mathbbm{1}\{ K_B < T\} \given G}  \nonumber \\
  &= T \prob{K_B < T \given G}. \label{eqn:postelim1}
\end{align}
We have
\begin{align}
  &\prob{K_B < T \given G}  = \prob{\exists 1 \leq k \leq T: \; \sum_{c
      \neq c^\star} N_c(k) > 3 \log^2(T) \given G} \quad \mbox{(by the defn. of $K_B$)}.  \nonumber
\end{align}
Continuing the calculation further, we can write
\begin{align}
  &\prob{\exists 1 \leq k \leq T: \; \sum_{c \neq c^\star} N_c(k) > 3
    \log^2(T) \given G} \nonumber \\
  &= \prob{\exists 1 \leq k \leq T: \; k \leq K_B, \sum_{c \neq c^\star} N_c'(k) + \sum_{c \neq c^\star} N_c''(k) > 3
    \log^2(T) \given G} \nonumber \\
  &\leq \prob{ \exists 1 \leq k \leq T: k \leq K_B, \sum_{c \neq c^\star} N_{c}'(k)
    > 2 \log^2(T) \given G} + \prob{\sum_{c \neq c^\star} N_c''(T) > 
    \log^2(T) \given G}  \nonumber \\
  &\leq \prob{ \exists 1 \leq k \leq T: k \leq K_B, \sum_{c \neq c^\star} N_{c}'(\hat{\tau}(c)) + \sum_{c \neq c^\star} \left[ N_{c}'(k) - N_{c}'(k \wedge \hat{\tau}(c)) \right] 
    > 2 \log^2(T) \given G}  \nonumber \\
  &\hspace{2cm} + \prob{\sum_{c \neq c^\star} N_c''(T) > 
    \log^2(T) \given G}  \nonumber \\
  &\stackrel{(a)}{\leq} \prob{ \exists 1 \leq k \leq T: k \leq K_B, \sum_{c \neq c^\star} \left[ N_{c}'(k) - N_{c}'(k \wedge \hat{\tau}(c)) \right] 
    > \log^2(T) \given G}  \nonumber \\
  &\hspace{2cm} + \prob{\sum_{c \neq c^\star} N_c''(T) > 
    \log^2(T) \given G}  \nonumber \\
  &\leq \prob{ \exists 1 \leq k \leq T: k \leq K_B, \sum_{c \neq c^\star} \sum_{j=\hat{\tau}(c) + 1}^k \mathbbm{1}\{\theta_j \in S_{c}'\}  
    > \log^2(T) \given G}  \nonumber \\
  &\hspace{2cm} + \prob{\sum_{c \neq c^\star} N_c''(T) > 
    \log^2(T) \given G} \nonumber \\
%   &\stackrel{(a)}{\leq} \prob{ \sum_{c \neq c^\star} \sum_{k=1}^T \mathbbm{1}\{\theta_k \in S_{c}', \hat{\tau}(c) < k \leq K_B\}
%     >  \log^2(T) \given G}  + ... \nonumber \\
  &\stackrel{(b)}{\leq} \prob{\sum_{k=1}^T \mathcal{Q}_k >
    \log^2(T)} + \prob{\sum_{c \neq c^\star} N_c''(T) > 
    \log^2(T) \given G} , \label{eqn:postelim2}
\end{align}
where $\{\mathcal{Q}_k\}$ are IID Bernoulli random variables with
success probability
$p_{\mathcal{Q}} \bydef \frac{|\mathcal{C}|}{a_3 T}$. Inequality $(a)$
follows from the assertion of Lemma \ref{lem:atmostlog} and the
hypothesis that $T$ is large enough to satisfy
$\mathcal{C}\left( 1 + \left\lceil
    \frac{(1+a_4)(1+\epsilon)}{\epsilon_{\min} (1 - \epsilon)} \log T
  \right\rceil \right) \leq \log^2 (T)$.
Inequality (b) is thanks to the observation that (i) as long as
$\hat{\tau}(c) < j \leq k \leq K_B$, the probability of sampling
$\theta_k \in S_{c}'$ for any $c \neq c^\star$, under $G$, is at most
$\frac{1}{a_3 T}$ by (\ref{eqn:1byTprob}), and (ii) then using a
standard stochastic dominance argument after coupling
$\mathbbm{1}\{\theta_j \in S_{c}'\}$ to the IID
Bernoulli$\left(\frac{|\mathcal{C}|}{a_3
    T}\right)$ random variables $\{\mathcal{Q}_k\}$.  \\

\noindent {\bf Estimating the first term in (\ref{eqn:postelim2}).}
We can now show that the first term in (\ref{eqn:postelim2}) is $o(1)$
using a version of Bernstein's inequality
\citep{BouLugBou04:concineq}: For zero-mean independent random
variables $\mathcal{Z}_1, \mathcal{Z}_2, \ldots, \mathcal{Z}_n$ almost
surely bounded above by $\mathcal{B}$, and
$\Sigma^2 \bydef \frac{1}{n} \sum_{i=1}^n \expect{\mathcal{Z}_i^2}$,
\[\prob{\sum_{i=1}^n \mathcal{Z}_i \geq n\iota} \leq
\exp\left(-\frac{n\iota^2}{2\Sigma^2 + 2 \mathcal{B}\iota/3} \right).\]
Applying this to our setting with Bernoulli random variables,
$\mathcal{B} = 2$ and $\Sigma^2 = p_{\mathcal{Q}}(1 -
p_{\mathcal{Q}})$,
\begin{align}
  &\prob{\sum_{k=1}^T \mathcal{Q}_k > \log^2(T)} \leq
  \prob{\sum_{k=1}^T \mathcal{Q}_k - T p_{\mathcal{Q}} > \log^2(T)}  \nonumber \\ 
  &\leq \exp\left(-\frac{\log^4(T)/T}{2p_{\mathcal{Q}}(1 -
p_{\mathcal{Q}}) + 4\log^2(T)/3T}
  \right) \leq \exp\left(-\frac{\log^4(T)/T}{2|\mathcal{C}|/a_3 T + 4\log^2(T)/3T}
  \right)  \nonumber \\
  &= \exp\left(-\frac{\log^4(T)}{2|\mathcal{C}|/a_3 + 4\log^2(T)/3}
  \right)  = \exp\left(-\frac{1}{2}\Omega(\log^2(T))\right), \label{eqn:postelim3}
\end{align}
provided $T$ is large enough so that $\log^2(T) \geq 3 |\mathcal{C}|/a_3$. \\

\noindent {\bf Estimating the second term in (\ref{eqn:postelim2}).}
The second term in (\ref{eqn:postelim2}) be dealt with in a similar
fashion -- the probabilities
$\prob{\mathbbm{1}\{\theta_k \in S_c''\} \given G}$,
$k \geq 1, c \neq c^\star$, decay exponentially in $k$ as established
in the proof of Proposition \ref{prop:Sc''}. Hence, an application of
Bernstein's inequality as above gives
\begin{equation}
  \prob{\sum_{c \neq c^\star} N_c''(T) > 
    \log^2(T) \given G} \leq \exp\left(-\frac{1}{2}\Omega(\log^2(T))\right)  \label{eqn:postelim4}
\end{equation}
for $T$ large enough.%
Combining (\ref{eqn:postelim1})-(\ref{eqn:postelim4}) yields 
\[ \expect{ \sum_{l=1}^{|\mathcal{C}|-1} \sum_{T \geq
    k\geq{\hat{\tau}}_l+1} \mathbbm{1}\{\theta_k \in
  S_{\mathbbm{c}_l}' \} \mathbbm{1}\{ k \geq K_B\} \given G} = 2T
\exp(-\frac{1}{2}\Omega(\log^2(T))) = o(1).\] This, together with
(\ref{eqn:bda}) and a sum over all $c \neq c^\star$ (i.e., $l = 1,
\ldots, |\mathcal{C}|-1$), finally gives us
\[ \expect{ \sum_{l=1}^{|\mathcal{C}|-1} \sum_{T \geq
    k\geq{\hat{\tau}}_l+1} \mathbbm{1}\{\theta_k \in
  S_{\mathbbm{c}_l}' \} \given G} \leq \frac{|\mathcal{C}|}{a_3} +
o(1).\] An application of Markov's inequality completes the proof of
the lemma.
\end{proof}

We can now finally bound the number of samples of suboptimal policies
to get our regret bound, under the event
\begin{align*} H \bydef G \; \bigcap & \left\{\forall c \neq c^\star 
    \; \sum_{k \geq 1} \mathbbm{1}\{\theta_k \in S_c''\} \leq
    \frac{\alpha
      |\mathcal{C}|}{\delta} \right\} \\
  &\bigcap \left\{ \forall l \leq |\mathcal{C}|-1\; \sum_{k \geq
      {\hat{\tau}}_l+1} \mathbf{1}\{\theta_k \in S_{\mathbbm{c}_l}'\}
    \leq \frac{|\mathcal{C}|}{\delta a_1} +
    o(1) \right\},
\end{align*} which, according to the conclusions of Proposition
\ref{prop:conc}, Proposition \ref{prop:Sc''} and Lemma
\ref{lem:notmuchSa'}, occurs with probability at least $1 -
3\delta$. The only step that now remains to prove Theorem
\ref{thm:main} is-

\begin{proposition}[Bounding the \# of plays of suboptimal policies in
  $\mathcal{C}$]
  \label{prop:optim}
  Under $H$, 
  \[\sum_{t=1}^T \mathbbm{1}\{A_t \neq c^\star(S_t)\}
  \leq \mathsf{C} \log T + O(\log T), \]
  where $\mathsf{C}$ solves
  \begin{equation}
    \label{eqn:optimization}
    \begin{aligned}
      \mathsf{C} \bydef \; &\max && \sum_{l=1}^{|\mathcal{C}|-1} x_l(l)\\
      &\text{ s.t.}
      && x_l \in \mathbb{R}_+^{|\mathcal{C}|}, \quad \forall l = 1, 2, \ldots, |\mathcal{C}|-1, \\
%      &&& x_l(|\mathcal{C}|) = 0, \quad \forall l = 1, 2, \ldots, |\mathcal{C}|-1, \\
      &&& x_i(l) = x_l(l), \quad \forall i \geq l,  l = 1, 2, \ldots, |\mathcal{C}|-1, \\
      &&& x_i \geq x_j, \quad \forall 1 \leq j \leq i \leq |\mathcal{C}| - 1, \\
      &&& \sigma: \{1, 2, \ldots, |\mathcal{C}|-1 \} \to \mathcal{C} \setminus \{c^\star\} \; \mbox{\emph{injective}}, \\
      &&& \min_{\theta \in S_{\sigma(l)}'} \quad  x_l \cdot  D(\theta^\star || \theta) = \frac{1+\epsilon}{1-\epsilon}, \quad \forall l = 1, 2, \ldots, |\mathcal{C}|-1.
    \end{aligned}
  \end{equation}
  [Note: $a(i)$ denotes the $i$th coordinate of the vector $a$; $a
  \cdot b$ is the standard inner product of vectors $a$ and $b$.]
\end{proposition}

\begin{proof}
Under the event $H$, we have
\begin{align}
  &\sum_{t=1}^T \mathbbm{1}\{A_t \neq c^\star(S_t)\} \leq  \sum_{t=1}^T \sum_{c \in
    \mathcal{C} \setminus \{c^\star\}} \mathbbm{1}\{A_t = 
  c(S_t)\} \nonumber \\
  &= \sum_{k=1}^T \sum_{t=t_{k-1}}^{t_k-1} \sum_{c \in
    \mathcal{C} \setminus \{c^\star\}}  \mathbbm{1}\{C_k = 
  c\} = \sum_{c \in \mathcal{C} \setminus \{c^\star\} } \tilde{\tau}_{N_c(T),c} \nonumber \\
  &\leq \sum_{c \in \mathcal{C} \setminus \{c^\star\} } \left( N_c(T)\bar{\tau}_{c} + \rho(N_c(T))\sqrt{N_c(T)} \right) \nonumber \\
  &\leq \sum_{c \in \mathcal{C} \setminus \{c^\star\} } N_c(T)\bar{\tau}_{c} + \sqrt{\sum_{c \in \mathcal{C} \setminus \{c^\star\} } \frac{\rho^2(N_c(T))}{\bar{\tau}_{c}}  } \sqrt{\sum_{c \in \mathcal{C} \setminus \{c^\star\} } N_c(T) \bar{\tau}_{c}},  \label{eqn:decomp} 
\end{align}
where the penultimate line is thanks to Proposition \ref{prop:conc},
and the final line is by applying the Cauchy-Schwarz
inequality. Notice that the sum $\sum_{c \in \mathcal{C} \setminus
  \{c^\star\} } \frac{\rho^2(N_c(T))}{\bar{\tau}_{c}} $ is $O(\log \log
T)$ by Proposition \ref{prop:conc} (with $\delta$ fixed as
usual). Hence, it is enough to show that the first sum $\sum_{c \in
  \mathcal{C} \setminus \{c^\star\} } N_c(T)\bar{\tau}_{c}$ is at most
$\mathsf{C} \log T + O(1)$. \\

Using our decomposition (\ref{eqn:deftau}) of the epoch boundaries
into the stopping times or stopping epochs $\hat{\tau}_l$, $l = 1, 2,
\ldots, |\mathcal{C}|-1$, we can write
\begin{align*}
  \sum_{c \in \mathcal{C} \setminus \{c^\star\} } N_c(T)\bar{\tau}_{c} &= \sum_{c \in \mathcal{C} \setminus \{c^\star\} } N_c'(T)\bar{\tau}_{c} + \sum_{c \in \mathcal{C} \setminus \{c^\star\} } N_c''(T)\bar{\tau}_{c} \\
  &\leq \sum_{c \in \mathcal{C} \setminus \{c^\star\} } N_c'(T)\bar{\tau}_{c} + \frac{\alpha |\mathcal{C}|^2}{\delta} \\
  & \leq \sum_{l=1}^{|\mathcal{C}|-1} N_{\mathbbm{c}_l}'(T)\bar{\tau}_{\mathbbm{c}_l} + \frac{\alpha |\mathcal{C}|^2}{\delta} \\
  &= \sum_{l=1}^{|\mathcal{C}|-1} N_{\mathbbm{c}_l}'(\hat{\tau}_l)\bar{\tau}_{\mathbbm{c}_l} + \sum_{l=1}^{|\mathcal{C}|-1} \left( N_{\mathbbm{c}_l}'(T) - N_{\mathbbm{c}_l}'(\hat{\tau_l}) \right) \bar{\tau}_{\mathbbm{c}_l} + \frac{\alpha |\mathcal{C}|^2}{\delta} \\
  & \leq \sum_{l=1}^{|\mathcal{C}|-1} N_{\mathbbm{c}_l}'(\hat{\tau}_l)\bar{\tau}_{\mathbbm{c}_l} + \frac{|\mathcal{C}|^2
    e^{\lambda |\mathcal{C}|}}{\delta a_1} + \frac{\alpha |\mathcal{C}|^2}{\delta} \\
  & \leq \sum_{l=1}^{|\mathcal{C}|-1} M_{\mathbbm{c}_l}'(\hat{\tau}_l)\bar{\tau}_{\mathbbm{c}_l} + \underbrace{\sum_{c \in \mathcal{C}} \bar{\tau}_{c}+ \frac{|\mathcal{C}|^2
    e^{\lambda |\mathcal{C}|}}{\delta a_1} + \frac{\alpha |\mathcal{C}|^2}{\delta}}_{O(1)}. 
\end{align*}

With regard to (\ref{eqn:deftau}), let us now take
  \[ \sigma(l) = \mathbbm{c}_l, \quad 1 \leq l \leq |\mathcal{C}|-1, \] and
  \begin{equation*}
    x_l(i) = \left\{ 
      \begin{array}{cc}
        \frac{M_{\sigma(i)}'({\hat{\tau}}_i) \bar{\tau}_{\sigma(i)}}{\log T}, &  {\hat{\tau}}_i \leq {\hat{\tau}}_l, \\
        & \\
        \frac{M_{\sigma(i)}'(\hat{\tau}_l) \bar{\tau}_{\sigma(i)}}{\log T},  &  {\hat{\tau}}_i > {\hat{\tau}}_l.
      \end{array}
    \right.
  \end{equation*}
  From the construction (\ref{eqn:deftau}), (\ref{eqn:defCl}) and
  (\ref{eqn:Mc'propty}), it can be checked that the $\{x_l\}$ and
  $\sigma$ satisfy the constraints of the optimization problem
  (\ref{eqn:optimization}). This completes the proof of the
  proposition.
\end{proof}

\section{Proof of Theorem \ref{thm:mainfinite}}
\label{app:mainfinite}

To prove Theorem \ref{thm:mainfinite}, we show that Assumptions
\ref{ass:largeprob} and \ref{ass:largedenominator} hold as stated. \\

\noindent {\bf Showing Assumption \ref{ass:largeprob}.} The following
lemma shows that under small deviations of the empirical pair epoch
counts $J$, we can bound the probability of sampling $\theta^\star$
from below.

\begin{lemma}[Uniform lower bound on pair-empirical KL divergence]
    \label{lem:wtbd1finite}
    Fix $\epsilon \in (0,1)$. There exists $\lambda < \infty$ such
    that for each $\theta \in \Theta, c \in \mathcal{C}$ and
    $k \geq 1$, it holds that
    \[ V_c(t_k) \sum_{(s_1,s_2) \in \mathcal{S}^2}
    U_{(s_1,s_2)}\left(V_c(t_k), c\right) \log
    \frac{p_{\theta^\star}(s_1,c(s_1),s_2)}{p_{\theta}(s_1,c(s_1),s_2)}
    \geq -\lambda\]
whenever
\[ \left| \frac{J_{(s_1,s_2)}(k_c,c)}{k_c} - \bar{\tau}_c \;
  \pi^{(c)}_{(s_1,s_2)} \right| \leq \sqrt{\frac{e_1 \log\left({e_2 \log
        k_c} \right)}{k_c}} \quad \forall s_1, s_2 \in \mathcal{S},
k_c \geq 1, c \in \mathcal{C}, k = \sum_{c \in \mathcal{C}} k_c. \]
  \end{lemma}
  
  \begin{proof}
    Set $V_c(t_k) = \tilde{\tau}_{k'_c,c}$
    for some integer $k'_c$. We can write
    \begin{align}
      & V_c(t_k) \sum_{(s_1,s_2) \in \mathcal{S}^2}
        U_{(s_1,s_2)}\left(V_c(t_k), c\right) \log
        \frac{p_{\theta^\star}(s_1,c(s_1),s_2)}{p_{\theta}(s_1,c(s_1),s_2)}  \nonumber \\
      &= \sum_{(s_1,s_2) \in \mathcal{S}^2} \tilde{\tau}_{k'_c,c} \cdot
        U_{(s_1,s_2)}\left(\tilde{\tau}_{k'_c,c}, c\right) \log
        \frac{p_{\theta^\star}(s_1,c(s_1),s_2)}{p_{\theta}(s_1,c(s_1),s_2)}  \nonumber \\
      &= \sum_{(s_1,s_2) \in \mathcal{S}^2} \left[\tilde{\tau}_{k'_c,c}
        \cdot U_{(s_1,s_2)}\left(\tilde{\tau}_{k'_c,c}, c\right) -
        k'_c\bar{\tau}_{c} \cdot {\pi^{(c)}_{(s_1,s_2)}}
        \right] \log
        \frac{p_{\theta^\star}(s_1,c(s_1),s_2)}{p_{\theta}(s_1,c(s_1),s_2)} \nonumber \\
      &\quad \quad \quad \quad + \sum_{(s_1,s_2) \in \mathcal{S}^2}
        k'_c\bar{\tau}_{c} \cdot {\pi^{(c)}_{(s_1,s_2)}} \log
        \frac{p_{\theta^\star}(s_1,c(s_1),s_2)}{p_{\theta}(s_1,c(s_1),s_2)} \nonumber \\
      &\geq - \sum_{(s_1,s_2) \in \mathcal{S}^2} \rho_{e_1,e_2}(k'_c) \sqrt{k'_c} \cdot
        \left| \log
        \frac{p_{\theta^\star}(s_1,c(s_1),s_2)}{p_{\theta}(s_1,c(s_1),s_2)}
        \right| \nonumber \\
      &\quad \quad \quad \quad + k'_c\bar{\tau}_{c} \sum_{s_1 \in \mathcal{S}}
        \pi^{(c)}_{s_1} \sum_{s_2 \in \mathcal{S}}
        \frac{\pi^{(c)}_{(s_1,s_2)}}{\pi^{(c)}_{s_1}} \log
        \frac{p_{\theta^\star}(s_1,c(s_1),s_2)}{p_{\theta}(s_1,c(s_1),s_2)} \nonumber \\
      &\geq k'_c\bar{\tau}_{c} \cdot D_c(\theta) - \Gamma |\mathcal{S}|^2 \rho_{e_1,e_2}(k'_c) \sqrt{k'_c}, \label{eqn:wtbd1finite}
    \end{align}
    where
    $\rho_{e_1,e_2}(x) \bydef \sqrt{{e_1 \log\left({e_2 \log x}
        \right)}}$.
    The first inequality above is obtained thanks to (\ref{eqn:conc2})
    of Proposition \ref{prop:conc}. For a fixed
    $\theta \neq \theta^{\star}$ and $c$, the expression in
    (\ref{eqn:wtbd1finite}) tends to $\infty$ as $k'_c \to
    \infty$.
    Denote the infimum of the expression over all $k'_c \geq 1$ by
    $-\lambda_{\theta,c}$. The lemma now follows by setting $\lambda$
    to be the largest $\lambda_{\theta,c}$ across the finitely many
    $\theta$ and $c$.
  \end{proof}  

  Using the bound of Lemma \ref{lem:wtbd1finite} in the expression for the
  posterior density (\ref{eqn:wt2}), we can bound the posterior
  probability of $\{\theta^\star\} \subseteq S_{c^\star}$ from below
  as:
  \[\forall k \geq 1 \; \pi_{t_k}(\theta^\star) \geq \frac{\pi(\theta^\star)}{\int_{\Theta} \exp(\lambda |\mathcal{C}|) 
    \pi(d\theta)} = \pi(\theta^\star) e^{-\lambda |\mathcal{C}|}
  \equiv p^\star > 0.\] 

  \noindent {\bf Showing Assumption \ref{ass:largedenominator}.}
  Assumption \ref{ass:largedenominator} is naturally seen to hold here
  by observing that since
  $D(\theta^\star || \theta^\star) = 0 \in
  \mathbb{R}^{|\mathcal{C}|}$, 
  \[ \pi\left(\left\{ \theta \in \Theta: \sum_{c \in \mathcal{C}} k_c
      \bar{\tau}_c D_c(\theta^\star || \theta) \leq 1 \right\} \right)
  \geq \pi\left(\left\{\theta^\star \right \}\right) > 0,\]
  by Assumption \ref{ass:finite} (grain of truth). Thus, Assumption
  \ref{ass:largedenominator} is seen to hold with $a_2 = a_4 = 0$.
  This completes the proof of the theorem.

\section{Proof of Theorem \ref{thm:maincontinuous}}
\label{app:maincontinuous}
\noindent {\bf Showing Assumption \ref{ass:largeprob}.} For $k_c$
epoch uses of policy $c$, and with $k = \sum_{c \in \mathcal{C}} k_c$,
it is seen that the posterior density factors into a product of {\em
  truncated} Beta densities, each for the $4$ independent components
$\theta^{(i)}_{jl}$ of the parameter $\theta$, and where the
truncation is simply the restriction
to the interval $[\upsilon, 1-\upsilon]$ for each component. \\

Let us now assume that for $k_c$ epoch uses of policy $c$, the
empirical state pair frequencies $J_{(s_1,s_2)}(k_c,c)$, $s_1, s_2 \in
\mathcal{S}$, $c \in \mathcal{C}$, are ``close to'' their respective
expectations, i.e.,
  \[ \left| \frac{J_{(s_1,s_2)}(k_c,c)}{k_c} - \bar{\tau}_c \;
    \pi^{(c)}_{(s_1,s_2)} \right| \leq \frac{\rho_{e_1,e_2}(k_c)}{\sqrt{k_c}} \bydef
  \sqrt{\frac{e_1 \log\left({e_2 \log k_c} \right)}{k_c}} \quad
  \forall s_1, s_2 \in \mathcal{S}, k_c \geq 1, c \in \mathcal{C}. \]
  This, in turn, can be used to show that the parameters
  $\alpha^{(i)}_{jl}, \beta^{(i)}_{jl}$ of the (truncated) Beta
  posterior density for each component $\theta^{(i)}_{jl}$ satisfy
  inequalities of the form
  \[\left| \frac{\alpha^{(i)}_{jl}}{\alpha^{(i)}_{jl} +
      \beta^{(i)}_{jl}} - \theta^{(i)}_{jl}\right| \leq
  \frac{\rho_{e'_1,e'_2}(\alpha^{(i)}_{jl} +
    \beta^{(i)}_{jl})}{\sqrt{(\alpha^{(i)}_{jl} +
      \beta^{(i)}_{jl})}}\] for some constants $e_1', e_2' > 0$, for
  all $i,j,l \in \{1,2\}, l\neq j$. \\

  Since Assumption \ref{ass:unique} is satisfied for $\theta^\star$,
  there must exist a closed $||\cdot||_\infty$ ball $\mathcal{N}$
  around $\theta^\star$, \[\mathcal{N} \equiv \prod_{i,j,l
    \in \{1,2\}, l\neq j} \mathcal{N}^{(i)}_{jl}, \]  such that
  $\mathcal{N} \subseteq S_{c^\star}$. We can bound from below the
  posterior probability of playing $c^\star$ as
  $\pi_{t_k}(S_{c^\star}) \geq \pi_{t_k}(\mathcal{N})$, after which
  the following lemma establishes a lower bound on the latter
  quantity, and hence Assumption \ref{ass:largeprob}.

  \begin{lemma}[Concentration of Beta probability mass]
    \label{lem:betalargeprob}
    For each $m = 1, 2, \ldots$, let $\mu_m$ be a truncated
    Beta$(\alpha_m,\beta_m)$, $\alpha_m + \beta_m = m$, probability
    measure on $[\upsilon,1-\upsilon]$, $0 < \upsilon < 1/2$, i.e., a
    standard Beta$(\alpha_m,\beta_m)$ probability measure on $[0,1]$
    restricted to $[\upsilon,1-\upsilon]$ and normalized. Let $I \in
    [\upsilon,1-\upsilon]$ be a sub-interval containing $\theta$ in
    its interior. If $\left|\frac{\alpha_m}{m} - \theta \right| =
    \frac{o(\log m)}{\sqrt{m}}$ for all $m$, then
    \[ \inf_{m \geq 1} \mu_m(I) > 0. \]
  \end{lemma}

  \begin{proof}
    Let $q > 0$ be such that the ($1$-dimensional) ball of radius $q$
    around $\theta$, $\mbox{Ball}(\theta;q)$, is contained in
    $I$. Since $\left|\frac{\alpha_m}{m} - \theta \right| =
    \frac{o(\log m)}{\sqrt{m}}$ for all $m$, there exists $m_0 \geq 1$
    such that for every $m > m_0$, we have (a) $\frac{\alpha_m}{m} \in
    \mbox{Ball}(\theta; q/2)$ and (b) $\frac{1}{\sqrt{2(m+1)}} <
    \frac{q}{2}$. Since the mean of a Beta$(\alpha_m,\beta_m)$
    distribution is $\frac{\alpha_m}{m}$ and its variance at most
    $\frac{1}{4(m+1)}$, Chebyshev's inequality can be used to argue
    that for $m \geq m_0$, $\mu_m(I) \geq \mu_m(\mbox{Ball}(\theta;q))
    \geq 1/2$. The proof is complete by taking the minimum with the
    positive probabilities $\mu_m(I)$, $1 \leq m \leq m_0$.
  \end{proof}

\noindent {\bf Showing Assumption \ref{ass:largedenominator}.}
% Letting $k
% = \sum_{c \in \mathcal{C}} k_c$, the posterior density at epoch $k$ is
% (\ref{eqn:postprob})
% \begin{align}
%   W_{t_k}(\theta) &= \exp \left(-\sum_{c,s_1,s_2}
%     J_{(s_1,s_2)}\left(k_c, c\right) \log
%     \frac{p_{\theta^\star}(s_1,c(s_1),s_2)}{p_\theta(s_1,c(s_1),s_2)}
%   \right) \nonumber \\
%   &\leq \exp \left(-\sum_{c} k_c \bar{\tau}_c D_c(\theta^\star||\theta) +
%     \sum_{c} \rho(k_c)\sqrt{k_c} \; L_c(\theta^\star||\theta)
%   \right), \label{eqn:contpostden}
% \end{align}
% where $L_c(\theta^\star||\theta) \bydef \sum_{s_1,s_2} \left| \log
%   \frac{p_{\theta^\star}(s_1,c(s_1),s_2)}{p_\theta(s_1,c(s_1),s_2)}
% \right|$. \\
Note that each marginal KL divergence, decouples additively across the
independent parameters: for each $c \equiv (i,j)$,
  \begin{align*}
   \bar{\tau}_c  D_c(\theta^\star||\theta) &= \bar{\tau}_c \sum_{s_1 \in \mathcal{S}}
    \pi^{(c)}_{s_1} \;
    \kldiv{p_{\theta^\star}(s_1,c(s_1),\cdot)}{p_{\theta}(s_1,c(s_1),\cdot)} \\
    &= \frac{\bar{\tau}_c \theta^{\star(j)}_{21}}{\theta^{\star(i)}_{12} +
      \theta^{\star(j)}_{21}} \;
    \kldiv{\theta^{\star(i)}_{12}}{\theta^{(i)}_{12}} +
    \frac{\bar{\tau}_c \theta^{\star(i)}_{12}}{\theta^{\star(i)}_{12} +
      \theta^{\star(j)}_{21}} \;
    \kldiv{\theta^{\star(j)}_{21}}{\theta^{(j)}_{21}} \\
    &\equiv \varphi_{c}(1) \; 
    \kldiv{\theta^{\star(i)}_{12}}{\theta^{(i)}_{12}} + \varphi_{c}(2) \; 
    \kldiv{\theta^{\star(j)}_{21}}{\theta^{(j)}_{21}},
  \end{align*}
  with $\varphi_{c}(1) \bydef \frac{\bar{\tau}_c
    \theta^{\star(j)}_{21}}{\theta^{\star(i)}_{12} +
    \theta^{\star(j)}_{21}}$, $ \varphi_{c}(2) \bydef
  \frac{\bar{\tau}_c \theta^{\star(i)}_{12}}{\theta^{\star(i)}_{12} +
    \theta^{\star(j)}_{21}}$.  Also, since $\Theta = [\upsilon,
  1-\upsilon]^4$, it follows by a Taylor series expansion of the
  KL-divergence that there exists a constant $\varrho > 0$ such that
  \[ \kldiv{\theta^{\star(i)}_{jl}}{x} \leq
  \varrho\left(\theta^{\star(i)}_{jl} - x \right)^2 \quad \forall x
  \in [\upsilon, 1-\upsilon], \forall i, j, l, l \neq j. \] With this
  observation, weighted KL divergence neighborhoods of $\theta^\star$ are
  seen to contain appropriately scaled Euclidean neighborhoods of
  $\theta^\star$. To show Assumption \ref{ass:largedenominator}(A), we
  compute
  \begin{align*}
    \pi\left(\left\{ \theta \in \Theta: \sum_{c \in \mathcal{C}} k_c
        \bar{\tau}_c D_c(\theta^\star || \theta) \leq 1 \right\}
    \right) &\geq \pi\left(\left\{ \theta \in \Theta: \sum_{l \neq
          j,i} \gamma^{(i)}_{jl}
        \left(\theta^{\star(i)}_{jl} - \theta^{(i)}_{jl} \right)^2
        \leq \frac{1}{\varrho \bar{\tau}_{\max}} \right\} \right),
  \end{align*}
  where $\bar{\tau}_{\max} \bydef \max_c \bar{\tau}_c$, and $\sum_{l
    \neq j, i} \gamma^{(i)}_{jl} = 2\sum_c k_c \equiv 2k$, since each
  policy $c$ is informative about exactly $2$ of the $4$ independent
  parameter components. Using this fact, we can continue the bound as
  follows.
  \begin{align*}
    \pi\left(\left\{ \theta \in \Theta: \sum_{l \neq j,i}
        \gamma^{(i)}_{jl} \left(\theta^{\star(i)}_{jl} -
          \theta^{(i)}_{jl} \right)^2 \leq \frac{1}{\varrho
          \bar{\tau}_{\max}} \right\} \right) &\geq \pi\left(\left\{
        \theta \in \Theta: \sum_{l \neq j,i}
        \left(\theta^{\star(i)}_{jl} - \theta^{(i)}_{jl} \right)^2
        \leq \frac{1}{2k\varrho \bar{\tau}_{\max}} \right\} \right) \\
    &\geq a_1 {k}^{-2}
  \end{align*}
  using the well-known volume of a multidimensional Euclidean ball.

  Assumption \ref{ass:largedenominator}(B) results from a calculation
  similar to the above, but by considering the ellipsoid $\left\{
    \theta \in \Theta: \sum_{l \neq j,i} \gamma^{(i)}_{jl}
    \left(\theta^{\star(i)}_{jl} - \theta^{(i)}_{jl} \right)^2 \leq
    \frac{1}{\varrho \bar{\tau}_{\max}} \right\}$ with a choice of
  weights $\gamma^{(1)}_{21} = \gamma^{(1)}_{21} \geq k - 3 \log^2(k)$
  and $\gamma^{(2)}_{21} + \gamma^{(2)}_{21} \leq 6 \log^2(k)$, in
  which case the volume of the ellipsoid is at least $a_3
  \sqrt{k}^{-2} = a_3 k^{-1}$.

\section{Proof of Theorem \ref{thm:atleastLmain}}
\label{app:atleastLmain}

  For each $c \neq c^\star$, let
  $\delta_c \bydef \min_{c \neq c^\star, \theta \in S_c'}
  D_c(\theta^\star || \theta)$.
  Consider a solution
  $\left((x_l)_{l =1}^{|\mathcal{C}|-1}, \sigma\right)$ to the
  optimization problem (\ref{eqn:main}). Since
  \begin{equation}
    \label{eqn:lastconstraint}
    \min_{\theta \in S_{\sigma(l)}'} \quad x_l \cdot D(\theta^\star
  || \theta) = (1+a_4)\left(\frac{1+\epsilon}{1-\epsilon}\right)
  \quad \forall 1 \leq l \leq |\mathcal{C}|-1,
  \end{equation}
  we must have $x_l(l) = z^{\circ}(l) \leq \chi / \tilde{\Delta} $
  with $\chi \bydef \frac{(1+a_4)(1+\epsilon)}{1-\epsilon}$
  $\forall l = 1, \ldots, |\mathcal{C}|-1$.

  Put $z^{\circ} \bydef x_{|\mathcal{C}|-1}$,
  $c^{\circ} \bydef \sigma(|\mathcal{C}|-1)$.
%   \[ \min_{\theta \in S_{c^{\circ}}'} \quad z^{\circ} \cdot
%   D(\theta^\star || \theta) =
%   (1+a_4)\left(\frac{1+\epsilon}{1-\epsilon}\right), \]
  We claim that
  $||z||_1 \equiv \ip{\mathbf{1}}{z}\leq \left(\frac{|\mathcal{A}| -
      L}{\tilde{\Delta}}\right) \chi$.
  If not, set
  $y^{\circ} \bydef \frac{\chi}{\tilde{\Delta}} \left(1, 1, \ldots, 1,
    0 \right) \in \mathbb{R}^{|\mathcal{C}|}$,
  and\footnote{$\min(x,y)$ for two vectors is to be interpreted as the
    pointwise minimum.}
  $D^{\tilde{\Delta}}(\theta^\star || \theta) \bydef \min\left(
    D(\theta^\star || \theta), \tilde{\Delta} \times \mathbf{1} \right)$.
  Let us estimate, for $ \theta \in S_{c^{\circ}}'$ that attains the
  minimum in (\ref{eqn:lastconstraint}) for $l = |\mathcal{C}|-1$,
  \begin{align}
    (y^{\circ} - z^{\circ}) \cdot D^{\tilde{\Delta}}(\theta^\star || \theta)  &= y^{\circ} \cdot D^{\tilde{\Delta}}(\theta^\star || \theta)  - z^{\circ} \cdot D^{\tilde{\Delta}}(\theta^\star || \theta) \nonumber \\
                                                             &\geq \chi \cdot L \cdot \Delta \cdot \frac{1}{\Delta} - \chi = \chi(L-1). \label{eqn:ip1}
  \end{align}

  But then\footnote{$\mathbf{1}$ represents the all-ones vector.},
  \begin{align*}
    (y^{\circ} - z^{\circ}) \cdot \mathbf{1}
    &= y^{\circ} \cdot \mathbf{1} - z^{\circ} \cdot \mathbf{1} \\
    &< \frac{\chi(|\mathcal{C}|-1)}{\tilde{\Delta}} -
      \frac{\chi(|\mathcal{C}|-L)}{\tilde{\Delta}}
      = \frac{\chi(L-1)}{\tilde{\Delta}} \\
    &\leq \frac{(y^{\circ} - z^{\circ}) \cdot D^{\tilde{\Delta}}(\theta^\star || \theta)}{\tilde{\Delta}} \quad \quad \quad \mbox{by (\ref{eqn:ip1})} \\
    &\leq \frac{\ip{(y^{\circ}-z^{\circ})}{(\Delta \times \mathbf{1})}}{\Delta} =
      \ip{(y^{\circ}-z^{\circ})}{\mathbf{1}},
  \end{align*}
  since
  $D^{\tilde{\Delta}}(\theta^\star || \theta) \preceq \Delta \times
  \mathbf{1}$
  by definition, and $z^{\circ} \preceq y^{\circ}$ by hypothesis. This
  is a contradiction.

  \section{Example: Single Parameter Queueing MDP with a Large Number
    of States (Section \ref{sec:scalingconst})}
\label{app:queueing}

In this section, we show an MDP possessing a large number of states
but only a small number of uncertain parameters, in which the regret
scaling with time can be demonstrated to {\em not} depend at all on
the number of states (and hence the number of possible stationary
policies).

Consider learning to control a discrete time, two-server single queue
MDP\footnote{Such a model has been classically studied in queueing and
  control theory \citep{LinKum84:queuecontrol, Koo95:threshold} in the
  planning context.}, parameterized by a single scalar parameter
$\theta$. The state space is $\mathcal{S} := \{0, 1, 2, ..., M\}$, $M$
a positive integer, representing the occupancy of a size-at most-$M$
queue of customers. A customer arrives to the system independently
each time with probability $\theta$, i.e., arrivals to the queue
follow a Bernoulli($\theta$) probability distribution, where
$\theta \in \Theta \bydef [\upsilon,1-\upsilon]$,
$0 < \upsilon \ll 1/2$, is the unknown parameter for the MDP. At each
state, one of $2$ actions -- Action $1$ (SLOW service) and Action $2$
(FAST service) may be chosen, i.e., $\mathcal{A} = \{1,2\}$. Applying
SLOW (resp. FAST) service results in serving one packet from the queue
with probability $\mu_1$ (resp. $\mu_2$) if it is not empty, i.e., the
service model is Bernoulli($\mu_i$) where $\mu_i$ is the packet
service probability under service type $i = 1,2$. Actions $1$ and $2$
incur a per-instant cost of $c_1$ and $c_2$ units respectively. In
addition to this cost, there is a holding cost of $c_0$ per packet in
the queue at all times. The system gains a reward of $r$ units
whenever a packet is served from the queue. Let us assume that
$\mu_1$, $\mu_2$, $c_0$, $c_1$, $c_2$ and $r$ are known constants,
with the only uncertainty being in $\theta \in \Theta$. Thus, the true
MDP is represented by some $\theta^\star \in \Theta$ with a
corresponding optimal policy $c^\star$ mapping each state to one of
$\{\mu_1, \mu_2\}$.  The total number of policies is of order $2^M$,
and the number of optimal policies $|\mathcal{C}|$ can potentially be
of order $M$ (this occurs, for instance, if optimal policies are of
threshold type w.r.t. the state space, and the threshold monotonically
increases from $0$ to $M$ as $\theta$ ranges in $\Theta$
\citep{LinKum84:queuecontrol}).

With regard to the TSMDP algorithm, let us assume that the start state
(and thus the epoch demarcating state) is $s_0 := 0$, and the prior a
uniform probability distribution over $\Theta$.

% another
% example of an MDP with monotone threshold policies is {\em RiverSwim}
% [check?]

  \noindent {\bf Analysis.} Let us estimate the marginal KL divergence
  $D_c(\theta^\star || \theta)$ for a candidate parameter
$\theta \in \Theta$ and a stationary policy $c$. First, notice that at
each state $0 < s < M$,
\[ \kldiv{p_{\theta^\star}(s,\mu_i,\cdot)}{p_{\theta^\star}(s,\mu_i,\cdot)} 
= \kldiv{[\mu_i\bar{\theta^\star}; \mu_i\theta^\star+
  \bar{\mu_i}\bar{\theta^\star};
  \bar{\mu_i}\theta^\star]}{[\mu_i\bar{\theta}; \mu_i\theta +
  \bar{\mu_i}\bar{\theta}; \bar{\mu_i}\theta]}, \]
where $\bar{x}$ denotes $1-x$. This can be bounded from below using
Pinsker's inequality to get
\begin{align*}
  \kldiv{p_{\theta^\star}(s,\mu_i,\cdot)}{p_{\theta^\star}(s,\mu_i,\cdot)}
  & \geq \frac{1}{2} \left| \left|[\mu_i\bar{\theta^\star}; \mu_i\theta^\star+
    \bar{\mu_i}\bar{\theta^\star};
    \bar{\mu_i}\theta^\star] -  [\mu_i\bar{\theta}; \mu_i\theta +
    \bar{\mu_i}\bar{\theta}; \bar{\mu_i}\theta] \right|\right|^2_1 \\
  &= \frac{1}{2}(\theta^\star - \theta)^2 (1 + |2\mu_i - 1|)^2 \geq a_1 (\theta^\star - \theta)^2,
\end{align*}
with $a_1 \bydef \frac{1}{2} \min_{i = 1,2} (1 + |2\mu_i - 1|)^2
$. Similarly, for states $s \in \{0,M\}$,
\begin{align*}
  \kldiv{p_{\theta^\star}(s,\mu_i,\cdot)}{p_{\theta^\star}(s,\mu_i,\cdot)}
  & \geq a_2 (\theta^\star - \theta)^2
\end{align*}
for some positive constant $a_2$. Thus, we have
$D_c(\theta^\star || \theta) \geq a(\theta^\star - \theta)^2$ for
$a \bydef \min\{a_1,a_2\}$, since $D_c(\theta^\star || \theta)$ by
definition is a convex combination of individual KL divergence terms
as above. In particular, it follows that for each suboptimal parameter
$\theta$ (i.e., $\theta \in S_c$, $c \neq c^\star$), the vector
$D(\theta^\star || \theta)$ of all $D_c(\theta^\star || \theta)$
values is such that each of its coordinates is at least
$a(\theta^\star - \theta)^2$. Let
$\theta_b \bydef \arg\min_{\theta \in S_c, c \neq c^\star}
|\theta^\star - \theta|$
be the closest suboptimal parameter to the true parameter
$\theta^\star$. Under the non-degenerate case where the MDP
parameterized by $\theta^\star$ possesses a unique optimal policy, we
must have $\delta^\star \bydef (\theta_b - \theta^\star)^2 > 0$. 

Theorem \ref{thm:atleastLmain} can now be applied, with
$\Delta \bydef \delta^\star$ and $L \bydef |\mathcal{C}|-1$, to get
that the scaling constant $\mathsf{C}$ satisfies
$\mathsf{C} \leq
\frac{(1+a_4)(1+\epsilon)}{\delta^\star(1-\epsilon)}$.

Thus, if all the assumptions required for Theorem \ref{thm:main} are
satisfied\footnote{These can be shown to be satisfied using techniques
  similar to those used to show Theorem \ref{thm:maincontinuous}.},
then the regret scaling does {\em not} depending on the number of
policies ($|\mathcal{C}|$). Using a naive bandit approach treating
each policy as an arm of the bandit (and thus completely ignoring the
structure of the MDP) would, in contrast, result in regret that scales
at rate $\frac{|\mathcal{C}|}{\delta^\star}\log T$ -- a huge blowup
compared to the former. In summary,
\begin{itemize}
\item The number of states $|\mathcal{S}|$ (and thus the number of
  possible optimal policies of the order of $\Omega(|\mathcal{S}|)$)
  can potentially be very large, while the number of uncertain
  parameter dimensions can be relatively much smaller. One can
  consider running a ``flat'' bandit algorithm on all possible optimal
  policies (order $|\mathcal{C}| = \Omega(|\mathcal{S}|)$ or
  larger). This will yield the standard decoupled regret that is
  $O\left(\frac{|\mathcal{C}|}{\delta^\star}\log T\right)$.
  Furthermore, even an MDP-specific algorithm like UCRL2, in this
  setup, is unable to exploit the high amount of generalizability
  across states/actions, and exhibits a regret scaling of
  $O\left( \frac{\mathcal{D}^2 |\mathcal{S}|^2 |\mathcal{A}|
      \log(T)}{g} \right)$
  \citep[Theorem 4]{JakschOA10}, where $\mathcal{D}$ is the MDP
  diameter and $g$ is the gap between the expected return of the best
  and second-best policies.

\item Thompson Sampling for MDPs, with a prior on the uncertainty
  space of parameters, can yield regret that scales as
  $O\left(\frac{1}{\delta^\star} \log T\right)$ which is {\em
    independent} of $|\mathcal{C}|$. This represents a dramatic
  improvement in regret especially when $|\mathcal{S}|$ is large. 

\item Intuitively, the reason for the saving in regret is that with a
  prior over the structure of the MDP, {\em every}
  transition/recurrence cycle in the Thompson Sampling algorithm (and
  the resulting posterior update) gives non-trivial information in
  resolving suboptimal models from the true underlying model, This is
  completely ignored by a flat bandit algorithm across policies which
  is forced to explore all available arms (policies). 

\end{itemize}

\section{Proof of Theorem \ref{thm:regret}}
\begin{lemma}[Concentration of the empirical reward process]
  \label{lem:conc3}
  Let $\delta \in (0,1]$. Then, there exist positive $d_3, d_4$ such
    that the following bound holds with probability at least
    $1-\delta$ over the choice of the matrix $Q$,
  \begin{equation} \forall k \geq 1 \quad \sum_{l=1}^{\tilde{\tau}_{k,c^\star}} (\mu^\star - Q_3(l,c^\star)) <  \sqrt{d_3 k \log
      \left(\frac{d_4 \log k}{\delta}\right)}. \label{eqn:conc3}
    \end{equation}
\end{lemma}

\begin{proof}
  The proof is along the same lines as that of Proposition
  \ref{prop:conc}. Break the sum on the left as
  $\sum_{l=1}^{\tilde{\tau}_{k,c^\star}} (\mu^\star - Q_3(l,c^\star))
  = \sum_{l'=1}^k \hat{B}_{l'}$,
  where the cycle-based random variables
  \[\hat{B}_{l'} \bydef
  \sum_{l=\tilde{\tau}_{l'-1,c^\star}+1}^{\tilde{\tau}_{l',c^\star}}
  (\mu^\star - Q_3(l,c^\star)), \quad l' = 1, 2, 3, \ldots,\] are IID
  owing to the Markov property. Also, by the renewal-reward theorem
  \citep{grimmett92} and Markov chain ergodicity, it follows that
  $\expect{\hat{B}_1} = 0$. Most importantly, $\hat{B}_1$ is
  stochastically dominated by $2 r_{\max} \tilde{\tau}_{1,c^\star}$, and
  thus possesses an exponentially decaying tail
  (\ref{eqn:tautail}). An application of Lemma \ref{lem:maximal} thus
  gives that for some $d_3, d_4$, with probability at least
  $1-\delta$,
  \[ \forall k \geq 1 \quad \sum_{l'=1}^k \hat{B}_{l'} \leq  \sqrt{d_3 k \log \left(\frac{d_4 \log k}{\delta}\right)}.\]
  This proves the lemma.
\end{proof}

We decompose the regret along the trajectory up to time $T$ as
follows.
\begin{align}
  &T\mu^\star - \sum_{t=1}^T r(S_t,A_t)  = \sum_{k=1}^{e(T)} \sum_{t=t_{k-1}}^{t_k -
    1} \sum_{c \in \mathcal{C}} \mathbbm{1}\{C_k = c\} (r(S_t,A_t) - \mu^\star) \nonumber \\
  &= \sum_{k=1}^{e(T)} \sum_{t=t_{k-1}}^{t_k - 1} \mathbbm{1}\{C_k =
  c^\star\} (\mu^\star - r(S_t,A_t)) + \sum_{k=1}^{e(T)} \sum_{t=t_{k-1}}^{t_k - 1}
  \sum_{c \neq
    c^\star} \mathbbm{1}\{C_k = c\} (\mu^\star - r(S_t,A_t)) \nonumber \\
  &\leq \sum_{k=1}^{e(T)} \sum_{t=t_{k-1}}^{t_k - 1} \mathbbm{1}\{C_k =
  c^\star\} (\mu^\star - r(S_t,A_t)) + 2r_{\max} \sum_{c \in \mathcal{C}
    \setminus \{c^\star\} } \tilde{\tau}_{N_c(T),c} \nonumber \\
  &\leq \sum_{k=1}^{e(T)} \sum_{t=t_{k-1}}^{t_k - 1} \mathbbm{1}\{C_k
  = c^\star\} (\mu^\star - r(S_t,A_t)) + 2 r_{\max} (\mathsf{B} + \mathsf{C} \log T)
  \quad \quad \mbox{(by Proposition \ref{prop:optim})} \nonumber \\
  &= \sum_{l=1}^{\tilde{\tau}_{N_{c^\star}(T),c^\star}} (\mu^\star - Q_3(l,c^\star))
  + r_{\max} (\mathsf{B} + \mathsf{C} \log T). \label{eqn:decomp2}
\end{align}
The first step above uses the recurrence cycle structure of the TSMDP
algorithm, $r_{\max}$ in the third step is defined to be the maximum
reward for any state-action pair: $r_{\max} \bydef \max_{s \in
  \mathcal{S}, a \in \mathcal{A}} r(s,a)$, and in the final step we
use the coupling with the alternative probability space described in
Section \ref{sec:alternative}.

Under the event $G$, we have the estimate
\[\forall k \quad \tilde{\tau}_{k,c^\star} \geq k
\bar{\tau}_{c^\star} - \sqrt{k d_1 \log \left(\frac{|\mathcal{C}|
    |\mathcal{S}|^2 d_2 \log k}{\delta}\right)}\]
\[ \Rightarrow \quad T \geq \tilde{\tau}_{N_{c^\star}(T),c^\star} \geq N_{c^\star}(T)
\bar{\tau}_{c^\star} - \sqrt{N_{c^\star}(T) d_1 \log
  \left(\frac{|\mathcal{C}| |\mathcal{S}|^2 d_2 \log
    N_{c^\star}(T)}{\delta}\right)}.\] The square-root correction term
above is $o(N_{c^\star}(T))$, thus for any $\epsilon_1 > 0$, we have
$N_{c^\star}(T) \leq \frac{(1+\epsilon_1)T}{\bar{\tau}_{c^\star}}$
for $T$ large enough. 

Let $G_1$ be the event, occurring with probability at least
$1-\delta$, for which (\ref{eqn:conc3}) is satisfied. Then, the event
$G \cap G_1$ occurs with probability at least $1-2\delta$ by the union
bound. Using the bound on $N_{c^\star}(T)$ from the preceding
paragraph in (\ref{eqn:decomp2}) thus gives that for $T$ large enough,
under the event $G \cap G_1$, 
\begin{align*}
 T\mu^\star - \sum_{t=1}^T r(S_t,A_t) &\leq  \sqrt{\frac{d_3(1+\epsilon_1)T}{\bar{\tau}_{c^\star}} \log
      \left(\frac{d_4 \log \left(\frac{(1+\epsilon_1)T}{\bar{\tau}_{c^\star}}\right) }{\delta}\right)} + r_{\max} \mathsf{B} + r_{\max} \mathsf{C} \log T \\
&= O\left(\sqrt{\frac{T}{\bar{\tau}_{c^\star}} \log \left(\frac{\log T}{\delta} \right)}\right).
\end{align*}

%\bibliography{tsmdp}
%\bibliographystyle{plainnat}

\end{document}